%% file: main.tex
\documentclass[12pt]{graphicsclass}
\pdfoutput=1
\usepackage[utf8]{inputenc}
\usepackage{palatino}
\usepackage{amssymb,amsmath,bm,amsthm}
\usepackage{graphicx}
\graphicspath{ {images/} }
\usepackage[caption=false,font=scriptsize]{subfig}
\usepackage[ruled,vlined]{algorithm2e}
\usepackage[labelfont=bf]{caption}
\usepackage{setspace}
\usepackage[margin=1in]{geometry}
\usepackage{hyperref}
\usepackage{cite}
\usepackage{xcolor}
\usepackage{overpic}
\usepackage[all]{nowidow}
\usepackage{enumerate}
\usepackage{color,colortbl}
\usepackage{epstopdf}
\usepackage{booktabs, multirow}
\definecolor{darkgreen}{rgb}{0, 0.5, 0}
\definecolor{red}{rgb}{1, 0, 0}
\definecolor{purple}{rgb}{0.5, 0, 0.5}
\usepackage{mathtools}

\DeclareMathOperator*{\argmin}{argmin}

\newcommand\algoname{META}

\newcommand\ie{\textit{i.e.}}
\newcommand\eg{\textit{e.g.}}
\newcommand\st{\textit{s.t.}}
\newcommand\wrt{\textit{w.r.t.}}
\newcommand\etc{\textit{etc.}}
\newcommand\doubleE{\mathbb{E}}
\newcommand\doubleP{\mathbb{P}}

\newcommand\doubleR{\mathbb{R}}
\newcommand\scriptR{\mathcal{R}}
\newcommand\scriptS{\mathcal{S}}
\newcommand\scriptE{\mathcal{E}}
\newcommand\scriptO{\mathcal{O}}
\newcommand\scriptT{\mathcal{T}}
\newcommand\scriptA{\mathcal{A}}
\newcommand\scriptB{\mathcal{B}}

\newtheorem*{theorem*}{Theorem}
\newtheorem{theorem}{Theorem}
\newtheorem{fact}{Fact}
\newtheorem{definition}{Definition}
\newtheorem*{proposition*}{Proposition}
\newtheorem{proposition}{Proposition}
\newtheorem{corollary}{Corollary}
\newtheorem*{corollary*}{Corollary}

\title{
{META-Learning Eligibility Traces for More Sample Efficient Temporal Difference Learning}\\~\\~\\
{\large Mingde Zhao, School of Computer Science \\ 
	McGill University, Montr\'eal \\ 
	Apr, 2020 \\~\\~\\
	A thesis submitted to McGill University in partial fulfillment of the requirements of the degree of \\~\\ Master of Computer Science }\\~\\
}

\author{\textcopyright Mingde Zhao, 2020}
\date{}

\begin{document}
\maketitle

\chapter*{Abstract}
\label{sec:engAbstract}
\addcontentsline{toc}{section}{\nameref{sec:engAbstract}}
Temporal-Difference (TD) learning is a standard and very successful reinforcement learning approach, at the core of both algorithms that learn the value of a given policy, as well as algorithms which learn how to improve policies. TD-learning with eligibility traces provides a way to do temporal credit assignment, \ie{} decide which portion of a reward should be assigned to predecessor states that occurred at different previous times, controlled by a parameter $\lambda$. However, tuning this parameter can be time-consuming, and not tuning it can lead to inefficient learning. To improve the sample efficiency of TD-learning, we propose a meta-learning method for adjusting the eligibility trace parameter, in a state-dependent manner. The adaptation is achieved with the help of auxiliary learners that learn distributional information about the update targets online, incurring roughly the same computational complexity per step as the usual value learner. Our approach can be used both in on-policy and off-policy learning. We prove that, under some assumptions, the proposed method improves the overall quality of the update targets, by minimizing the overall target error. This method can be viewed as a plugin which can also be used to assist prediction with function approximation by meta-learning feature (observation)-based $\lambda$ online, or even in the control case to assist policy improvement. Our empirical evaluation demonstrates significant performance improvements, as well as improved robustness of the proposed algorithm to learning rate variation.

\chapter*{Abrégé}
\label{sec:frAbstract}
\addcontentsline{toc}{section}{\nameref{sec:frAbstract}}
L'apprentissage par différence temporelle (TD) est une approche d'apprentissage par renforcement standard et très réussie, au cœur des deux algorithmes qui apprennent la valeur d'une politique donnée, ainsi que des algorithmes qui apprennent à améliorer les politiques. L'apprentissage TD avec des traces d'éligibilité fournit un moyen de faire une attribution de crédit temporelle, \ie{} décide quelle portion d'une récompense doit être affectée aux états prédécesseurs qui se sont produits à différents moments précédents, contrôlé par un paramètre $\lambda$. Cependant, le réglage de ce paramètre peut prendre du temps et ne pas le régler peut conduire à un apprentissage inefficace. Pour améliorer l'efficacité de l'échantillon d'apprentissage TD, nous proposons une méthode de méta-apprentissage pour ajuster le paramètre de trace d'éligibilité, d'une manière dépendante de l'état. L'adaptation est réalisée avec l'aide d'apprenants auxiliaires qui apprennent en ligne les informations de distribution sur les cibles de mise à jour, entraînant à peu près la même complexité de calcul par étape que l'apprenant de valeur habituelle. Notre approche peut être utilisée à la fois dans l'apprentissage sur les politiques et hors politique. Nous prouvons que, sous certaines hypothèses, la méthode proposée améliore la qualité globale des cibles de mise à jour, en minimisant l'erreur cible globale. Cette méthode peut être considérée comme un plugin qui peut également être utilisé pour aider à la prédiction avec l'approximation des fonctions par la fonction de méta-apprentissage (observation) basée sur $\lambda$ en ligne, ou même dans le cas de contrôle pour aider à l'amélioration des politiques. Notre évaluation empirique démontre des améliorations significatives des performances, ainsi qu'une meilleure robustesse de l'algorithme proposé à la variation du taux d'apprentissage.

\chapter*{Acknowledgements}
\label{sec:acknowledgements}
\addcontentsline{toc}{section}{\nameref{sec:acknowledgements}}
I want to thank Prof. Doina Precup and Prof. Xiao-Wen Chang for funding and supervising me during my master's studies. I also want to thank my colleagues and friends for supporting my studies and my life, especially Sitao Luan and Ian Porada, who are not only good collaborators but also cheerful companions. I would also like to express my sincere gratitude to the people that shared their ideas with me and inspired me to go beyond!

\tableofcontents
\listoffigures %
\addcontentsline{toc}{section}{\listfigurename}
\listoftables
\addcontentsline{toc}{section}{\listtablename}

\newpage

\pagenumbering{arabic} 


\chapter{Introduction}
\label{sec:intro}
Temporal-difference learning is an important approach which enables an agent interacting with its environment to learn how to assign credit to different states it encounters and actions it takes.
Eligibility trace-based policy evaluation (prediction) methods, \eg, TD($\lambda$), use geometric sequences, controlled by a ``trace-decay'' parameter $\lambda$, to solve the temporal credit assignment problem. Eligibility traces weight multi-step returns and assemble compound update targets. The sample complexity (speed and accuracy of convergence given the number of samples) is in practice sensitive to the choice of $\lambda$.

To address this, in this thesis, we propose meta-learning as an approach for adapting the learning \textit{state-based} $\lambda$s. First, we propose the methodology of improving sample efficiency by improving the quality of the update targets during temporal difference learning. Then, we derive the method of achieving improvement on the overall update targets for each state. Finally, we propose an  approximate way to implement the method in an online learning setting, with the help of auxiliary learners and trust region-style updates.

The thesis is structured as followed
In Chapter 2, we introduce  the fundamentals of reinforcement learning that are \textit{directly related} to this thesis. The content is fully re-written for the state-based decay and discount settings that we consider in this thesis. The content of the thesis is topologically sorted so that all mentioned methods and settings could be found in this chapter. Then, in Chapter 3, a comprehensive review of the relevant literature, which focuses on the adaptation of trace-based temporal difference learning, is provided. Chapter 4 and 5 contain our main contribution: a meta-learning approach is proposed, discussed and then validated through experiments. Chapter 6 concludes the thesis and discusses avenues for future work.


\input{chapter_2_basics.tex}

\input{chapter_3_preliminary.tex}

\input{chapter_4_algorithm.tex}

\input{chapter_5_experiments.tex}

\input{chapter_6_conclusion.tex}

\bibliography{references}
\bibliographystyle{acm}

\input{chapter_appendices.tex}

\end{document}

%% file: chapter_2_basics.tex
\chapter{Basics of Reinforcement Learning}
\textit{A learning paradigm of learning by trial and error. Interacting with an environment to map situations to actions in a way that some notion of cumulative reward is maximized.}

\section{What is Reinforcement Learning?}
Learning from interaction is one of the essential and fundamental ideas of intelligence and learning theories \cite{sutton2018reinforcement}. Reinforcement Learning (RL) was introduced by A. Harry Klopf as a computational approach that focused on learning by interacting with the environment without explicit supervision.

In the RL setting, there is an \textit{agent} (the algorithm or the method), an environment (essentially a task that we are facing, a probabilistic system). The environment is an ensemble of a reward function (a scalar feedback for the decisions) and state dynamics (transition probabilities of states by actions). According to the state, the agent computes a series of action to be taken and in this way interacts with the environment.

From a ``dataset''-label perspective, we could say that in RL problems, the data samples are dynamically collected by the agents' decisions series. This also means the quality of the ``dataset'' is also determined by the quality of the decisioning processes of the agents. Therefore, RL is, like many problems, a exploration-exploitation tradeoff problem with no static datasets and thus it stands out from the classical machine learning paradigms. The learner is not told which actions to take but instead must discover which actions yield the most reward by trying them. 


Another distinct feature of RL is that RL agents learn what to do - how to map situations to actions - all by the interactions after its deployment in the environment, \ie{} it does not need domain specific knowledge. 


\section{Basics}
Now, we formally identify the main subelements of a RL system: a policy, a reward signal, a value function and optionally a model of the environment.

A \textit{policy} defines the learning agent’s way of behaving for a given state at a given time. Roughly speaking, a policy is a mapping from perceived states of the environment to actions to be taken when in those states.

A \textit{reward} signal defines the goal of a RL problem. On each timestep, the environment sends to the RL agent a scalar feedback called the \textit{reward}. The agent’s sole objective is to maximize the total reward it receives over the long run. The reward signal thus defines what are the good and bad events for the agent.

A \textit{value function} specifies what is good in the long run, whereas the reward signal indicates what is good in an immediate sense. Roughly speaking, the value of a state is the total amount of reward an agent can expect to accumulate over the future, starting from that state. Whereas rewards determine the immediate, intrinsic desirability of environmental states, values indicate the long-term desirability of states after taking into account the states that are likely to follow and the rewards available in those states. Rewards are in a sense primary, whereas values, as predictions of rewards, are secondary. Without rewards there could be no values, and the only purpose of estimating values is to achieve more reward.

The fourth and final element of some reinforcement learning systems is a model of the environment. This is something that mimics the behavior of the environment, or more generally, that allows inferences to be made about how the environment will behave. For example, given a state and action, the model might predict the resultant next state and next reward. Models are used for planning, by which we mean any way of deciding on a course of action by considering possible future situations before they are actually experienced. Methods for solving reinforcement learning problems that use models and planning are called model-based methods, as opposed to simpler model-free methods that are explicitly trial-and-error learners—viewed as almost the opposite of planning.

\section{Scope of Reinforcement Learning}
RL relies heavily on the concept of \textit{state} — as input to the policy and value function, and as both input to and output from the model. We can think of the state as a signal conveying to the agent some sense of ``how the environment is'' at a particular time. The formal definition of state as we use it here is given by the framework of Markov decision processes.

\subsection{Markov Decision Processes}
Markov Decision Processes (MDPs), a mathematically idealized form of the RL problem for which precise theoretical statements can be made, are a classical formalization of sequential decision making. In MDPs, actions influence not just immediate rewards, which are essentially some feedback from the environment, but also subsequent situations, or states, and through the future rewards. Thus MDPs involve delayed reward and the need to tradeoff immediate and delayed reward.

The learner and decision maker is called the \textit{agent}. The thing it interacts with, comprising everything outside the agent, is called the \textit{environment}. They interact continually, the agent selecting actions and the environment responding to these actions and presenting new situations to the agent. The environment also gives \textit{rewards}, special numerical values which the agent seeks to maximize over time through its choice of actions.

As illustrated in Figure \ref{fig:agent_environment_interaction}, in an MDP, the agent and the environment interact at each of a sequence of discrete timesteps $t\in\{ 0, 1, 2, \dots \}$\footnote{The ideas of the discrete time case can be extended to the continuous-time case, \eg{} \cite{doya2000reinforcement,bertsekas1996neuro}.}. At each timestep $t$, the agent receives some representation of the environment's state, $S_t \in \scriptS{}$ and on that basis selects an action, $A_t \in \scriptA{}$. One timestep later, in part as a consequence of its action, the agent receives a numerical reward, $R_{t+1} \in \scriptR{} \in \doubleR{}$ and finds itself in a new state $S_{t+1}$. In a finite MDP, the state set $\scriptS$, the action set $\scriptA$ and the reward set $\scriptR$ are all finite. In this case, the random variables $R_t$ and $S_t$ have well defined discrete probability distributions that only depend on the preceding state and action. That is, for particular values of these random variables, $s' \in \scriptS$ and $r \in \scriptR$, there is a probability of those values occurring at timestep $t$, given particular preceding state and action:
$$p(s',r|s,a) \equiv \doubleP\{ S_t = s', R_t = r | S_{t-1} = s, A_{t-1} = a \}, \forall s, s' \in \scriptS, \forall r \in \scriptR, \forall a \in \scriptA$$

The function $p: \scriptS \times \scriptR \times \scriptS \times \scriptA \to [0,1]$ defines the \textit{dynamics} of the MDP, and is often recognized as the \textit{transition probability function} or simply \textit{transition function}.

The MDP and the agent together thereby give rise to a \textit{trajectory} like:

$$S_0,A_0,R_1,S_1,A_1,R_2,\dots$$

Note that we have used the uppercase letters to denote the random variables since we have not yet observed the states, the rewards or the actions. Yet if we have already, we would use lowercase to denote their specific instantiation. For example, at timestep $t$, the agent took action $a_t$ based on state $s_t$ and transitioned to the state $s_{t+1}$ while receiving the reward $r_{t+1}$.

The states that the agent start from in a finite MDP can be described using a probability distribution.

\begin{definition}[Initial State Distribution]
For a finite Markov decision process, the distribution $d_0: \scriptS \to [0, 1]$ of the first state $S_0$, from which the agent-environment interactions start, is called the \textit{starting state distribution} or \textit{initial state distribution}.
\end{definition}

For an RL agent, the distribution $d_0$ is part of the unknown environment and can be only learnt through interactions.

\subsubsection{Markov Property}
In an MDP, the probabilities given by $p$ completely characterize the environment's dynamics. That is, the probability of each possible value for $S_t$ and $R_t$ depends only on the immediately preceding state and action $S_{t-1}$ and $A_{t-1}$ not at all on earlier states and actions. This is best viewed a restriction not on the decision process but on the state, which means the state must include information about all aspects of the past agent-environment interaction that make a difference for the future.

The $4$-argument deterministic transition function is actually the most general form of a transition function defined in a MDP. There are also alternate forms of the transition function, which rely on either additional assumptions of the environment or exist as marginalized expectations.

Marginalizing over rewards yields the $3$ argument state-transition probability function:
$$p(s'|s,a) \equiv \doubleP\{ S_t = s' | S_{t-1} = s, A_{t-1} = a \} = \sum_{r \in \scriptR{}}{p(s',r|s,a)}$$
where $p$ is overloaded. However, if we assume that the state transition and the reward are jointly determined, \ie{} a fixed transition from one state to another always generates the same reward, then the $4$-argument transition function collapse into the $3$-argument version $p: \scriptS \times \scriptS \times \scriptA \to \times [0,1]$.

Many other useful expected statistics can be derived from the general $4$-argument $p$ by marginalizing. These include:

Expected rewards for state-action pairs as a $2$-argument function $r: \scriptS \times \scriptA \to \doubleR$:
$$r(s,a) \equiv \doubleE[ R_t | S_{t-1} = s, A_{t-1} = a ] = \sum_{r \in \scriptR{}}{r \sum_{s' \in \scriptS{}}{ p(s',r|s,a) }} $$
Since the only way that the agent could interact with the environment is through the action, there is no way for the agent to optimize the transition and reward by any other means, this $2$-argument expected reward function should be an appropriate choice when the agent tries to model the reward function for decisioning, through agent-environment interactions.

Expected rewards for state-action-next-state triples as a three argument function $r: \scriptS \times \scriptA \times \scriptS \to \doubleR$:

$$r(s,a,s') \equiv \doubleE[ R_t | S_{t-1} = s, A_{t-1} = a, S_{t} = s' ] = \sum_{r \in \scriptR{}}{r \frac{p(s',r|s,a)}{p(s'|s,a)}} $$

where $p(s'|s,a)$ is the $3$-argument transition function we derived earlier. This function can be estimated to predict the reward incurred by some certain transition, which is often used in model-based RL.

\begin{figure*}
\centering
\includegraphics[width=0.65\textwidth]{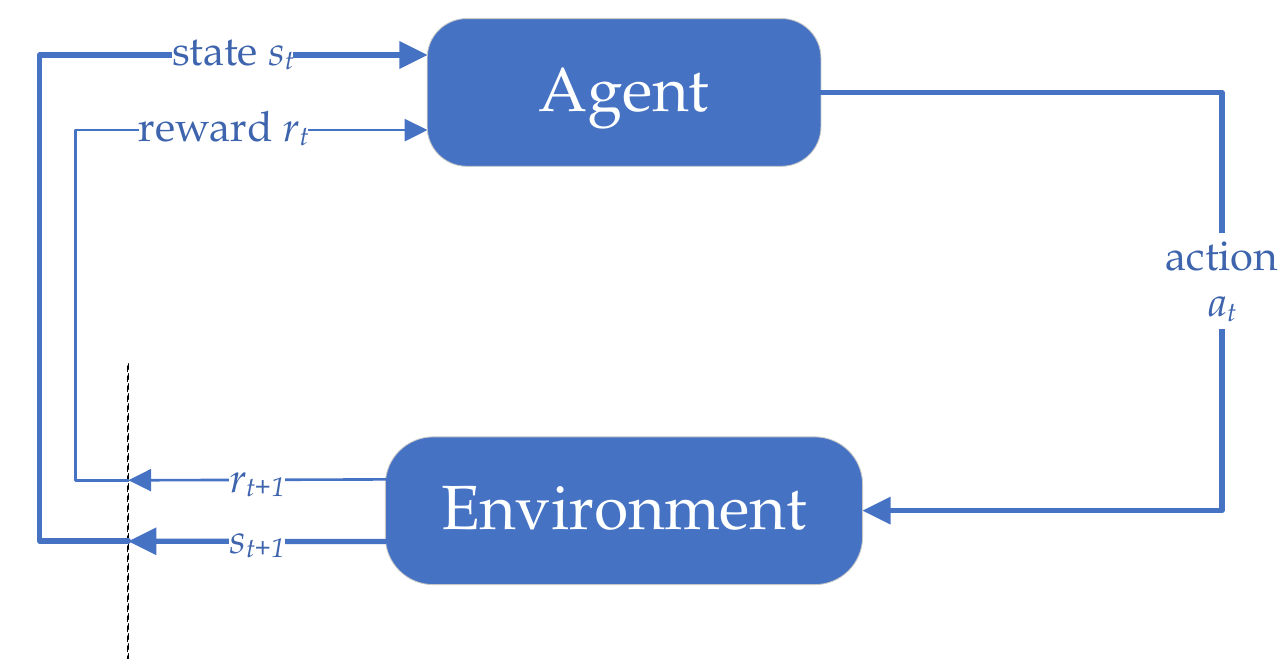}
\caption{Agent-Environment Interaction in a Markov Decision Process}
\label{fig:agent_environment_interaction}
\end{figure*}

\subsubsection{Implicit Pause-able Environment Assumption}
MDP implicitly assumes that the environment only changes when the agent is taking actions, whereas this assumption seems inappropriate in many decision problems and in the cases where the decisioning process takes in-negligible time, \eg{} when the agent is planning with a complex model.

Very recently, the notion of realtime reinforcement learning has been introduced in \cite{ramstedt2019real} to address such problem.

\subsection{Returns \& Episodes}

\subsubsection{Undiscounted Episodic Setting}
In general, RL seeks to maximize the expected return $G_t$, which, in the simplest case, is defined as sum of the rewards:

\begin{equation}
\label{eq:undiscounted_return}
G_t \equiv R_{t+1} + R_{t+2} + \cdots + R_{T}
\end{equation}

where $T$ is a final timestep. The notion of \textit{episodes} is naturally formed as the interaction sequence from the starting timestep $0$ until the terminal time $T$, a random variable that normally varies, for example when playing a game repeatedly.

Each episode ends in a \textit{terminal state}. In this setting, the episodes are assumed to be finite-length, \ie{} $T$ is finite, and independent with each other, \ie{} one episode does not affect the environment dynamics of the next.

In episodic tasks, it is sometimes necessary to distinguish the set of all nonterminal states, denoted $\scriptS$ from the set of all states plus the terminal states $\scriptS^{+}$.

\subsubsection{Discounted Episodic Setting}
In many cases however, the agent-environment interaction does not break naturally into identifiable episodes, but goes on continually without limit. The previous formulation of return is problematic in these \textit{continuing tasks} because not only the final step would be $T = \infty$, as well as that the return, which we seek to optimize, goes easily to infinite.

To fix this, an additional concept of \textit{discounting} is needed. According to this approach, the agent tries to select actions so that the sum of the discounted rewards it receives over the future is maximized. In particular, it chooses to maximize the \textit{expected discounted return}:

\begin{equation}
\label{eq:discounted_return}
G_t \equiv R_{t+1} + \gamma R_{t+2} + \gamma^2 R_{t+3} + \cdots = \sum_{k=0}^{\infty}{\gamma^{k}R_{t+k+1}}
\end{equation}

where $\gamma \in [0, 1]$ is the \textit{discount parameter}, also recognized as the \textit{discount rate}. Note that here the discount parameter is constant throughout the states and the episodes. While other more complicated discount settings are also possible. For example, in this thesis, we will adopt the \textit{per-state discount function}, which can be interpreted as discounting the future rewards by some degree after \textbf{entering} some state:

\begin{equation}
\label{eq:state_discounted_return}
G_t \equiv R_{t+1} + \gamma_{t+1} R_{t+2} + \gamma_{t+1}\gamma_{t+2} R_{t+3} + \cdots = \sum_{k=0}^{\infty}{R_{t+k+1} \cdot \prod_{m=1}^{k}{\gamma_{t+m}}}
\end{equation}

where $\gamma_{t+1} \equiv \gamma(S_{t+1})$, $\gamma: \scriptS{} \to [0, 1]$.

The discount rate determines the present value of the future rewards: a reward received $k$ timesteps in the future is worth only $\prod_{m=1}^{k-1}{\gamma_{t+m}}$ times what it  would be worth if it were received immediately.

If $\gamma: \scriptS{} \to [0, 1)$, the infinite sums in \ref{eq:state_discounted_return} will have finite values as long as the reward sequence $\{ R_{k} \}$ is bounded.

Returns at successive timesteps are recursively related:


\begin{equation}
\label{eq:recursive_return_general}
\begin{aligned}
G_t & \equiv R_{t+1} + \gamma_{t+1} R_{t+2} + \gamma_{t+1}\gamma_{t+2} R_{t+3}  + \cdots\\
    & = R_{t+1} + \gamma_{t+1} \left(R_{t+2} + \gamma_{t+2} R_{t+3} + \cdots \right)\\
    & = R_{t+1} + \gamma_{t+1} G_{t+1}
\end{aligned}
\end{equation}

This relation is generalizable for all timesteps $t<T$.

The fact that the discount parameters can be set as $1$ gives unified formulation for the episodic tasks as well as the continuing tasks. Ideally, the discount factors should come from the task itself, as its value reflects the goal. However, for complicated tasks with Deep Reinforcement Learning (DRL), which is essentially using artificial neural networks for reinforcement learning, it is generally observed that lowering the discount factor yields significantly more stable performance rather than using $\gamma = 1$, even if the goal tells that there should be no discounting \cite{mnih2015human}. These blur the line how we should see the discount factor, which classically should be seen as some kind of built-in characteristics of the environment yet now a parameter that could be set or learnt for some purposes.

In the following parts of this thesis, we will focus on the episodic setting with state-based discount functions.

\subsection{Policies \& Value Functions}

A \textit{policy} is a function used by an agent to decide what to do given the state it is in. Formally,

\begin{definition}[policy]
In an MDP, a policy $\pi$ is a mapping from states to probabilities of selecting each possible action, \ie{} $\pi: \scriptS{} \times \scriptA{} \to [0, 1]$.
\end{definition}

If the agent is following policy $\pi$ at timestep $t$, then $\pi(a|s)$ is the probability that $A_t = a$ if $S_t = s$. Note that it is also quite common to define state-based action sets, however we will stick loyal to the setting of \cite{sutton2018reinforcement} for this thesis. The policies in RL is essentially stationary decision rules defined for more general Markov chains \cite{puterman2014markov}, where ``stationary'' means that the decision rules are consistent for every possible states.

The \textit{value function} of a state $s$ under a policy $\pi$, denoted $v_{\pi}(s)$, is the expected (discounted) return when starting in $s$ and following $\pi$ thereafter. Formally,

\begin{definition}[state-value function]
In an MDP, the \textit{state-value function} for policy $\pi$ or simply \textit{value function} $v_\pi(s)$, given discount function $\gamma(s)$ and policy $\pi(s)$, is defined as
$$v_{\pi}(s) \equiv \doubleE_{\pi}[ G_t | S_t = s ] = \doubleE_{\pi} \left[ \sum_{k=0}^{\infty}{R_{t+k+1} \cdot \prod_{m=1}^{k}{\gamma_{t+m}}} | S_t = s \right]$$
\end{definition}
where $\doubleE_{\pi}$ denotes the expected value of random variable given that the agent follows policy $\pi$ and the values of the terminal states are defined as $0$.

There is also the state-action value function, which is more useful when searching for policies.

\begin{definition}[state-action-value function]
In an MDP, the \textit{state-action-value for policy $\pi$} $q_\pi(s)$, given discount function $\gamma$ and policy $\pi(s)$, is defined as the expected return starting from $s$, taking the action $a$ and thereafter following $\pi$:
$$q_{\pi}(s, a) \equiv \doubleE_{\pi}[ G_t | S_t = s, A_t = a ] = \doubleE_{\pi} \left[ \sum_{k=0}^{\infty}{R_{t+k+1} \cdot \prod_{m=1}^{k}{\gamma_{t+m}}} | S_t = s, A_t = a \right]$$
\end{definition}

One of the key subroutines of RL is to estimate $v_\pi$ or $q_\pi$ from experience, as $V_\pi$ or $Q_\pi$, which is often recognized as \textit{policy evaluation} or \textit{prediction}.

With the recursive relationship derived in \ref{eq:recursive_return_general}, we can obtain the following equation for the state-value function $v_\pi$, given a policy $\pi$:

\begin{equation}
\label{eq:bellman_v_general}
\begin{aligned}
v_{\pi}(s) & \equiv \doubleE_{\pi}[ G_t | S_t = s ]\\
& = \doubleE_{\pi}[ R_{t+1} + \gamma_{t+1} G_{t+1}]\\
& = \sum_{a}{\pi(a|s) \sum_{s',r}{p(s',r|s,a)\left[ r + \gamma(s') \cdot \doubleE_{\pi}[G_{t+1}|S_{t+1}=s'] \right]}}\\
& = \sum_{a}{\pi(a|s) \sum_{s',r}{p(s',r|s,a)\left[ r + \gamma(s') \cdot v_{\pi}(s')\right]}}
\end{aligned}
\end{equation}

where $\gamma_{t+1} \equiv \gamma(S_{t+1})$ is a conventional abbreviation, which will be frequently used hereafter. The equation is essentially one basic form of the \textit{Bellman equation} for $v_{\pi}$. Serving as the core equation of RL, it expresses the relationship between the value of one state and its successor states.

\subsubsection{Optimal Policies \& Optimal Value Functions}
Value functions define a partial ordering over policies in an MDP with certain discounting function $\gamma$.

\begin{definition}[partial order of policies]
In an MDP with certain discounting function $\gamma$, a policy $\pi$ is defined to be \textit{better than or equal to} a policy $\pi'$ if its expected return is greater than or equal to that of $\pi'$ for all states, \ie{} $\pi \geq \pi'$ \textit{\textbf{iff}} $\forall s \in \scriptS{}, v_{\pi}(s) \geq v_{\pi'}(s)$.
\end{definition}

There always exists at least one policy that is better than or equal to all other policies, which is identified as the \textit{optimal policy} $\pi_{*}$. More specifically,

\begin{fact}[Existence of Optimal Deterministic Markovian Policies]
Given a finite MDP, there always exist some deterministic policy $\pi_{*}$ that achieves the optimal return, which is independent of the transition history, \ie{} Markovian.
\end{fact}

This fact is the reason why we could have confidently focus on the world of Markovian policies, since an optimal policy always exist inside. This also justifies the value-based methods that uses $\epsilon$-greedy as policies since deterministic policies can achieve optimal return. There can be more than one optimal policies but we denote all of them by $\pi_{*}$. The optimal policies share the same state-value function, which is called the \textit{optimal state-value function} $v_{*}$, which is defined as:

$$v_{*}(s) \equiv \max_{\pi}{v_{\pi} (s)}, \forall s \in \scriptS{}$$

where the $\max$ operator is defined upon the policy partial orders.

Optimal policies also share the same optimal action-value function $q_{*}$, which is defined as

$$q_{*}(s,a) \equiv \max_{\pi}{q_{\pi} (s, a)}, \forall s \in \scriptS{}, \forall a \in \scriptA{}$$

The function gives the expected return for taking action $a$ in state $s$ and thereafter following an optimal policy. Thus,

$$q_{*}(s,a) = \doubleE \left[ R_{t+1} + \gamma(S_{t+1}) \cdot v_{*}(S_{t+1}) | S_t = s, A_t = a \right]$$

There are also Bellman equations for optimal policies, which are recognized as \textit{Bellman optimality equations}.

The methods for finding better policies, are recognized as \textit{policy improvement} methods. A representative family of such methods, which are called policy gradient methods, will be introduced in Section \ref{sec:policy_gradient}.

\section{Dynamic Programming}\label{sec:dynamic_programming}

Dynamic Programming (DP) refers to a family of algorithms that can be used for solving the value function given a policy, or finding better, even the optimal policies given a perfect model of the environment in the form of an MDP.

Though guaranteed with optimality, DP algorithms are limited for practical use for the their need of a perfect environment model as well as their expensive computational cost. However, they are still widely used to calculate the ground truth values of relatively small environments for the analyses of new RL methods.

\subsection{DP for Policy Evaluation}
First, we consider the method of computing the state-value function $v_{\pi}$ for an arbitrary policy $\pi$, which is called \textit{policy evaluation} or \textit{prediction}.

\begin{fact}[Existence \& Uniqueness of State-Value Function]
The existence and the uniqueness of state-value function $v_{\pi}$ are guaranteed as long as either $\gamma(\cdot) < 1$ or termination will be reached from any state following $\pi$.
\end{fact}

From the Bellman equation (\ref{eq:bellman_v_general}), we have

$$v_{\pi}(s) = \sum_{a}{\pi(a|s) \sum_{s',r}{p(s',r|s,a)\left[ r + \gamma(s') \cdot v_{\pi}(s')\right]}}$$

Suppose for a specific MDP, we have a fixed indexing method for all the possible states. When the environment dynamics ($4$-argument $p$) is known, the Bellman equation gives a system of $|\scriptS|$ linear equations with $|\scriptS|$ unknowns, \ie{} solvable. It is desirable to transform it to matrix form and then compute the true values. After some algebraic manipulation, we end up in the following system.

$$\bm{v}_{\pi} = \bm{r}_{\pi} + P_{\pi} \bm{v}_{\pi} \Gamma$$

where $\bm{r}_{\pi}$ is a $|\scriptS{}| \times 1$ vector in which $\bm{r}_{\pi}[i] = \sum_{a}{\pi(a|s_i) \sum_{s_j, r}{r \cdot p(s_j, r | s_i, a)}}$, $P$ is a $|\scriptS{}| \times |\scriptS{}|$ matrix in which $P_{\pi}[i,j] = \sum_{a}{\pi(a|s_i) p(s_j, r | s_i, a)}$ and $\Gamma$ is a diagonal matrix whose $\Gamma(i,i) \equiv \gamma(s_i)$. With the system, the rest is just to solve it, whose complexity is no lower than $\scriptO({|\scriptS{}|}^3)$. The method is elaborated in Algorithm \ref{alg:PEM}.

\begin{algorithm*}[htbp]
\caption{Policy Evaluation in Matrix Form}
\label{alg:PEM}
\KwIn{$p(s', r | s, a)$ (transition probability function), $\pi$ (policy to be evaluated), $\gamma$ (discount function)}
\KwOut{$\bm{v}_{\pi}$ (the state value vector for policy $\pi$, where the components correspond to indexed states in $\scriptS{}^{+}$)}

Form policy-conditioned transition matrix $P_{\pi}$, where $P_{\pi}[i,j] = \sum_{a}{\pi(a|s_i) p(s_j, r | s_i, a)}$\\
Form policy-conditioned reward vector $\bm{r}_{\pi}$, where $\bm{r}_{\pi}[i] = \sum_{a}{\pi(a|s_i) \sum_{s_j, r}{r \cdot p(s_j, r | s_i, a)}}$\\
Form a diagonal matrix $\Gamma$, where $\Gamma(i,i) \equiv \gamma(s_i)$\\
$\bm{v}_{\pi} = (I - P_{\pi} \Gamma) \backslash \bm{r}_{\pi}$ \textcolor{darkgreen}{// or $\bm{v}_{\pi} = (I - P_{\pi} \Gamma)^{-1} \bm{r}_{\pi}$}\\
\end{algorithm*}

The $\scriptO({|\scriptS|}^3)$ complexity is a nightmare for problems with large state sets. Thus, it is desirable to change this method into methods with lower computational complexity. The tool for this conversion is the matrix splitting methods.

Matrix splitting is a method for converting the problems of solving linear equations into iterative methods with presumably lower computational complexity.

\begin{fact}[Matrix Splitting]
To solve linear equation $A\bm{x} = \bm{b}$, where $A$ is non-singular, we can split matrix $A$ as $A = M - N$, where $M$ and $N$ are greater or equal to $0$ element-wise. Then, the formula $$\bm{x}_{t+1} = M^{-1} \bm{b} + M^{-1} N \bm{x}_t$$
leads to the solution if the spectral radius $\rho(M^{-1}N)$, \ie{} the largest absolute value of the eigenvalues of $M^{-1}N$, is less than $1$ and the spectral radius is also the convergence rate of the iteration. 
\end{fact}

With this, when we set $M=I$ and $N=P_{\pi} \Gamma$ and $\bm{b} = \bm{r}_\pi$, we achieve the \textit{iterative policy evaluation method}. It is simple and powerful, turning Bellman equation into an iterative formula achieves the convergence to the true values. One may also prove the convergence of iterative policy evaluation using Banach's fixed point theorem, by showing that the Bellman operator is a contraction.

\begin{definition}[Bellman Operator]
Given an MDP with its dynamics $p$, a policy $\pi$ and discount function $\gamma$, the \textit{Bellman operator} $\scriptB_{\pi}: \doubleR^{|\scriptS|} \to \doubleR^{|\scriptS|}$ is defined by
\begin{equation}\label{eq:bellman_operator}
(\scriptB_{\pi}v)(s) \equiv \sum_{a}{\pi(a|s) \sum_{s', r}{p(s',r|s,a)[r + \gamma(s') v(s')]}}
\end{equation}
Or equivalently in matrix form,
\begin{equation}\label{eq:bellman_operator_matrix}
\scriptB_{\pi} \bm{V} \equiv \bm{r}_{\pi} + P_{\pi} \Gamma \bm{V}
\end{equation}
where $\bm{r}_{\pi}$ is a $|\scriptS{}| \times 1$ vector in which $\bm{r}_{\pi}[i] = \sum_{a}{\pi(a|s_i) \sum_{s_j, r}{r \cdot p(s_j, r | s_i, a)}}$, $P$ is a $|\scriptS{}| \times |\scriptS{}|$ matrix in which $P_{\pi}[i,j] = \sum_{a}{\pi(a|s_i) p(s_j, r | s_i, a)}$ and $\Gamma$ is a diagonal matrix whose $\Gamma(i,i) \equiv \gamma(s_i)$.
\end{definition}

\begin{definition}[Contraction]
Let $\langle X, d \rangle$ be a complete metric space. Then a map $\scriptT: X \to X$ is called a \textbf{contraction mapping} on $X$ if there exists $q \in [0, 1)$ \st{}
$$\forall x, y \in X, d(\scriptT(x), \scriptT(y)) \leq q \cdot d(x, y)$$
\end{definition}

\begin{theorem}[Banach–Caccioppoli]\label{thm:fixed_point}
Let $\langle X, d \rangle$ be a non-empty complete metric space with a contraction mapping $\scriptT: X \to X$, then $\scriptT$ admits a \textbf{unique} fixed point $x^* \in X$, \ie{} $\scriptT(x^*) = x^*$. Moreover, $x^*$ can be found as follows: start with an arbitrary element $x_0 \in X$ and define a sequence $\{ x_n \}$ by $x_n = \scriptT(x_{n-1})$ for $n \geq 1$, then $x_n \to x^*$. 
\end{theorem}

We can prove that the Bellman operator, which is essentially turning the Bellman equation into an iterative formula, on the \textit{estimated value function} or simply \textit{value estimate} is a contraction. The unique fixed point must be the true value because that is when the Bellman equation holds. The vanilla version is provided in Algorithm \ref{alg:IPEM}.



\begin{algorithm*}[htbp]
\caption{Iterative Policy Evaluation (Matrix)}
\label{alg:IPEM}
\KwIn{$p(s', r | s, a)$ (transition probability function), $\pi$ (policy to be evaluated), $\gamma$ (discount factor), $\theta$ (accuracy threshold)}
\KwOut{$\bm{v}_{\pi}$ (the state value vector for policy $\pi$, where the components correspond to indexed states in $\scriptS{}^{+}$)}

Initialize $\bm{v}_{|\scriptS{}|\times1} = \bm{0}$\\
Form policy-conditioned transition matrix $P_{\pi}$, where $P_{\pi}[i,j] = \sum_{a}{\pi(a|s_i) p(s_j, r | s_i, a)}$\\
Form policy-conditioned reward vector $\bm{r}_{\pi}$, where $\bm{r}_{\pi}[i] = \sum_{a}{\pi(a|s_i) \sum_{s_j, r}{r \cdot p(s_j, r | s_i, a)}}$\\
$\Delta = \infty$\\
Form a diagonal matrix $\Gamma$, where $\Gamma(i,i) \equiv \gamma(s_i)$\\
\While{$\Delta \geq \theta$}{
    $\bm{v}' = \bm{r}_{\pi} + P_{\pi} \Gamma \bm{v}$\\
    $\Delta = \| \bm{v}' - \bm{v} \|_{\infty}$\\
    $\bm{v} = \bm{v}'$
}
\end{algorithm*}

\subsection{DP for State Distribution}
Given a policy, the expected frequency of states that the agent meet in the environment, which is recognized as the \textit{state frequency},  or the \textit{on-policy distribution}, can also be computed using DP. Note that, given a fixed policy, this frequency is neither dependent on discount function nor the reward function. 

In RL literature, the state on-policy distribution is often assumed to be equal to the ``stationary distribution''\footnote{This ``stationary distribution'' is actually not a stationary distribution, as the terminal states of MDP eliminates the ergodicity. The ``stationary distribution'' $\bm{d}$ in the episodic setting means only the solution of $\bm{d}^T P_{\pi} = \bm{d}^T$.} of the Markov chain induced by the MDP and the policy $\pi$. However, this assumption is inappropriate because they are never the same in the episodic setting.

One can use an iterative method to solve the true on-policy distribution, given initial state distribution $d_0$ and dynamics $P_{\pi}$. The idea is to calculate the averaged number of state visits for agents within their lifetime, as shown in Algorithm \ref{alg:IDPSF}.

\begin{algorithm*}[htbp]
\caption{Iterative DP for State Frequencies in Matrix Form}
\label{alg:IDPSF}
\KwIn{$p(s', r | s, a)$ (transition probability function), $\pi$ (policy to be evaluated), $\theta$ (accuracy threshold), $\bm{d}_0$ (initial distribution of states, represented with an indexed vector whose components are the discrete probabilities)}
\KwOut{$\bm{d}_{\pi}$ (expected frequencies of states under policy $\pi$, where the components correspond to indexed states in $\scriptS{}^{+}$)}

Initialize $\bm{d}_{\pi} = \bm{d}_0$\\
Form policy-conditioned transition matrix $P_{\pi}$, where $P_{\pi}[i,j] = \sum_{a}{\pi(a|s_i) p(s_j, r | s_i, a)}$\\
$\bm{d}' = \bm{d}_0$\\
\For{$i \in \{ 1, \dots, |\scriptS| \}$}{
    \If{$s_i$ is terminal}{
        $\bm{d}'(i) = 0$ \textcolor{darkgreen}{// terminal states can only be visited once}\\
    }
}
\While{$\|\bm{d}'\|_1 \geq \theta$}{
    $\bm{d}' = P_{\pi}^T \bm{d}'$\\
    $\bm{d}_{\pi} = \bm{d}_{\pi} + \bm{d}'$\\
    \For{$i \in \{ 1, \dots, |\scriptS| \}$}{
        \If{$s_i$ is terminal}{
            $\bm{d}'(i) = 0$
        }
    }
}
$\bm{d}_{\pi} = \bm{d}_{\pi} / \|\bm{d}_{\pi}\|_1$ \textcolor{darkgreen}{//normalize}\\
\end{algorithm*}

To achieve a more compact form of this algorithm, we need the following proposition:
\begin{theorem}
In an MDP, the expected state frequency of an agent during one lifetime (from the beginning of an episode to the end) is the same as the expected state frequency of an agent that is redeployed to the MDP after termination.

Also, the state-frequency $\bm{d}_{\pi}$ satisfies:
\begin{equation}
\bm{d}_{\pi}^T \tilde{P}_{\pi} = \bm{d}_{\pi}^T
\end{equation}
where $\tilde{P}_{\pi}$ is the matrix that replaces the rows corresponding to the terminal states in $P_{\pi}$ with $\bm{d}_0^T$.
\end{theorem}

\begin{proof}
Redeploying a terminated agent back to the MDP is upon the initial state distribution. Thus the portion of terminated agents will follow exactly the one-life state frequencies in each of their following lives.
\end{proof}

With this we have a compact matrix-form algorithm and can spot the distribution mismatch more easily. Having spotted the problem, in this thesis, we stick to the more accurate definition of on-policy distribution.

Note that to use this augmented $\tilde{P}_{\pi}$ for policy evaluation, we need to replace the core operation $\bm{V} = \bm{r}_\pi + P_\pi \Gamma \bm{V}$ with $\bm{V} = \bm{r}_\pi + \tilde{P}_\pi \tilde{\Gamma} \bm{V}$, where $\tilde{\Gamma}$ satisfies that the discount for terminal states are $0$\footnote{We think it is more appropriate to call $P_\pi$ a ``value'' transition matrix, and $\tilde{P}_\pi$ the ``state'' transition matrix.}, as suggested in \cite{white2016unifying}.

Also, in the linear case for projected Bellman operator, which we will discuss in future sections, the contraction of policy evaluation heavily relies on the assumption that the ``stationary distribution'' is the on-policy distribution, which is critical for a required lemma\footnote{$\| P_\pi \bm{z} \|_{\bm{d}} \leq \| \bm{z} \|_{\bm{d}}$ if $\bm{d}^T P_{\pi} = \bm{d}^T$}. We will try to re-prove it with the correct setting\footnote{$\| P_\pi \bm{z} \|_{\bm{d}} \leq \| \bm{z} \|_{\bm{d}}$ if $\bm{d}^T \tilde{P}_{\pi} = \bm{d}^T$}. Please refer to \cite{bertsekas1995dynamic} theorem 6.3.1 for more details.

\subsection{DP for Policy Improvement}
DP can also be used to find better policies, which ultimately leads to the optimal policies. Finding better policies requires \textit{policy iteration}, which is essentially alternating policy evaluation and policy improvement. The related details will not be covered since they are too distantly related to this thesis.

\section{Methods without Perfect Models}
DP can solve the value function perfectly and improve the policy only when the perfect model of the environment is given. When not, other methods are needed. Learning from actual experience is striking because it requires no prior knowledge of the environment’s dynamics, yet can still attain optimal behavior.

\subsection{Monte-Carlo Methods for Episodic Tasks}
For policy evaluation, the simplest strategy would be to repeatedly utilize the policy in the environment and observe its average performance. Such strategy leads us to the the simplest policy evaluation method named \textit{Monte-Carlo} or simply \textit{MC}, whose effectiveness is backed by the law of large numbers.

Monte Carlo methods are ways of solving the RL problem based on averaging sample returns. MC is well-defined for episodic setting.

\begin{algorithm*}[htbp]
\caption{Episodic First-visit Monte-Carlo Prediction}
\label{alg:MC}
\KwIn{$\pi$ (policy to be evaluated), $\gamma$ (discount function), $N$ (maximum number of episodes)}
\KwOut{$V(s), \forall s \in \scriptS{}^{+}$ (state values for policy $\pi$)}
Initialize $V(s)$ arbitrarily, $\forall s \in \scriptS$\\
$\text{Returns}(s) = \text{an emptylist}$, $\forall s \in \scriptS$\\
\For{$n \in \{ 1, \dots, N \}$}{
    generate an episode following $\pi$: $S_0, A_0, R_1, S_1, A_1, R_2, \dots, S_{T-1}, A_{T-1}, R_{T}$\\
    $G = 0$\\
    \For{$t \in \{ T-1, T-2, \dots, 0 \}$}{
        $G = \gamma{S_{t+1}} G + R_{t+1}$\\
        \If{$S_t \notin \{ S_0, \dots, S_{t-1} \} $}{ \textcolor{darkgreen}{// visit every state only once, prefer the longer sum if many}\\
            Append $G$ to $\text{Returns}(S_t)$\\
        }
    }
}
\For{$s \in \scriptS{}$}{
    $V(s) = \text{mean}(\text{Returns}(s))$
}
\end{algorithm*}

The state-action values can also be predicted using a very similar method. Policy improvement can also be achieved using generalized policy improvement. The details are omitted.

\subsection{Temporal Difference Learning}
Temporal Difference (TD) learning is the most central idea and methodology of RL, which is a combination of MC and DP. Like MC, TD methods can learn directly from raw experience without a model of the environment’s dynamics. Like DP, TD methods \textit{bootstrap}: they update estimates based in part on other learned estimates, without waiting for a final outcome.

Whereas MC must wait until the end of the episode to determine the increment to $V(S_t)$ (only when $G_t$ is known), TD methods need to wait only until the next timestep. At time $t+1$ they
immediately make a useful update using the observed reward $R_{t+1}$ and the estimate $V(S_{t+1})$. The simplest TD method makes the update
\begin{equation}
\label{eq:TD_1}
V(S_t) = V(S_t) + \alpha \left[ R_{t+1} + \gamma(S_{t+1}) V(S_{t+1}) - V(S_{t}) \right]
\end{equation}
immediately on transition to $S_{t+1}$ and receiving $R_{t+1}$. The update rule \ref{eq:TD_1} is called the \textit{$1$-step TD update}, where we recognize $R_{t+1} + \gamma(S_{t+1}) V(S_{t+1}) - V(S_{t})$ as the \textit{($1$-step) TD error} and $R_{t+1} + \gamma(S_{t+1}) V(S_{t+1})$ as the \textit{update target}. Every $1$-step TD update can be understood as: walk towards the update target $R_{t+1} + \gamma(S_{t+1}) V(S_{t+1})$ from the current (estimated) value $V(S_{t})$ with a step length of $\alpha \left[ R_{t+1} + \gamma(S_{t+1}) V(S_{t+1}) - V(S_{t}) \right]$ (decreasing the distance by ratio $\alpha$). Note that the update target $R_{t+1} + \gamma(S_{t+1}) V(S_{t+1})$ is also a random variable. The $1$-step update yields the simplest TD method, which is named TD($0$) and presented in Algorithm \ref{alg:TD0}.

\begin{algorithm*}[htbp]
\caption{Tabular $1$-step TD Policy Evaluation (TD(0))}
\label{alg:TD0}
\KwIn{$\pi$ (policy to be evaluated), $\gamma$ (discount function), $\alpha \in (0, 1]$ (learning rate), $N$ (maximum number of episodes)}
\KwOut{$V(s), \forall s \in \scriptS{}^{+}$ (state values for policy $\pi$)}
Initialize $V(s) = 0$, $\forall s \in \scriptS$\\

\For{$n \in \{1, \dots, N\}$}{
    Initialize $S$\\
    \While{$S$ is not terminal}{
        $A = \text{action given by} \pi(\cdot|s)$\\
        Take action $A$, observe $R, S'$\\
        $V(S_t) = V(S_t) + \alpha \left[ R_{t+1} + \gamma(S_{t+1}) V(S_{t+1}) - V(S_{t}) \right]$\\
        $S = S'$
        $G = \gamma{S_{t+1}} G + R_{t+1}$
    }
}
\end{algorithm*}

Because TD($0$) bases its update in part on an existing estimate, we say that it is a \textit{bootstrapping} method, like DP.

\begin{fact}
Under the episodic setting, given an MDP and a policy $\pi$, either discounted or not, TD($0$) achieves convergence to $v_\pi$ asymptotically.
\end{fact}

\subsection{Multi-Step TD}
Besides the $1$-step TD target, the update target could be many things as long as the convergence to the true value can be guaranteed. In this subsection, we introduce $n$-step TD methods that generalize both TD and MC methods so that one can shift from one to the other smoothly as needed to meet the demands of a particular task. $n$-step methods span a spectrum with MC methods at one end and $1$-step TD methods at the other.

\begin{definition}[$n$-step return]
The $n$-step return is defined as:
\begin{equation}
\label{eq:n_step_return}
G_{t:t+n} = R_{t+1} + \gamma(S_{t+1})R_{t+2} + \cdots + \prod_{k=1}^{n}{\gamma(S_{t+k})} V_{t+n-1}(S_{t+n}), \forall n \geq 1, \forall 0 \leq t \leq T-n
\end{equation}
\end{definition}

The $n$-step return serves as the update target for an \textit{$n$-step TD update}.

\begin{algorithm*}[htbp]
\caption{Tabular $n$-step TD Policy Evaluation (TD($n$))}
\label{alg:TDn}
\KwIn{$\pi$ (policy to be evaluated), $\gamma$ (discount function), $\alpha$ (learning rate), $M$ (maximum number of episodes), $n$ (step parameter)}
\KwOut{$V(s), \forall s \in \scriptS{}^{+}$ (state values for policy $\pi$)}
Initialize $V(s) = 0$, $\forall s \in \scriptS$\\

\For{$m \in \{1, \dots, M\}$}{
    Initialize and store non-terminal $S_0$\\
    $T = \infty$\\
    \For{$t \in \{0, 1, 2, \dots\}$}{
        \If{$t < T$}{
            Take action according to $\pi(\cdot|s)$, observe and store $R_{t+1}, S_{t+1}$\\
            \If{$S_{t+1}$ is terminal}{
                $T = t + 1$\\
            }
        }
        $\tau = t - n + 1$\\
        \If{$\tau \geq 0$}{
            $G = \sum_{i=\tau+1}^{min(\tau+n, T)}{\left(\prod_{k=1}^{i}{\gamma(S_{k})}\right) \cdot R_i}$\\
            \If{$\tau + n < T$}{
                $G = G + \prod_{k=1}^{n}{\gamma(S_{\tau + k})} V(S_{\tau + n})$\\
            }
            $V(S_{\tau}) = V(S_{\tau}) + \alpha [ G - V(S_{\tau}) ] $\\
        }
        \If{$\tau = T - 1$}{
            \textbf{break}
        }
    }
}
\end{algorithm*}

The $n$-step return uses the value function $V_{t+n-1}$ to correct for the missing rewards beyond $R_{t+n}$. The error reduction property of $n$-step returns lead to its convergence under appropriate technical conditions \cite{sutton2018reinforcement}. 

We notice that, when $n$ is set to be the timestep difference between the current timestep and the timestep for the end of the episode, the $n$-step return becomes the MC return. And when $n=1$, the method collapses to TD($0$). This is to say that $n$-step returns, as update targets, generalizes the TD and MC and yields all the shades between them.

\section{Off-Policy Learning Using Importance Sampling}
How can an agent estimate the values of one policy when acting upon another? Let us call the policy to learn about the \textit{target policy}, and the policy used to generate behavior the \textit{behavior policy}. In this case we say that learning is from data ``off'' the \textit{target policy}, and the overall process is termed off-policy learning.

\subsection{Importance Sampling}
Importance sampling is a technique to use a sample of examples from a different distribution to estimate the expectation of some target distribution. It requires the knowledge to explicitly compute the probability of each sample under the two distributions. Let the target distribution be $\pi$ and the distribution that generated the sample be $b$, we have the following.

\begin{definition}[Importance Sampling Estimator]
Suppose that $\pi$ and $b$ are probability density (mass) functions that satisfy $b(x)=0 \implies \pi(x)=0$, \ie{} $\pi$ is absolutely continuous \wrt{} $b$, we define the importance sampling estimator $\hat{\mu}_{IS}$ of $\doubleE[f(x)]$ as:
\begin{equation}\label{eq:is_estimator}
\hat{\mu}_{IS} \equiv \frac{1}{n} \sum_{i=1}^{n}{f(x_i)\rho(x_i)}
\end{equation}
where $\rho(x_i) \equiv \frac{\pi(x_i)}{b(x_i)}, x_i \sim b$ is recognized as an \textit{importance weight function} and its value is recognized as an \textit{importance sampling ratio}.
\end{definition}

The importance sampling estimator has some properties that we need to know.

\begin{theorem}
Let $\mu \equiv \doubleE_{\pi}[f(x)]$,
$$\doubleE_{b}[\hat{\mu}_{IS}] = \mu, {Var}_{b}[\mu_{IS}] = \int{\frac{(f(x)\pi(x) - \mu b(x))^2}{b(x)} dx}$$
\end{theorem}

The proof is straightforward algebra. The unbiasedness shows that importance sampling ratios can be used to estimate the statistics of data even if they are generated using from different sources. However, the variance will bring trouble, if $\pi$ and $b$ are different.

\subsection{Importance Sampling based Off-Policy Learning}
Now we plug the theory of importance sampling into RL. Let the target policy be $\pi$ and the behavior policy be $b$. Given a starting state $S_t$, the probability of the subsequent state-action trajectory, $A_t, S_{t+1}, \dots, S_T$, occurring under $\pi$ is:
\begin{equation}
\nonumber
\begin{aligned}
& \doubleP\{ A_t, S_{t+1}, \dots, S_T |S_t, A_{t: T-1} \sim \pi \}\\
& = \pi(A_t | S_t) p(S_{t+1} | S_t, A_t) \pi(A_{t+1} | S_{t+1}) \cdots p(S_T | S_{T-1}, A_{T-1})
& = \prod_{k = t}^{T-1}{\pi(A_k | S_k) p(S_{k+1} | S_k, A_k)}
\end{aligned}
\end{equation}
where $p$ is the $3$-argument transition function. The relative probability (importance sampling ratio) of the trajectory under the policies $\pi$ and $b$ is:

$$\rho_{t: T-1} \equiv \frac{\prod_{k = t}^{T-1}{\pi(A_k | S_k) p(S_{k+1} | S_k, A_k)}}{\prod_{k = t}^{T-1}{b(A_k | S_k) p(S_{k+1} | S_k, A_k)}} = \prod_{k = t}^{T-1}{\frac{\pi(A_k | S_k)}{b(A_k | S_k)}}$$

The canceling of the terms show that importance sampling ratio of trajectories does not depend on the MDP's dynamics. With this we have

$$\doubleE_b[\rho_{t:T-1}G_t | S_t = s] = v_{\pi}(s)$$

This means that Monte-Carlo method can learn the target values as long as we have the computational access to the importance sampling ratios. We will not cover the details of off-policy MC.

\subsection{Per-Decision Importance Sampling}
The off-policy MC estimator, the unbiased one with high-variance, of return is:

\begin{equation}
\begin{aligned}
\rho_{t:T-1}G_t & = \rho_{t:T-1}\left( R_{t+1} + \gamma_{t+1}R_{t+2} + \cdots + \prod_{k=t+1}^{T-1}{\gamma_k}R_T \right)\\
& = \rho_{t:T-1} R_{t+1} + \gamma_{t+1} \rho_{t:T-1} R_{t+2} + \cdots + \prod_{k=t+1}^{T-1}{\gamma_k} \rho_{t:T-1} R_T
\end{aligned}
\end{equation}

Each sub-term is a product of a random reward and a random importance sampling ratio. For example, the first sub-term is:

$$\rho_{t:T-1}R_{t+1} = \frac{\pi(A_t | S_t)}{b(A_t | S_t)}\frac{\pi(A_{t+1} | S_{t+1})}{b(A_{t+1} | S_{t+1})}\frac{\pi(A_{t+2} | S_{t+2})}{b(A_{t+2} | S_{t+2})}\cdots\frac{\pi(A_{T-1} | S_{T-1})}{b(A_{T-1} | S_{T-1})}R_{t+1}$$

For this term, it is intuitive to see that only $\frac{\pi(A_t | S_t)}{b(A_t | S_t)}$ and $R_{t+1}$ are related, as one can easily show:

$$\doubleE_{b}\left[\frac{\pi(A_k | S_k)}{b(A_k | S_k)}\right] \equiv \sum_{a \sim b}{\frac{\pi(a | S_k)}{b(a | S_k)}} = \sum_{a}{b(a | S_k) \cdot \frac{\pi(a | S_k)}{b(a | S_k)}} = \sum_{a}{\pi(a | S_k)} = 1$$

Thus,

$$\doubleE_b[\rho_{t:T-1} R_{t+1}] = \doubleE_b[\rho_{t:t} R_{t+1}]$$

and also

$$\doubleE_b[\rho_{t:T-1} R_{t+k}] = \doubleE_b[\rho_{t:t+k-1} R_{t+k}]$$

With this we can get another unbiased return estimator, which is recognized as the \textit{per-decision} importance sampling estimator $\tilde{G}_t$ for return:

$$\doubleE_b[\rho_{t:T-1}G_t] = \doubleE_b[\tilde{G}_t]$$

and

\begin{equation}
\begin{aligned}
\tilde{G}_t &\equiv \rho_{t:t} R_{t+1} + \gamma_{t+1} \rho_{t:t+1} R_{t+2} + \cdots + \prod_{k=t+1}^{T-1}{\gamma_k} \rho_{t:T-1} R_T\\
& = \rho_{t} R_{t+1} + \gamma_{t+1} \rho_{t}\rho_{t+1} R_{t+2} + \cdots + \prod_{k=t+1}^{T-1}{\gamma_k} \cdot \prod_{j=t}^{T-1}{\rho_{j}} \cdot R_T &&\text{($\rho_{j} \equiv \frac{\pi(A_j | S_j)}{b(A_j | S_j)}$)}\\
& = \rho_{t} \left( R_{t+1} + \gamma_{t+1} \rho_{t+1} \left( R_{t+2} + \gamma_{t+2} \rho_{t+2} \left( \cdots \right) \right)  \right)
\end{aligned}
\end{equation}

Per-decision importance sampling enables off-policy bootstrapping. The change of the algorithm is just to multiply the learning rates of the TD updates by the per-decision importance sampling ratio. 

\section{Function Approximation}
The notion of state that we have discussed before are recognized as \textit{tabular}, in a sense that we can list a table for all the states and their corresponding properties. An agent, at a particular time, can only be in exactly one state, and these states do not influence each other. However, this setting is problematic for the cases in which the state space is too large to be discretized as tables, \eg{} when the state space is continuous. In this case, which is recognized as the \textit{function approximation case}, we have to use the approximate value function, which is not represented as a table but as a parameterized functional form with some corresponding weight vector $\bm{w}$.


In the tabular case a continuous measure of prediction quality was not necessary because the learned value function could converge to the true value function exactly. Moreover, the learned values at each state were decoupled  —an update at one state affected no other. But with function approximation, an update at one state affects many others, and it is not possible to get the values of all states exactly correct. By assumption we have far more states than weights, so making one state’s estimate more accurate invariably means making others’ less accurate. However, this could also mean that making one state more accurate will make some similar states also more accurate. This is often recognized as the dilemma of generalization: on one hand it introduces the forgetting problem; On the other hand it could significantly accelerate learning for its updates to the similar states. 

Since there could be many states, generally the states' importance are weighted using the state frequency distribution $d_{\pi}$. With this we obtain a natural objective function, the state-value error, which is essentially the weighted mean squared value error between the true value and the value estimate.

\begin{definition}[State-Value Error]
Given an MDP, for a state $s$, let its true value under target policy $\pi$ be $v_{\pi}(s)$. Given an estimated value $V_{\pi}(s)$ of the state $s$, the \textit{state value error} of the estimate $V_{\pi}(s)$ is defined as:
\begin{equation}
\nonumber
J(s) \equiv 1/2 \cdot (V_{\pi}(s) - v_{\pi}(s))^2
\end{equation}
\end{definition}

The state value error of a value estimate is its squared distance to the true value. Weighting the state value error by their state frequency $d_{\pi}$ yields the following:

\begin{definition}[Overall Value Error]
Given an MDP and a particular fixed indexing of its states, let its true state-values under target policy $\pi$ be $\bm{v}_{\pi}$, where each element of the vector corresponds to the true value of an indexed state and an value estimate. Given an estimate $\bm{V}_{\pi}$ of all the states, the \textbf{overall value error} of the estimate $\bm{V}_{\pi}$ is defined as:
\begin{equation}
\nonumber
J(\bm{V}_{\pi}) \equiv 1/2 \cdot {\| D_{\pi}^{1/2} \cdot (\bm{V}_{\pi} - \bm{v}_{\pi}) \|}_2^2
\end{equation}
where $D_{\pi}$ is the diagonalized state frequencies under $\pi$, \ie{}
\begin{equation}
\label{eq:Ddef}
D_{\pi} \equiv diag(d_\pi(s_1), d_\pi(s_2), \cdots, d_\pi(s_{|S|}))
\end{equation}
\end{definition}

This criterion can be used with any form of value estimator, either tabular or with function approximators. The weights $D_{\pi}$ favor the states that will be met with higher frequency. The overall value error is often used to evaluate the performance of policy evaluation \cite{singh1997analytical}. When a perfect model of the environment MDP is known, the $\bm{v}_{\pi}$ and $D_{\pi}$ can be exactly solved using DP, as discussed in Section \ref{sec:dynamic_programming}. Thus, DP-solvable MDPs are the first-choices of testing the policy evaluation algorithms.

An ideal goal in terms of policy evaluation would be to find a global optimum, a weight vector $\bm{w}^{*}$ for which $J(\bm{V}_{\pi}(\bm{w}^{*})) \leq J(\bm{V}_{\pi}(\bm{w}))$ for all possible $\bm{w}$. Reaching this goal is sometimes possible for simple function approximators such as linear function approximators, which are to be introduced later, but is rarely possible for complex function approximators, \eg{} artificial neural networks.

Usually, we will use differentiable value estimate functions $V(s; \bm{w})$ parameterized by a weight vector $\bm{w}$ to enable stochastic gradient-descent methods for approaching the update targets.

$\bm{w}$ will be updated at each of a series of discrete timesteps as before, $t \in \{1, 2, \dots\}$, trying to minimize the state-value error. Stochastic gradient-descent (SGD) methods do this by adjusting the weight vector after each example by a small amount in the direction that would most reduce the error on that example:

\begin{equation}
\label{eq:fa_true_gradient}
\bm{w}_{t+1} = \bm{w}_{t} - \frac{1}{2} \alpha \nabla \left[v_{\pi}(S_t) - V_{\pi}(S_t, \bm{w}_t)\right]^2 = \bm{w}_{t} + \alpha \left[v_{\pi}(S_t) - V_{\pi}(S_t, \bm{w}_t)\right] \nabla V_{\pi}(S_t, \bm{w}_t)
\end{equation}
where $\alpha$ is the learning rate, a positive step-size parameter.

Gradient descent methods are called ``stochastic'' when the update is done on only a single example, selected stochastically. Over many steps, the overall effect is to minimize an average performance measure such as the overall value error.

Obviously, we cannot use (\ref{eq:fa_true_gradient}) to do update because the true value $v_{\pi}(S_t)$ is unknown. Thus, we must replace the update target $v_{\pi}(S_t)$ with an estimate $U_{t}$. If $U_{t}$ is unbiased, \ie{} $\doubleE[U_t | S_t = s] = v_{\pi}(S_t), \forall t$, then $\bm{w}_t$ is guaranteed to converge to a local optimum under the usual SGD conditions with decreasing $\alpha$. One simplest instance of this kind of method is to use the MC returns as the update targets, which leads to the gradient-MC method.

Such unbiasedness cannot be achieved with TD updates, which are essentially using bootstrapping estimates as targets. Bootstrapping targets or DP target depend on the current value of the value estimate and the parameter $\bm{w}_t$ for the value estimate. This implies that they are biased and will not produce a true gradient method. It has been proved that bootstrapping methods are not in fact instances of true gradient descent \cite{barnard1993temporal}, as they take into account the effect of changing the weight $\bm{w}_t$ on the estimate but ignore its effect on the target. They are recognized as \textit{semi-gradient} methods because they only take into consideration a part of the gradient.

Although semi-gradient bootstrapping methods do not converge as robustly as gradient methods, they do converge reliably in important cases such as the linear case. Algorithm \ref{alg:SGTD0} shows the simplest instance, the semi-gradient TD(0), which uses $1$-step target as the update target.

\begin{algorithm*}[htbp]
\caption{Semi-Gradient TD($0$) for Policy Evaluation}
\label{alg:SGTD0}
\KwIn{$\pi$ (policy to be evaluated), $\gamma$ (discount function), $\alpha \in (0, 1]$ (learning rate), $N$ (maximum number of episodes)}
\KwOut{$V(s), \forall s \in \scriptS{}^{+}$ (state values for policy $\pi$)}
Initialize $V(s) = 0$, $\forall s \in \scriptS$\\

\For{$n \in \{1, \dots, N\}$}{
    Initialize $S$\\
    \While{$S$ is not terminal}{
        $A = \text{action given by} \pi(\cdot|s)$\\
        Take action $A$, observe $R, S'$\\
        $\bm{w} = \bm{w} + \alpha \left[ R + \gamma(S') \cdot V(S'; \bm{w}) - V(S; \bm{w}) \right] \nabla_{\bm{w}} V(S; \bm{w})$\\
        $S = S'$
    }
}
\end{algorithm*}

\subsection{Linear Methods}
One of the simplest and most important special cases of function approximation is the \textit{linear function}, where $V(\bm{x}; \bm{w}) = \bm{w}^T\bm{x}$, is a linear function of the weight vector $\bm{w}$, and $\bm{x}$ is some real-valued feature vector corresponding to some state $s$. The linear case brings some important properties:




First, the gradient of the parameter $\bm{w}$ has a special form that is independent of the parameter $\bm{w}$ - the gradient of the approximate value function with respect to $\bm{w}$ is $\nabla_{\bm{w}} V(s; \bm{w}) = \bm{x}(s)$. This means that a once a gradient of yielded by some feature is calculated, it remains valid, \ie{} remains a true gradient, forever. This is a very special property that empowers eligibility traces, a fundamental RL method for policy evaluation.

Second, in particular, in the linear case there is only one optimum (or, in degenerate cases, one set of equally good optima), and thus any method that is guaranteed to converge to or near a local optimum is automatically guaranteed to converge to or near the global optimum.

Note that the convergence of linear semi-gradient TD(0) algorithm presented in Algorithm \ref{alg:SGTD0} does not follow from general results on SGD but a separate theorem. The weight vector converged to is also not the optimum, but rather a point nearby. The update for linear semi-gradient TD(0) is
\begin{equation}
\nonumber
\begin{aligned}
\bm{w}_{t+1} & = \bm{w}_{t} + \alpha (R_{t+1} + \gamma(\bm{x}_{t+1}) \bm{w}_t^T \bm{x}_{t+1} - \bm{w}_{t}^T \bm{x}_{t}) \bm{x}_t\\
& = \bm{w}_{t} + \alpha (R_{t+1}\bm{x}_t - \bm{x}_t(\bm{x}_t - \gamma(\bm{x}_{t+1})\bm{x}_{t+1})^T \bm{w}_t)
\end{aligned}
\end{equation}
where $\bm{x}_t \equiv \bm{x}(S_t)$.

If we use onehot encoding for tabular states, which is to use a binary vector with length $|\scriptS|$ and mark the corresponding state with $1$, we can see that tabular TD($0$) is a special case of semi-gradient TD($0$).

\subsection{Tile Coding}
Tile coding uses overlapping tilings to generate binary features for multi-dimensional continuous spaces, which are beneficial for generalization.

In tile coding, multiple tilings are used. Each \textit{tiling} is a grid partition of the state space and each element of the partition is called a \textit{tile}. The tilings are put on the state space, each offset by a fraction of a tile width. A simple case with $4$ tilings is shown on the right side of Figure \ref{fig:example_tile_coding}. Every state, such as that indicated by the white spot, falls in exactly one tile in each of the $4$ tilings. These $4$ tiles correspond to $4$ features that become active. Specifically, the feature vector $\bm{x}(s)$ has one component for each tile in each tiling. In this example there are $4 \times 4 \times 4 = 64$ components, all of which will be $0$ except for the $4$ corresponding to the tiles that $s$ is within.

\begin{figure*}
\centering
\includegraphics[width=0.9\textwidth]{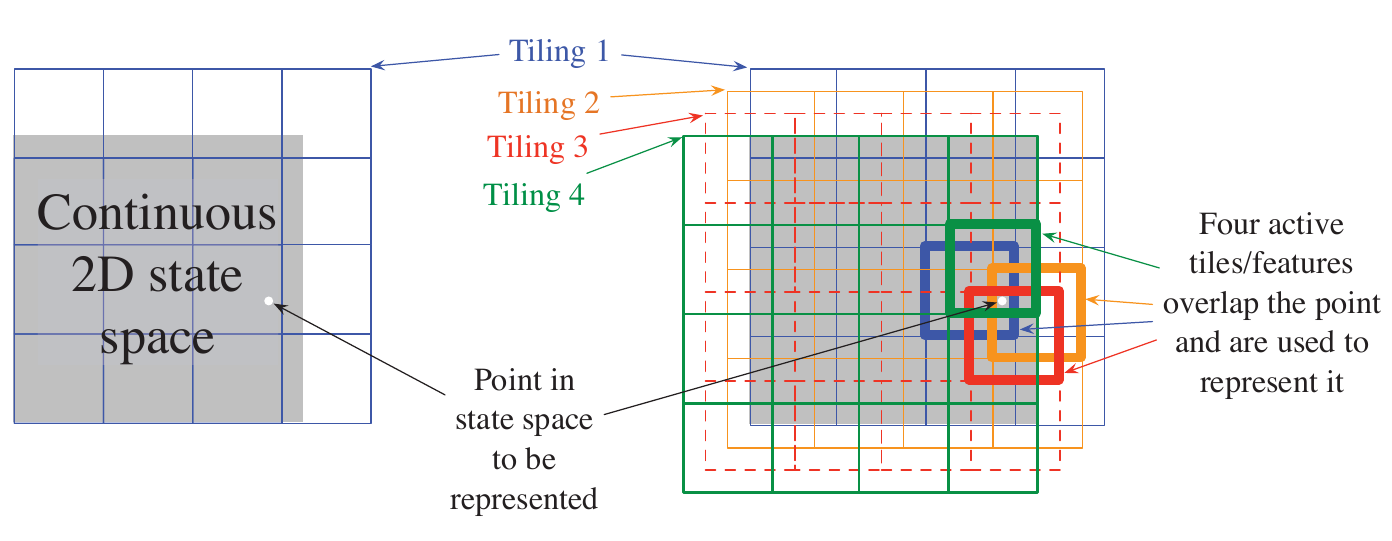}
\caption{Multiple overlapping grid-tilings on a 2D box space with uniform offset in each dimension. Figure from \cite{sutton2018reinforcement}.}
\label{fig:example_tile_coding}
\end{figure*}

There are several advantages of using tile coding for feature construction. 1) Since the overall number of activated tiles are always the same for any state, the learning rate parameter can be set intuitively and easily; 2) Since the features are always binary, efficient binary-float multiplication can be implemented for calculation; 3) Since there is no constraint on the shape of the tilings as well as the offsets, significant degrees of freedom can be utilized to design effective feature constructors.

Tile coding will be used as one of the feature construction methods in the experimental studies section.

\subsection{Off-policy Methods with Function Approximation}
When learning off-policy with function approximation, new troubles emerge for semi-gradient methods. First, the update targets need to be fixed with importance sampling ratios; Second and most importantly, the state distribution will no longer match the target policy. This means that the cumulative effect of gradient or semi-gradient updates does not optimize the overall value error.

Let us first look into why off-policy learning is much more difficult for the case of function approximation than the tabular case. In the function approximation case, generally the updates for one state affect all the similar states whereas in the tabular case, the updates for one state have no influence on others. This means that, in the off-policy case, the tabular updates do not have to care about the state-frequencies when doing updates as long as the update targets are fixed using the importance sampling ratios. The blessing that the tabular case updates do not rely on any special distribution for stability has not been passed to the function approximation case. In the function approximation case, the semi-gradient methods that we have introduced before rely on the state-frequencies for updates. This means we either have to ``reweight'' the updates, \ie{} to warp the update distribution back to the on-policy distribution using importance sampling methods, or we have to develop true gradient methods that do not rely on any special distribution for stability. In fact, the problem of the coexistence of bootstrapping, off-policy learning and function approximation is so troublesome that it is considered as ``the deadly triad''.

Here, we focus on the gradient methods.

\subsubsection{Gradient Methods for Linear Case}

\begin{definition}[Bellman Error]
Given an MDP with its dynamics $p$, a policy $\pi$ and discount function $\gamma$, let the corresponding Bellman operator be $\scriptB_{\pi}$, the \textit{Bellman Error (BE)} for a value estimate $V_{\bm{w}}$ is defined as the norm of the \textit{Bellman Error vector}, \ie{} the expected TD error vector, induced by the state-frequencies $d_{\pi}$.
\begin{equation}
\overline{\text{BE}}(\bm{w}) \equiv \| \overline{\bm{\delta}}_{\bm{w}} \|_{d_\pi}^2 \equiv \| D_{\pi}^{1/2} \cdot \overline{\bm{\delta}}_{\bm{w}} \|_2^2
\end{equation}
where $D_\pi$ is the diagonalized $\bm{d}_\pi$ and the BE vector $\overline{\bm{\delta}}_{\bm{w}}$ is defined as:
\begin{equation}\label{eq:bellman_error}
\overline{\bm{\delta}}_{\bm{w}} \equiv \scriptB_{\pi}V_{\bm{w}} - V_{\bm{w}}
\end{equation}
\end{definition}

It has been proved that, unfortunately, true gradient methods optimizing the Bellman error, named \textit{residual methods}, cannot be realized for general RL settings unless the environment transitions are deterministic or if the transitions of the environment can be somehow reversed. Also, geometric analyses in the linear case show that the minimizers of the Bellman error may not be desirable. These lead to the gradient methods seeking to optimize the Mean Squared Projected Bellman Error (MSPBE).

\begin{definition}[Projected Bellman Error Vector \& Mean Squared Projected Bellman Error]
Given an MDP, target policy $\pi$ and a parameterized value estimate $\bm{V}_{\bm{w}}$, the \textit{Mean Squared Projected Bellman Error (MSPBE)} is defined as the norm of the \textit{Projected Bellman Error vector (PBE)}, induced by the state-frequencies $d_\pi$. 
\begin{equation}
\overline{\text{PBE}}(\bm{w}) \equiv \| \hat{\bm{\delta}}_{\bm{w}} \|_{d_\pi}^2 \equiv \| D_{\pi}^{1/2} \cdot \hat{\bm{\delta}}_{\bm{w}} \|_2^2
\end{equation}
where $D_\pi$ is the diagonalized $\bm{d}_\pi$ and the PBE vector $\hat{\bm{\delta}}_{\bm{w}}$ is defined as:
$$\hat{\bm{\delta}}_{\bm{w}} \equiv \Pi \overline{\bm{\delta}}_{\bm{w}}$$
where $\overline{\bm{\delta}}_{\bm{w}}$ is the Bellman error vector of $\bm{V}_{\bm{w}}$ and $\Pi$ is the projection operator that takes an arbitrary value function $\bm{V}'$ to the representable function that is closest in the weighted norm, which is:
$$\Pi \bm{V}' \equiv \bm{V}_{\bm{w}'} \text{ and } \bm{w}' = \argmin_{\bm{w}}{\|\bm{V}' - \bm{V}_{\bm{w}}\|_{d_\pi}^2}$$
\end{definition}

In the linear case, given a fixed state-to-feature mapping and a fixed policy $\pi$, the projection operator is a static linear transformation which does not depend on the parameter $\bm{w}$. This gives birth to the linear gradient-TD methods which achieve convergence by minimizing MSPBE under some conditions. In this thesis, we propose an assistive method for general function approximation based policy evaluation, but test on these linear methods with the back of the convergence guarantees.

\section{From $\lambda$-Return to Eligibility Traces}
For a given MDP, the update targets of different steps, \eg{} $1$-step target, $2$-step target, \etc{}, yield different biases and variances, which lead to different qualities of the estimates. Naturally, one would like to combine them in a way \st{} we get better estimates, \eg{} with lower MSE towards the tru value, and without losing the properties of convergence. The following fact unlocks such possibility.

\begin{fact}[Compound Targets \& Compound Updates]
Given a MDP and appropriate learning rate, using convex combinations of multi-step returns as update targets for policy evaluation achieves convergence to fixed-points near the true value.
\end{fact}

Often, these targets are recognized as \textit{compound targets}. The updates using compound targets are called \textit{compound updates}. The convergence of compound updates relies on the fact that compound targets are composed of multi-step updates, and each of them has convergence guarantees.

There are potentially many ways to mix the multi-step targets to achieve compound updates. Optimizing on the way of mixing should be beneficial for the sample efficiency of policy evaluation.

One of the most popular way is to mix the multi-step targets with a geometric weight sequence, which yields the famous $\lambda$-return:

\begin{definition}[$\lambda$-return]
Given timestep $t$ and the corresponding state $S_t$, the $\lambda$-return of $S_t$ is defined as a convex combination of the multi-step targets of $S_t$, with the weights specified using a geometric sequence, controlled by a scalar parameter $\lambda$. Specifically, the $\lambda$-return $G_{t}^{\lambda}$ combines all $n$-step returns $G_t^{(n)}$ using weight  $(1-\lambda)\lambda^{n-1}$:
$$G_{t}^{\lambda} \equiv (1 - \lambda) \sum_{n=1}^{\infty}{\lambda^{n-1} G_t^{(n)}}$$
where $\lambda \in [0, 1]$.
\end{definition}

Methods that use $\lambda$-return as update target, given that parameter $\lambda$ is set to be appropriate, often achieve significantly higher sample efficiency than using only some fixed-step returns as targets, for the fact that $\lambda$-return has a better bias-variance tradeoff.

Calculating this na\"ively requires knowing all the multi-step targets up until the end of the trajectory. Thus, updates using $\lambda$-return directly can only be done offline, which is recognized as the \textit{offline $\lambda$-return algorithm}. Offline algorithms can be unsatisfactory for many reasons. Fortunately, there is a way to approximate the updates towards $\lambda$-return in an online fashion\footnote{Actually, the online method that approximates the $\lambda$-return updates was discovered before the identification of $\lambda$-return as a compound target. However the introduction to the online method via the notion of compound targets are beneficial for understanding.}. The compound update using $\lambda$-return can be approximated using the eligibility traces from a backward view, with the help of the eligibility trace vectors:

\begin{proposition}[Trace Approximation]
Given an MDP and an infinitely long trajectory under policy $\pi$, the updates using $\lambda$-return as targets can be approximated with online updates. More specifically, the update rules are:
$$\bm{z}_t = \gamma \lambda \bm{z}_{t-1} + \bm{\nabla} V(\bm{x}_{t}, \bm{w}_t)$$
$$\bm{w}^{(t+1)} = \bm{w}_t + \alpha [R_{t+1} + \gamma V(\bm{x}_{t+1}, \bm{w}_t) -V(\bm{x}_{t}, \bm{w}_t)] \bm{z}_t$$
where $\bm{z}$ is the \textit{eligibility trace} vector which is initialized as $\bm{0}$ at the beginning of the episode and $\alpha$ is some appropriate learning rate.
\end{proposition}

\begin{proof}
\begin{equation}
\begin{aligned}
& G_t^\lambda - V(S_t)\\
& = -V(S_t) + (1 - \lambda) \lambda^0 (R_{t+1} + \gamma^1 V(S_{t+1})) + (1 - \lambda) \lambda^1 (R_{t+1} + \gamma^1 R_{t+2} + \gamma^2 V(S_{t+2})) + \cdots\\
& = -V(S_t) + (1 - \lambda)[\lambda^0\gamma^0(\lambda^0 + \lambda^1 + \cdots)R_{t+1} + \lambda^1\gamma^1(\lambda^0 + \lambda^1 + \cdots)R_{t+2} + \cdots + \gamma^1\lambda^0 V(S_{t+1}) + \gamma^2\lambda^1 V(S_{t+2}) + \cdots]\\
& = -V(S_t) + (1 - \lambda)[\frac{\lambda^0\gamma^0}{1-\lambda}R_{t+1} + \frac{\lambda^1\gamma^1}{1-\lambda}R_{t+2} + \cdots + \gamma^1\lambda^0 V(S_{t+1}) + \gamma^2\lambda^1 V(S_{t+2}) + \cdots]\\
& = (\gamma \lambda)^0(R_{t+1} + \gamma V(S_{t+1}) - \gamma \lambda V(S_{t+1})) + (\gamma \lambda)^1(R_{t+2} + \gamma V(S_{t+2}) - \gamma \lambda V(S_{t+2})) + \cdots\\
& = \sum_{k=t}^{\infty}{(\gamma \lambda)^{k-t}(R_{k+1} + \gamma V(S_{k+1}) - V(S_k))}
\end{aligned}
\nonumber
\end{equation}
This is exactly the form of accumulating trace\footnote{People refer to ``eligibility'' trace as the family of the algorithms that employ trace vectors to incrementally approximate something that cannot be calculated directly. There are several variants of the eligibility traces. The accumulating traces are the original and the most classical one.}:
\end{proof}
Note that the equality holds only when the trajectory is infinitely long. When it is not (as it will always be), there will be differences between this online approximation algorithm, which is named $TD(\lambda$) and the offline $\lambda$-return algorithm. Interestingly, even if so, it is proved in \cite{dayan1994td} that the TD($\lambda$), the algorithm that approximates with eligibility traces vectors, achieves convergence to a fixed-point near the true value almost surely.

Naturally, scalar parameter $\lambda$ for all states can be generalized to state-based $\lambda(\cdot)$ \cite{sutton1995models,sutton2018reinforcement,maei2010general,sutton2014interim,yu2012least}. If we index the states and concatenate all the $\lambda(s) \forall s \in \scriptS{}$, we arrive at the generalized $\bm{\lambda}$-return\footnote{$\bm{\lambda}$ here is in bold font to emphasize that it is a vector.}:

\begin{definition}[$\bm{\lambda}$-return]
The generalized state-based $\bm{\lambda}$-return $G_t^{\bm{\lambda}}$, where $\bm{\lambda} \equiv {[\lambda_1, \dots, \lambda_i \equiv \lambda(s_i), \dots, \lambda_{|\scriptS|}]}^T$, for state $S_t$ in a trajectory $\tau$ is recursively defined as
\begin{equation}\nonumber
G_t^{\bm{\lambda}} = R_{t+1} + \gamma_{t+1} [(1 - \lambda_{t+1})V(S_{t+1}) + \lambda_{t+1}G_{t+1}^{\bm{\lambda}}]
\end{equation}
where $G_t^{\bm{\lambda}} = 0$ for $t \geq |\tau|$.
\end{definition}

Convergence of method using $\bm{\lambda}$-return as the update target can be shown based on the fact that the generalized Bellman operator with state-dependent discount and trace decay is a contraction \cite{sutton2015emphatic}. On the other hand, time-dependent decay may not be equipped with well-defined fixed points.

Similarly, online approximations using traces exist for offline updates with $\lambda$-return.

\begin{proposition}[Generalized Trace Approximation]
Given an MDP, an infinitely long trajectory under policy $\pi$, the state-based discount function $\gamma: \scriptS \to [0, 1]$ and the state-based trace-decay function $\lambda: \scriptS \to [0, 1]$, the updates using $\bm{\lambda}$-return as targets can be approximated with online updates. More specifically, the update rules are:
$$\bm{z}_t = \gamma(\bm{x}_{t}) \lambda(\bm{x}_{t}) \bm{z}_{t-1} + \bm{\nabla} V(\bm{x}_{t}, \bm{w}_t)$$
$$\bm{w}^{(t+1)} = \bm{w}_t + \alpha [R_{t+1} + \gamma V(\bm{x}_{t+1}, \bm{w}_t) -V(\bm{x}_{t}, \bm{w}_t)] \bm{z}_t$$
where $\alpha$ is some appropriate learning rate.
\end{proposition}

\begin{proof}
Let $\gamma_k \equiv \gamma(S_k)$ and $\lambda_k \equiv \lambda(S_k)$.
\begin{equation}
\begin{aligned}
& G_t^{\bm{\lambda}} - V(\bm{x}_t)\\
& = -V(\bm{x}_t) + R_{t+1} + \gamma_{t+1}(1 - \lambda_{t+1})V(\bm{x}_{t+1}) + \gamma_{t+1}\lambda_{t+1}G_{t+1}^{\bm{\lambda}}\\
& = -V(\bm{x}_t) + R_{t+1} + \gamma_{t+1}(1 - \lambda_{t+1})V(\bm{x}_{t+1}) + \gamma_{t+1}\lambda_{t+1}(R_{t+2} + \gamma_{t+2}(1 - \lambda^{(t+2)})V(\bm{x}_{t+2}) + \gamma_{t+2}\lambda^{(t+2)}G_{t+2}^{\bm{\lambda}})\\
& = -V(\bm{x}_t) + R_{t+1} + \gamma_{t+1}(1 - \lambda_{t+1})V(\bm{x}_{t+1}) + \gamma_{t+1}\lambda_{t+1}R_{t+2} + \gamma_{t+1}\lambda_{t+1}\gamma_{t+2}(1 - \lambda^{(t+2)})V(\bm{x}_{t+2}) + \gamma_{t+1}\lambda_{t+1}\gamma_{t+2}\lambda^{(t+2)}G_{t+2}^{\bm{\lambda}}\\
& = \cdots\\
& = \sum_{k=t}^{\infty}{(\gamma_k \lambda^{(k)})^{k-t}(R_{k+1} + \gamma_k V(\bm{x}_{k+1}) - V(\bm{x}_k))}
\end{aligned}
\nonumber
\end{equation}
\end{proof}

Generalization from scalar $\lambda$ to state-based $\bm{\lambda}$ greatly increases the potential of trace-based policy evaluation in a sense that significantly more degrees of freedom are unlocked for the mixing of multi-step targets and potentially the achievement of update targets with significantly less bias and variance.

The complexity of the online trace updates stops anyone from optimizing it, unless the true equivalence can be established between the online and offline algorithms.

Recently, a new family of ``true online'' algorithms has been discovered, which achieves exact equivalence of the online approximation and the offline updates. The equivalence is achieved by maintaining extra traces that correct the updates towards a true convex combination of multi-step targets.

\begin{fact}[True Online Equivalence]
The following incremental update rules achieve exact equivalence to updates using ${\bm{\lambda}}$-return targets using linear function approximations:
$$\delta_t = R_{t+1} + \gamma_{t+1} \bm{x}_{t+1}^T w_{(t)} - \bm{x}_{t}^T w_{(t)} $$
$$\bm{z}_{(t)} = \gamma_{t}\lambda_{(t)}\bm{z}_{(t-1)} + \bm{x}_{t} - \alpha\gamma_{t+1}\lambda_{t+1}(\bm{z}_{(t)}^T\bm{x}_{t})\bm{x}_{t}$$
$$\bm{w}_{(t+1)} = \bm{w}_{(t)} + \alpha \delta_t \bm{z}_{(t)} + \alpha (\bm{w}_{(t)}^T \bm{x}_t - \bm{w}_{(t-1)}^T \bm{x}_t)(\bm{z}_{(t)} - \bm{x}_t)$$
where $\alpha$ is the step-size hyperparameter that is consistent from both the forward view and the backward view.
\end{fact}

With the true online algorithms, we can enjoy the efficiency of online updates and the guarantee of convergence from the offline returns simultaneously. The equivalence also could serve as a bridge to optimize the online updates using the mathematical properties of the backward view.


Very recently, it is discovered that, it is possible to learn the statistics, \eg{} variance, second-moment, of $\bm{\lambda}$-return online using eligibility traces. The online learning of these statistics as auxiliary tasks provides more information for online adaptation of the learning parameters. Here we introduce $2$ relevant ones.

The first method VTD comes from \cite{white2016greedy}, which seeks to learn the second moment of $\bm{\lambda}$-return. A Bellman operator is constructed for the squared $\bm{\lambda}$-return, which is the second moment of return. A fixed point objective Var-MSPBE, similar to MSPBE, is introduced. The recursive form for the squared return is:

\begin{equation}\label{eq:vtd}
\begin{aligned}
\left( G_{t}^{\bm{\lambda}} \right)^2 & = \left(\rho_t R_{t+1} + \gamma_{t+1} [(1 - \lambda_{t+1})V(S_{t+1}) + \lambda_{t+1}G_{t+1}^{\bm{\lambda}}] \right)^2\\
& = \overline{R}_{t+1} + \overline{\gamma}_{t+1} (G_{t+1}^{\bm{\lambda}})^2
\end{aligned}
\end{equation}
where for a given $\bm{\lambda} \in [0, 1]^{|\scriptS{}|}$,
$$\overline{G}_{t} \equiv R_{t+1} + \gamma_{t+1} (1 - \lambda_{t+1}) M(\bm{x}_{t+1}; \bm{w}_{M}) = R_{t+1} + \gamma_{t+1} (1 - \lambda_{t+1}) \bm{w}_{M}^T \bm{x}_{t+1}$$
$$\overline{R}_{t+1} \equiv \rho_t^2 \overline{G}_{t}^2 + 2 \rho_t^2 \gamma_{t+1} \lambda_{t+1} \overline{G}_{t} G_{t+1}^{\bm{\lambda}}$$
$$\overline{\gamma}_{t+1} \equiv \rho_t^2 \gamma_{t+1}^2 \lambda_{t+1}^2$$

Here $M$ is the linear approximation of the expected squared return, parameterized by $\bm{w}_{M}$. It has been proved that even if $\overline{\gamma}_{t+1} < 1$ cannot be guaranteed, the Bellman operator for these new types of ``reward'' $\overline{R}_{t+1}$ and discount function $\overline{\gamma}_{t+1}$ is still a contraction under appropriate conditions, which is when the second moment of return is finite. These tell us that the second moment of return can be learned in the same way as the first moment of return (which is the expected $\bm{\lambda}$-return) using MC, (trace-enabled) TD or even GTDs. Notice that, using estimated first moment and second moment we can indirectly estimate the variance of $\bm{\lambda}$-return by constructing $Var[G_t^{\bm{\lambda}}] = \doubleE[(G_t^{\bm{\lambda}})^2] - \doubleE^2[G_t^{\bm{\lambda}}]$. Interestingly, such estimation is proved to be possible if carried out directly, which leads to the following.

The second method DVTD is more recent, published in \cite{sherstan2018directly}, which seeks to learn the variance of the $\bm{\lambda}$-return directly by constructing a Bellman operator for the variance of the $\bm{\lambda}$-return. The recursive form for the variance of the $\bm{\lambda}$-return is:
\begin{equation}\label{eq:dvtd}
\begin{aligned}
Var[G_t^{\bm{\lambda}}] = \doubleE \left[ \delta_t^2 + (\gamma_{t+1}\lambda_{t+1})^2 Var[G_{t+1}^{\bm{\lambda}}] | S_t = s \right]
\end{aligned}
\end{equation}
It has been proved that the induced Bellman operator converges under appropriate assumptions. Similar to how we would apply VTD, this method can be used to learn the variance of $\bm{\lambda}$-return flexibly if we treat reward as $\delta_t^2$ and discount to be $(\gamma_{t+1}\lambda_{t+1})^2$.

Trace based true online methods will be used as baseline methods to validate the empirical performance of the contributed method in this thesis, in Chapter \ref{chap:experiments}.

\section{Policy Gradient Methods}\label{sec:policy_gradient}
Just like the value estimates, policies can also make use of the power of generalization by parameterization. In this section, we introduce the ways of dealing with parameterized policies that selects actions, independent of the value estimates\footnote{We focus on the type of parameterized policies disentangled from the value estimates.}.

We can use the notation $\bm{\theta}$ for the policy's parameter vector. Thus we write $\pi(a|s; \bm{\theta}) = \doubleP\{ A_t = a | S_t = s, \bm{\theta}_t = \bm{\theta} \}$ for the probability that action $a$ is taken at time $t$ given that the environment is in state $s$ at time $t$ with parameter $\bm{\theta}$. 

We seek to maximize the value function of the starting states by optimizing $\bm{\theta}$. The representative method is to conduct gradient ascent, whose performance is guaranteed by the following\footnote{This is an original proof for the generalized state-based $\gamma$ case.}:

\begin{theorem}[Episodic Discounted Policy Gradient]
Given a policy $\pi(\cdot; \bm{\theta})$, let the true state-action value function be $q_\pi(\cdot)$, the true state value function be $v_\pi(\cdot)$, the starting state distribution be $d_0$ and the state distribution be $d_\pi$. The gradient of $v_\pi(\cdot)$ \wrt{} $\bm{\theta}$ satisfies:
\begin{equation}
\begin{aligned}
\bm{\nabla} v_\pi(s_0) &\propto \sum_{s}{ d_\pi(s) \cdot \gamma(s_0 \to s, \pi) \cdot  \sum_{a}{\left(q_\pi(s, a) - h(s)) \bm{\nabla} \pi(a|s; \bm{\theta} \right)} }
\end{aligned}
\end{equation}
where $s_0 \sim d_0$, $s \sim d_\pi$, $h: \scriptS \to \doubleR$ is any random variable that is not dependent on $a$ and $\gamma(s_0 \to s, \pi)$ is the expected cumulative product of discount factors transitioning from $s_0$ to $s$, specifically:
$$\gamma(s_0 \to s, \pi) \equiv \doubleE_{\tau \sim \pi} \left[ {{\sum_{k=0}^{\infty}{\doubleP\{ s_0 \xrightarrow{\tau} s , k, \pi \} \left(\prod_{y=s_1}^{s}{\gamma(y)}\right)} }} \right]$$
where $\tau$ is a trajectory sampled with $\pi$, $\doubleP\{ s \xrightarrow{\tau} x , k, \pi \}$ is the probability of transitioning according to trajectory $\tau$ from state $s$ to state $x$ in exactly $k$ steps under policy $\pi$ and $\prod_{y=s_1}^{s}{\gamma(y)}$ is the cumulative product of discount factors following the trajectory $\tau$ along $s_0, s_1, s_2, \dots, s_{k-1}, s$.
\end{theorem}

\begin{proof}
The gradient of $v_\pi$ can be written in terms of $q_\pi$ as
\begin{equation}
\begin{aligned}
\bm{\nabla} v_\pi (s) & = \bm{\nabla} \left[ \sum_{a}{\pi(a|s)q_\pi(s,a)} \right], \forall s \in \scriptS{}\\
& = \sum_{a}{\left[ q_\pi(s, a) \bm{\nabla} \pi(a|s) + \pi(a|s) \bm{\nabla} q_\pi(s, a) \right]}\\
& = \sum_{a}{\left[ q_\pi(s, a) \bm{\nabla} \pi(a|s) + \pi(a|s) \bm{\nabla} \sum_{s',r}{p(s',r|s,a)(r + \gamma(s') v_\pi(s'))} \right]}\\
& = \sum_{a}{\left[ q_\pi(s, a) \bm{\nabla} \pi(a|s) + \pi(a|s) \sum_{s',r}{p(s',r|s,a)\gamma(s') \bm{\nabla} v_\pi(s')} \right]}\\
& = \cdots \text{(keep unrolling $\bm{\nabla} v_\pi(\cdot)$)}\\
& = \sum_{x \in \scriptS{}}{\sum_{\tau \sim \pi}{\sum_{k=0}^{\infty}{ \left[ \doubleP\{ s \xrightarrow{\tau} x , k, \pi \} \cdot \left(\prod_{y=s'}^{x}{\gamma(y)}\right) \cdot \sum_{a}{ \bm{\nabla} \pi (a | x) q_\pi (x, a) } \right]}}}
\end{aligned}
\end{equation}
The value of the initial state is what we care about, thus
\begin{equation} 
\begin{aligned}
& \bm{\nabla}v_\pi (s_0)\\
& = \sum_{s}{ (\sum_{\tau \sim \pi}{{\sum_{k=0}^{\infty}{\doubleP\{ s_0 \xrightarrow{\tau} s , k, \pi \} \left(\prod_{y=s_1}^{s}{\gamma(y)}\right)} }}) \sum_{a}{q_\pi(s,a) \cdot \left(\prod_{y=s_1}^{s}{\gamma(y)}\right) \cdot \bm{\nabla} \pi(a|s) } }\\
& = \sum_{s}{ \eta_\pi(s) \sum_{a}{q_\pi(s,a) \cdot \gamma(s_0 \to s, \pi) \cdot \bm{\nabla} \pi(a|s) } }\\
& \text{$\eta_\pi(s)$ is the expected number of visits to state $s$ under $\pi$}\\
& \propto \sum_{s}{ d_\pi(s) \sum_{a}{q_\pi(s,a) \gamma(s_0 \to s, \pi) \cdot \bm{\nabla} \pi(a|s) } }
\end{aligned}
\end{equation}
We notice that:
\begin{equation}
\begin{aligned}
0 & = \bm{\nabla}1\\
& = \bm{\nabla} \sum_{a}{\pi(a | s)}\\
& = h(\cdot) \bm{\nabla} \sum_{a}{\pi(a | s)}\\
& = \sum_{a}{h(\cdot) \bm{\nabla} \pi(a | s)} &&\text{as long as $h$ has nothing to do with $a$}
\end{aligned}
\end{equation}
Thus,
\begin{equation}
\begin{aligned}
\bm{\nabla}v_\pi (s_0) & = \bm{\nabla}v_\pi (s_0) + 0\\
& \propto \sum_{s}{ d_\pi(s) \cdot \gamma(s_0 \to s, \pi) \cdot \sum_{a}{q_\pi(s,a) \bm{\nabla} \pi(a|s) } }  + \sum_{a}{h(\cdot) \bm{\nabla} \pi(a | s)}\\
& \propto \sum_{s}{ d_\pi(s) \cdot \gamma(s_0 \to s, \pi) \cdot \sum_{a}{\left(q_\pi(s,a) - h(s)\right) \bm{\nabla} \pi(a|s) } }
\end{aligned}
\end{equation}

\end{proof}

When learning online, it is desirable, instead of to sum over $a$, use $A_t$ to do a stochastic update.
\begin{equation}\label{eq:online_pg}
\begin{aligned}
& \bm{\nabla}v_\pi (s_0)\\
& = \doubleE_{\pi}\left[\gamma(s_0 \to s, \pi) \cdot\sum_{a}{q_\pi(S_t, a) \bm{\nabla} \pi(a|S_t; \bm{\theta}) } \right]\\
& = \doubleE_{\pi}\left[\prod_{X = S_1}^{S_t}{\gamma(X)} \cdot \sum_{a}{\pi(a|S_t; \bm{\theta}) \cdot q_\pi(S_t, a) \cdot \frac{\bm{\nabla} \pi(a|S_t; \bm{\theta})}{\pi(a|S_t; \bm{\theta})} } \right]\\
& = \doubleE_{\pi}\left[\prod_{X = S_1}^{S_t}{\gamma(X)} \cdot q_\pi(S_t, a) \cdot \frac{\bm{\nabla} \pi(a|S_t; \bm{\theta})}{\pi(a|S_t; \bm{\theta})} \right]\\
& = \doubleE_{\pi}\left[\prod_{X = S_1}^{S_t}{\gamma(X)} \cdot q_\pi(S_t, A_t) \cdot \frac{\bm{\nabla} \pi(A_t|S_t; \bm{\theta})}{\pi(A_t|S_t; \bm{\theta})} \right] &&\text{$\doubleE_\pi[A_t]=a$}\\
& = \doubleE_{\pi}\left[\prod_{X = S_1}^{S_t}{\gamma(X)} \cdot U_t \cdot \frac{\bm{\nabla} \pi(A_t|S_t; \bm{\theta})}{\pi(A_t|S_t; \bm{\theta})} \right] &&\text{as long as $\doubleE_\pi[U_t|S_t, A_t]=q_\pi(S_t, A_t)$}
\end{aligned}
\end{equation}

We realize that the corresponding online update rule of the policy gradient theorem has the same problem as that of the gradient of values (\ref{eq:fa_true_gradient}) - $q_\pi(\cdot)$ is unknown and must be replaced with an estimator.



$h$, either a function or a random variable, is recognized as a \textit{baseline}. In general, the baseline leaves the expected value of the gradient unchanged, yet having significant effect on its variance. The possibility of variance reduction is enabled by the idea recognized as \textit{control variates}. One popular choice of the baseline is an estimate of the state value $V$. Replacing $q_\pi(s,a)-h(s)$ with $G-V(s)$, where $G$ is the MC return estimator, then we will arrive at the simplest policy gradient method, which is named \textit{REINFORCE}. However, its details will not be discussed since they are not related to the contribution of this thesis.

\subsection{Actor-Critic}
Enabling bootstrapping in policy gradient methods is crucial, since the bias introduced through bootstrapping reduces the variance and boosts sample efficiency (makes learning faster and more accuracy). REINFORCE with baseline is unbiased and will converge asymptotically to a local minimum, but since it has no bootstrapping and updates only upon a high-variance target (MC return), it is problematic for online learning. Actor-Critic methods eliminate these inconveniences with TD and through the mixing of multi-step targets, we can flexibly determine the degree of bootstrapping.

The first and simplest instance of these methods is the $1$-step actor-critic method, which is fully online and incremental, yet avoid the complexities of eligibility traces. Replacing the target $U_t$ in (\ref{eq:online_pg}) with $1$-step target $R_{t+1} + \gamma(S_{t+1})V(S_{t+1})$, we have it as follows:

\begin{equation}
\begin{aligned}
\bm{\theta}_{t+1} & = \bm{\theta}_{t} + \alpha \left( G_{t:t+1} - V(S_t, \bm{w}) \right) \frac{\bm{\nabla} \pi(A_t | S_t, \bm{\theta}_t)}{\pi(A_t | S_t, \bm{\theta}_t)}\\
& = \bm{\theta}_{t} + \alpha \left( R_{t+1} + \gamma(S_{t+1})V(S_{t+1}, \bm{w}) - V(S_t, \bm{w}) \right) \frac{\bm{\nabla} \pi(A_t | S_t, \bm{\theta}_t)}{\pi(A_t | S_t, \bm{\theta}_t)}
\end{aligned}
\end{equation}

With this, we have the $1$-step Actor-Critic method for episodic tasks, as presented in Algorithm \ref{alg:AC_0}.

\begin{algorithm*}[htbp]
\caption{Episodic $1$-step Actor-Critic for estimating $\pi_{*}$}
\label{alg:AC_0}
\KwIn{$\pi(a|s; \bm{\theta})$ (differentiable policy parameterization), $V(s; \bm{\theta})$ (differentiable state-value estimate parameterization), $\gamma$ (discount function), $\alpha_{\bm{\theta}}, \alpha_{\bm{w}}$ (learning rates for actor and critic, respectively), $N$ (maximum number of episodes)}
\KwOut{$\pi \approx \pi_{*}$ (an estimate of the optimal policy), $V(s), \forall s \in \scriptS{}^{+}$ (estimated state-values for policy $\pi_{*}$)}

Initialize weights $\bm{\theta}$ and $\bm{w}$, \eg{} to $\bm{0}$\\
\For{$n \in \{1, \dots, N\}$}{
    Initialize $S$ \textcolor{darkgreen}{// first state of episode}\\
    $I = 1$ \textcolor{darkgreen}{// cumulative product of discount factors}\\
    \While{$S$ is not terminal}{
        $A \sim \pi(\cdot | S, \bm{\theta})$\\
        Take action $A$, observe $R, S'$\\
        $\delta = R + \gamma(S')V(S'; \bm{w}) - V(S; \bm{w})$ \textcolor{darkgreen}{// TD error, $V(S'; \bm{w}) \equiv 0$ if $S'$ is terminal}\\
        $\bm{w} = \bm{w} + \alpha_{\bm{w}} \delta \cdot \bm{\nabla}_{\bm{w}} V(S; \bm{w})$\textcolor{darkgreen}{// $1$-step semi-gradient TD update for critic}\\
        $\bm{\theta} = \bm{\theta} + \alpha_{\bm{\theta}} I \delta \cdot \bm{\nabla}_{\bm{\theta}} ln(\pi(A|S; \bm{\theta}))$\textcolor{darkgreen}{// $1$-step update for actor}\\
        $I = I \cdot \gamma(S')$\\
        $S = S'$
    }
}
\end{algorithm*}

The generalizations to the forward view of $n$-step methods and then to a $\lambda$-return algorithm are straightforward. The episodic Actor-Critic method with eligibility traces is presented as follows, in Algorithm \ref{alg:AC_lambda}.

\begin{algorithm*}[htbp]
\caption{Episodic Actor-Critic with Eligibility Traces for estimating $\pi_{*}$}
\label{alg:AC_lambda}
\KwIn{$\pi(a|s; \bm{\theta})$ (differentiable policy parameterization), $V(s; \bm{\theta})$ (differentiable state-value estimate parameterization), $\gamma$ (discount function), $\alpha_{\bm{\theta}}, \alpha_{\bm{w}}$ (learning rates for actor and critic, respectively), $\lambda_{\bm{\theta}}, \lambda_{\bm{w}}$ (trace-decay functions for actor and critic, respectively), $N$ (maximum number of episodes)}
\KwOut{$\pi \approx \pi_{*}$ (an estimate of the optimal policy), $V(s), \forall s \in \scriptS{}^{+}$ (estimated state-values for policy $\pi_{*}$)}

Initialize weights $\bm{\theta}$ and $\bm{w}$, \eg{} to $\bm{0}$\\
\For{$n \in \{1, \dots, N\}$}{
    Initialize $S$ \textcolor{darkgreen}{// first state of episode}\\
    $I = 1$ \textcolor{darkgreen}{// cumulative product of discount factors}\\
    $\bm{z}_{\bm{\theta}} = \bm{0}$ \textcolor{darkgreen}{// eligibility trace for $\bm{\theta}$}\\
    $\bm{z}_{\bm{w}} = \bm{0}$ \textcolor{darkgreen}{// eligibility trace for $\bm{w}$}\\
    \While{$S$ is not terminal}{
        $A \sim \pi(\cdot | S, \bm{\theta})$\\
        Take action $A$, observe $R, S'$\\
        $\delta = R + \gamma(S')V(S'; \bm{w}) - V(S; \bm{w})$ \textcolor{darkgreen}{// TD error, $V(S'; \bm{w}) \equiv 0$ if $S'$ is terminal}\\
        $\bm{z}_{\bm{w}} = \gamma(S)\lambda_{\bm{w}}(S)\bm{z}_{\bm{w}} + \bm{\nabla}_{\bm{w}} V(S, \bm{w})$\textcolor{darkgreen}{// accumulating traces for $\bm{z}_{\bm{w}}$}\\
        $\bm{z}_{\bm{\theta}} = \gamma(S)\lambda_{\bm{\theta}}(S)\bm{z}_{\bm{\theta}} + I \cdot \bm{\nabla}_{\bm{\theta}} ln(\pi(A|S; \bm{\theta}))$\textcolor{darkgreen}{// accumulating traces for $\bm{z}_{\bm{\theta}}$}\\
        $\bm{w} = \bm{w} + \alpha_{\bm{w}} \delta \cdot \bm{z}_{\bm{w}}$\\
        $\bm{\theta} = \bm{\theta} + \alpha_{\bm{\theta}} \delta \cdot \bm{z}_{\bm{\theta}}$\\
        $I = I \cdot \gamma(S')$\\
        $S = S'$
    }
}
\end{algorithm*}

%% file: chapter_3_preliminary.tex
\chapter{Sample Efficiency of Temporal Difference Learning}
\textit{Learning faster and more accurately.}

\section{Sample Efficiency}
\textit{Sample efficiency} is widely regarded as one of the main bottlenecks of the existing RL methods, for both prediction and control. For example, on even very simple tasks, RL agents often require very large amount of agent-environment interactions in order to predict or control well. For prediction, it is observed from the average number of samples (interactions) to reach certain accuracy of the value estimate. For control, it is observed from the average number of samples to reach an optimal policy (or a good policy, if there is no guarantee to reach optimal). To our knowledge, there is no formal definition of sample efficiency. However, it can be naturally compared among different methods on identical tasks, based on their learning performance. Intuitively for prediction, having a higher sample efficiency, equivalent to having a lower sample efficiency, achieves higher learning speed as well as better accuracy. This thesis focuses on improving the sample efficiency of eligibility trace based methods, \eg{} TD($\lambda$), which are widely adopted.

\section{Sample Efficiency \& Tuning $\lambda$}
Not only shown in \cite{sutton2018reinforcement,singh1997analytical,kearns2000bias}, but also pervasively in RL practices that, given appropriate learning rates and amount of steps, TD($\lambda$) with different $\lambda$ performs very differently: they exhibit patterns of U-shaped curves when comparing the squared error of their value estimates against the true value, \ie{} the sample efficiency of TD($\lambda$) is sensitive to the value of $\lambda$. To deal with this \cite{singh1997analytical,kearns2000bias}, the normal approach would be simply to run TD($\lambda$) with different values of $\lambda$ and pick whichever $\lambda$ value that yields the best performance.

This brute-force search is a general strategy for hyperparameter search. In prediction tasks, where a fixed policy is given and the distribution of samples do not change with the values of $\lambda$, this parallel search requires no more samples from the environment and is effective but still very computationally expensive. In control tasks, where the quality of policy evaluation also determines the change of policies, this method is unsatisfactory.

The sensitivity naturally leads researchers to investigate into possible methods of adapting $\lambda$ for better sample efficiency. Different interpretation of methods give rise to different approaches, which are currently few and mostly limited to special cases.

\subsection{Counting Heuristics}
One can think of $\lambda$-return constructed with a mix of current value estimate and the experienced reward transitions (MC return). Thus, from the view of credit assignment, it is appropriate to interpret $\lambda$ as ``confidence'' to the current value estimates against the MC return: the more confident the method is to the current value estimate, the lower $\lambda$-values it should use. Thus, counting methods keep count of the visits of states (or neighborhoods of states) and adjust $\lambda$ based on the counts to reflect the confidence.

\subsection{Bypass Methods}
Some methods seek to bypass the framework of TD($\lambda$) and ameliorate the mechanisms of TD. \cite{konidaris2011complex} introduces TD$_{\gamma}$ as a method to remove the $\lambda$ parameter altogether. Their approach, however, has not been extended to the off-policy setting and their full algorithm is computationally expensive for incremental estimation, while their incremental variant introduces a sensitive meta-parameter. \cite{carlos2019adaptive} proposes a method to mitigate the propagation of errors by adaptively switching from $1$-step TD updates to MC updates in each state.

\subsection{Meta-Learning Methods}
Meta-learning, recognized often as ``learning to learn'', is a methodology with which a meta-objective upon the problem objectives, in the RL case either optimizing errors of value estimates or improving policies, is proposed to enhance the problem solving process. Meta-learning algorithms adapt the parameters of the agent continuously, based on the stream of experience and some notion of the agent's own learning progress. Meta-learning has been heavily proposed as an approach for adapting the learning rates \cite{dabney2012adaptive} however rarely for the trace-decay parameter $\lambda$.

Meta-learning $\lambda$ dates back to \cite{Sutton94stepsize}, where adaptation of $\lambda$ using an meta-learning algorithm is proposed, which however unfortunately requires access to the transition model of the MDP. \cite{downey2010bayesian} explores a Bayesian variant of TD learning by assuming the specific forms of the stochasticity of the environment. Yet, such method can only be used offline. \cite{xu2018meta} proposed approximate meta-gradients for directly optimizing the prediction error or the policy. However, the assumptions upon which the meta-gradients are derived are very strong and they are very likely unsatisfied for many environments. Thus the inaccuracy of the meta-gradients may lead to failed improvements of sample efficiency. Also, it is a offline $\lambda$-return method that does not use efficient updates of eligibility traces. Some methods have been proposed for online meta-learning, with high extra computational complexities that are intolerable for practical use \cite{mann2016adaptive}. \cite{white2016greedy} proposed an online meta-learning methods that achieves off-policy compatible adaptation of state-based $\lambda$'s without introducing new meta-parameters by locally optimizing a bias-variance tradeoff of a surrogate update target. However, the method optimizes a surrogate target which has huge space to be improved.

To summarize, the existing methods have trouble satisfying the following properties for various difficulties \cite{kearns2000bias,schapire1996worst}:
\begin{enumerate}
\item Achieve the optimization of sample efficiency provably.
\item Achieve adjustment of $\lambda$ online.
\item Achieve compatibility with off-policy learning.
\item Incur low additional computational costs.
\item Does not change the mechanisms of TD($\lambda$) \st{} they may be universally applied as a plugin for adapting $\lambda$.
\item Does not require access to MDP dynamics.
\end{enumerate}

Although this long-history of prior work has helped develop our intuitions about $\lambda$, the available solutions are still far from the use cases outlined above. In this thesis, we build upon \cite{white2016greedy} to achieve a principled method for meta-learning state- or feature-based parametric $\lambda$s \footnote{For the tabular case, it is state-based and for the function approximation case it is feature-based.} which aims directly at the sample efficiency. Under some assumptions, the method has the following properties:

\begin{enumerate}
\item Meta-learns online and uses only incremental computations, incurring the same computational complexity as usual value-based eligibility-trace algorithms, such as TD($\lambda$).
\item Optimizes (approximately) the overall quality of the update targets.
\item Works in off-policy cases.
\item Works with function approximation.
\item Works with adaptive learning rate.
\end{enumerate}

\subsection{Background Knowledge}


Asymptotic convergence in the tabular case is not always useful in practice: we cannot afford infinite episodes nor can we apply tabular methods to environments with the number of states too large to be enumerated, \eg, continuous state spaces. Yet, with limited episodes or function approximations, we lose most of the convergence guarantees. 


Though TD($\bm{\lambda}$), under some cases, can be analyzed using $\bm{\lambda}$-returns, the actual complicated trace updates make it almost unlikely to optimize the value error directly by optimizing $\bm{\lambda}$. Thus in this paper, we develop our ideas upon the \textit{optimization of update targets}. To our knowledge, this is the first paper that explicitly utilize this idea to improve sample efficiency in a principled way. Update targets have important connections to the quality of the learned value estimate \cite{singh1997analytical}, which we ultimately pursue in policy evaluation tasks.

\begin{proposition}
Given suitable learning rates, value estimates using targets with lower overall target error achieve lower overall value error.
\end{proposition}

Starting from the same initial candidate distribution for the learnable parameters, with the same value estimation method, the convergence rate is the same for all the instances towards their update targets. According to stochastic approximation theory, given the same number of update steps, the instances with better update targets yield value estimates with lower MSE compared to the ground truth. The conclusion is very powerful: we can make prediction more sample efficient by using better update targets, which in trace-based prediction means optimizing the difference between the update target and the true value function \wrt{} $\bm{\lambda}$\footnote{Note that with function approximation, updates are never conducted for the features of the terminal states in RL algorithms. Yet, the target error and value error for terminal states should always be set $0$, since these states are always identifiable for they are accompanied by a terminal signal.} \cite{zhao2020meta}. This is the core idea for the $\lambda$-greedy algorithm which we are about to discuss as well as our proposed method.




How are we exactly going to learn faster and achieve better accuracy by changing ${\bm{\lambda}}$? We shall first review a semi-principled approach proposed in \cite{white2016greedy}. The $\lambda$-greedy method is a landmark for online adaptation of $\lambda$: it achieves off-policy compatible meta-learning without introducing meta-parameters. The idea of this paper is going to be the foundation of the contributions of this thesis.

\subsection{$\lambda$-Greedy \cite{white2016greedy}: An Existing Work}
It is intuitively clear that using state-based values of $\lambda$ provides more flexibility than using a constant for all states. Also, state-based $\lambda$s are equipped with well-defined fixed points. Out of these reasons, TD($\lambda$) with different $\lambda$ values for different states has been proposed as a more general formulation of trace-based prediction methods. While preserving good mathematical properties such as convergence to fixed points, this generalization also unlocks significantly more degrees of freedom than only adapting a constant $\lambda$ for every state.

\cite{white2016greedy} investigated the use of state-based $\lambda$s, while outperforming constant $\lambda$ values on some prediction tasks. The authors implicitly conveyed the idea that better update targets lead to better sample efficiency, \ie, update targets with smaller Mean Squared Error (MSE) lead to smaller MSE in learned values. Their proposed online adaptation of $\bm{\lambda}$ is achieved via efficient incremental estimation of statistics about the update targets, gathered by some auxiliary learners. Yet, such method does not seek to improve the overall sample efficiency, because the meta-objectives does not align with the overall target quality.

The idea is to minimize the error between a pseudo target $\tilde{G}(s_t)$ and the true value $\bm{v}(s_t)$, where the pseudo target is defined as:

$$\tilde{G}(s_t) \equiv \tilde{G}_t \equiv R_{t+1} + \gamma_{t+1} [(1 - \lambda_{t+1})V(s_{t+1}) + \lambda_{t+1} G_{t+1}]$$
where $\lambda_{t+1} \in [0, 1]$ and $\lambda_k = 1, \forall k \geq t + 2$.

With this we can find that $\tilde{J}(s_t) \equiv \doubleE[{(\tilde{G}_t - \doubleE[G_t])}^2]$ is a function of only $\lambda_{t+1}$ (given the value estimate $V(s_{t+1})$). The greedy objective corresponds to minimizing the error of the pseudo target $\tilde{G}_t$:

\begin{theorem}[Greedy Objective Minimizer]\label{prop:greedy_minimizer}
Let $t$ be the current timestep and $s_t$ be the current state, the agent takes action at $s_t$ \st{} it will transition into $s_{t+1}$ at $t+1$. Given the pseudo update target $\tilde{G}_t$ of $s_t$, the minimizer $\lambda_{t+1}^*$ of the target error of the state $\tilde{J}(s_t) \equiv \doubleE[{(\tilde{G}_t - \doubleE[G_t])}^2]$ \wrt{} $\lambda_{t+1}$ is:
\begin{equation}\label{eq:white_argmin}
\lambda_{t+1}^* = \frac{(V(s_{t+1}) - \doubleE[G_{t+1}])^2}{\doubleE^2[V(s_{t+1}) - G_{t+1}] + \text{Var}[G_{t+1}]}
\end{equation}
where $G_{t+1}$ is the MC return for state $s_{t+1}$. 
\end{theorem}

The proof can be found in \cite{white2016greedy}. Adaptation based on (\ref{eq:white_argmin}) requires knowledge of $\doubleE[G_{t+1}]$ and $\text{Var}[G_{t+1}]$: and since they are not known, they are approximated using auxiliary learners, that run in parallel with the value learner, for the additional distributional information needed, preferably in an incremental manner. The solutions for learning these have been contributed in \cite{white2016greedy,sherstan2018directly} and illustrated in (\ref{eq:vtd}) and (\ref{eq:dvtd}). These methods learn the variance of $\bm{\lambda}$-return in the same way TD methods learn the value function, however with different ``rewards'' and ``discount factors'' for each state, that can be easily obtained from the known information without incurring new interactions with the environment.

The $\lambda$-greedy method shows strong boost for sample efficiency in some prediction tasks. However, there are two reasons that $\lambda$-greedy has much space to be improved. The first is that the pseudo target $\tilde{G}_t$ used for optimization is not actually the update target used in TD($\bm{\lambda}$) algorithms: we will show that it is rather a compromise for a harder optimization problem; The second is that setting the $\lambda$s to the minimizers does not help the overall quality of the update target: the update targets for every state is controlled by the whole $\bm{\lambda}$, thus unbounded changes of $\bm{\lambda}$ for one state will inevitably affect the other states as well as the overall target error.

From the next chapter, we build upon the mindset provided in \cite{white2016greedy} to propose our method \algoname{}.


%% file: chapter_4_algorithm.tex
\chapter{\algoname{}}\label{cha:\algoname{}}
\textit{The contributed Meta Eligibility Trace Adaptation method\footnote{A brief version was represented in \cite{zhao2020meta}.}. In this chapter, we propose our method \algoname{}, whose goal is to optimize the overall target error for sample efficiency.}

\section{On-policy Decomposition}

We first investigate how the goal of optimizing overall target error can be achieved online. A key to solving this problem is to acknowledge the following points:
\begin{enumerate}
\item 
The states that the agent meets carrying out the policy $\pi$ follows the ``on-policy'' distribution of $d_\pi$.
\item
The overall value error and its surrogate overall target error is mixed of per-state errors according to the $d_\pi$.
\item
Jointly optimizing the overall target error is possible via optimizing each state target error with a consistent optimization method. 
\end{enumerate}

With this, we develop the following theorem to construct this process.

\begin{theorem}\label{thm:nongreedy}
Given an MDP and target policy $\pi$, let $D = diag(d_\pi(s_1), \cdots, d_\pi(s_{|\scriptS|}))$ be the diagonalized on-policy distribution and $\hat{\bm{G}} \equiv [G_{s_1}(\bm{\lambda}), \cdots, G_{s_{|\scriptS|}}(\bm{\lambda})]^T$ be an enumeration random vector of the update targets of all the states in $\scriptS$, in which the targets are all parameterized by a shared parameter vector $\lambda$. Stochastic gradient descent on the overall target error $J(\hat{\bm{G}}, \scriptS) \equiv \frac{1}{2}\cdot \doubleE_\pi \left[ {\| D^{1/2} \cdot (\hat{\bm{G}} - \bm{v})] \|}_2^2 \right]$ can be achieved by doing $1$-step gradient descent on the state target error $J(\hat{G}_s, s) \equiv \frac{1}{2}\cdot (\hat{G}_s(\bm{\lambda}) - v_s)^2$ of update target $\hat{G}_s$ for every state $s$ the agent is in when acting upon $\pi$. Specifically:
$$\nabla_{\bm{\lambda}} J(\hat{\bm{G}}, \scriptS) \propto \sum_{s \sim \pi}{\nabla_{\bm{\lambda}} J(\hat{G}_s, s)}$$
\end{theorem}

\begin{proof}
According to the definition of overall target error,
\begin{equation}
\begin{aligned}
J({\bm{{\bm{\lambda}}}}) \equiv \sum_{s \in \scriptS}{d_\pi(s) \cdot J_s(\bm{\lambda}))} = \sum_{s \in \scriptS}{d_\pi(s) \cdot \doubleE[G^{\bm{{\bm{\lambda}}}}(s) - v(s)]^2}
\end{aligned}\nonumber
\end{equation}
If we take the gradient \wrt{} ${\bm{\lambda}}^{(t+1)}$ we can see that:

\begin{equation}
\begin{aligned}
& \nabla J(\hat{\bm{G}}, \scriptS)\\
& = \sum_{s \in \scriptS}{d_\pi(s) \cdot \nabla J(\hat{G}_s, s)}\\
& \text{push the gradient inside}\\
& = \sum_{s \in \scriptS}{\sum_{k=0}^{\infty}{\doubleP\{s_0 \to s, k, \pi, s_0 \sim d(s_0) \} \cdot \nabla J(\hat{G}_s, s)}}\\
& \text{$\doubleP\{\cdots\}$ is the prob. of $s_0 \to \cdots \to s$ in exactly $k$ steps,}\\
& \text{$s_0$ is sampled from the starting distribution $d(s_0)$.}\\
& = \sum_{s \in \scriptS}{\sum_{k=0}^{\infty}{\sum_{\tau}{\doubleP\{s_0 \xrightarrow{\tau} s, k, \pi, s_0 \sim d(s_0) \} \cdot \nabla J(\hat{G}_s, s)}}}\\
& \text{$\tau$ is a trajectory starting from $s_0$ and transitioning to $s$ in exactly $k$ steps.}\\
& \text{equivalent to summing over the experienced states under $\pi$}\\
& = \sum_{s \sim \pi}{\nabla J(\hat{G}_s, s)}
\end{aligned}\nonumber\normalsize
\end{equation}
\end{proof}

\begin{corollary}
When the agent is acting upon another behavior policy $b$, then the theorem still holds if the gradients of the target error for each state $s$ is weighted by the cumulative product $\rho_{acc}$ of importance sampling ratios from the beginning of the episode until $s$. Specifically:
$$\nabla_{\bm{\lambda}} J(\hat{\bm{G}}, \scriptS) \propto \sum_{s \sim b}{\rho_{acc} \cdot \nabla_{\bm{\lambda}} J(\hat{G}_s, s)}$$
\end{corollary}

\begin{proof}
According to the definition of overall target error,
\begin{equation}
\begin{aligned}
& \nabla J(\hat{\bm{G}}, \scriptS)\\
& = \sum_{s \in \scriptS}{\sum_{k=0}^{\infty}{\sum_{\tau}{\doubleP\{s_0 \xrightarrow{\tau} s, k, \pi, s_0 \sim d(s_0) \} \cdot \nabla J(\hat{G}_s, s)}}}\\
& \text{see the proof of the theorem}\\
& = \sum_{s \in \scriptS}{\sum_{k=0}^{\infty}{\sum_{\tau}{\cdots p(\tau_{k-1}, a_0, s)\pi(a_{k-1}|\tau_{k-1}) \cdot \nabla J(\hat{G}_s, s)}}}\\
& \text{$\tau_i$ is the $i+1$-th state of trajectory $\tau$ and $p(s, a, s^{'})$ is the prob. of $s \xrightarrow{a} s^{'}$ in the MDP}\\
& = \sum_{s \in \scriptS}{\sum_{k=0}^{\infty}{\sum_{\tau}{\cdots p(\tau_{k-1}, a_0, s) \frac{\pi(a_{k-1}|\tau_{k-1})}{b(a_{k-1}|\tau_{k-1})} b(a_{k-1}|\tau_{k-1}) \cdot \nabla J(\hat{G}_s, s)}}}\\
& \text{for the convenience of injecting importance sampling ratios}\\
& = \sum_{s \in \scriptS}{\sum_{k=0}^{\infty}{\sum_{\tau}{\cdots p(\tau_{k-1}, a_0, s) \rho_{k-1} b(a_{k-1}|\tau_{k-1}) \cdot \nabla_{\bm{\lambda}}J_s(\bm{\lambda}))}}}\\
& \text{$\rho_{i} \equiv \frac{\pi(a_i | \tau_i)}{b(a_i | \tau_i)}$ is the importance sampling ratio}\\
& = \sum_{s \in \scriptS}{\sum_{k=0}^{\infty}{\sum_{\tau}{\rho_{0:k-1} \cdot \cdots p(\tau_{k-1}, a_0, s) b(a_{k-1}|\tau_{k-1}) \cdot \nabla J(\hat{G}_s, s)}}}\\
& \text{$\rho_{0:i} \equiv \prod_{v=0}^{i}{\rho_{v}}$ is the cumulative product of importance sampling ratios of $\tau$ from $\tau_0$ to $\tau_i$}\\
& \approx \sum_{s \sim b}{\rho_{acc} \cdot \nabla J(\hat{G}_s, s)}\\
& \text{equivalent to summing over the experienced states under $b$}
\end{aligned}\nonumber\normalsize
\end{equation}
\end{proof}

Note that we can replace the gradient operator with other consistent optimization operators that is consistent with the linearity of the decomposition of errors, \eg{} the partial derivative operators. The idea draws similarity to emphatic methods that use importance sampling ratios to correct the distribution of updates \cite{sutton2015emphatic}. The theorem and the corollary apply for general parametric update targets including the $\bm{\lambda}$-return, whose target error we seek to optimize in this paper. Due to the nature of function approximation, optimizing $\bm{\lambda}$ for each state may inevitably affect the error for other states, \ie, decreasing target error for one state may increase those for others. The theorem shows if we can do gradient descent on the target error of the states according to $\bm{d}_\pi$, we can achieve optimization on the overall target error, assuming the value function is changing slowly.

\section{Approximation of State Gradients}
The problem left for us is to find a way to calculate or approximate the gradients of $\bm{\lambda}$ for the state target errors.

Sadly, we can find that the exact computation of the per-state gradient seems infeasible in the online setting: in the state-based $\lambda$ setting, the $\bm{\lambda}$-return for every state is dependent on every $\lambda$ of every state. These states are unknown before observation.

To remedy this, we propose a method to estimate this gradient by estimating the partial derivatives in the dimensions of the gradient vector, which are further estimated online using auxiliary learners that estimates the distributional information of the update targets. The method can be interpreted as optimizing a bias-variance tradeoff.

\begin{proposition}\label{prop:objective}
Let $t$ be the current timestep and $s_t$ be the current state (the state corresponding to the time $t$). The agent takes action $a_t$ at $s_t$ and will transition into $s_{t+1}$ at $t+1$ while receiving reward $R_{t+1}$. Suppose that $R_{t+1}$ and $G_{t+1}^{\bm{\lambda}}$ are uncorrelated, given the update target $G_t^{\bm{{\bm{\lambda}}}}$ for state $s_t$, the (semi)-partial derivative of the target error $J_{s_t}(\bm{\lambda}) \equiv 1/2 \doubleE[(G_t^{\bm{{\bm{\lambda}}}} - \doubleE[G_t])^2]$  of the state $s_t$ \wrt{} $\lambda_{t+1} \equiv \lambda(s_{t+1})$ is:
\begin{equation}
    \begin{aligned}
    \frac{\partial J_{s_t}(\bm{\lambda})}{\partial \lambda_{t+1}} = & \gamma_{t+1}^2 [\lambda_{t+1} \left[ (V(s_{t+1}) - \doubleE[G_{t+1}^{\bm{{\bm{\lambda}}}}])^2 + Var[G_{t+1}^{\bm{{\bm{\lambda}}}}] \right]\\
    & + (\doubleE[G_{t+1}^{\bm{{\bm{\lambda}}}}] - V(s_{t+1}))
    (\doubleE[G_{t+1}] - V(s_{t+1}))]
    \end{aligned}\nonumber
\end{equation}
And its minimizer \wrt{} $\lambda_{t+1}$ is:
\begin{equation}
\begin{aligned}
& \argmin_{\lambda_{t+1}}{J_{s_t}(\bm{\lambda})} = \frac{
(V(s_{t+1}) - \doubleE[G_{t+1}^{\bm{{\bm{\lambda}}}}])
    (V(s_{t+1}) - \doubleE[G_{t+1}])
}
{(V(s_{t+1}) - \doubleE[G_{t+1}^{\bm{{\bm{\lambda}}}}])^2 + Var[G_{t+1}^{\bm{{\bm{\lambda}}}}]}
\end{aligned}\nonumber
\end{equation}
\end{proposition}
\begin{proof}
\begin{equation}
\begin{aligned}
2 \cdot J_{s_t}(\bm{\lambda}) \equiv \doubleE[(G_t^{\bm{{\bm{\lambda}}}} - \doubleE[G_t])^2] & = \doubleE^2[G_t^{\bm{{\bm{\lambda}}}} - G_t] + Var[G_t^{\bm{\lambda}}]
\end{aligned}
\nonumber
\end{equation}

\begin{equation}
\begin{aligned}
\doubleE[G_t^{\bm{{\bm{\lambda}}}} - G_t] & = \doubleE[R_{t+1}+\gamma_{t+1}((1 - \lambda_{t+1})V(s_{t+1})+\lambda_{t+1}G_{t+1}^{\bm{\lambda}}) -(R_{t+1}+\gamma_{t+1}G_{t+1})]\\
& = \gamma_{t+1} (1 -\lambda_{t+1})V(s_{t+1}) + \gamma_{t+1} \lambda_{t+1} \doubleE[G_{t+1}^{\bm{\lambda}}] - \gamma_{t+1} \doubleE[G_{t+1}]
\end{aligned}
\nonumber
\end{equation}

\begin{equation}
\begin{aligned}
Var[G_t^{\bm{\lambda}}] & = Var[R_{t+1} + \gamma_{t+1} [(1 - \lambda_{t+1})V(s_{t+1}) + \lambda_{t+1}G_{t+1}^{\bm{\lambda}}]]\\
& = Var[R_{t+1}] + \gamma_{t+1}^2 Var[(1 - \lambda_{t+1})V(s_{t+1}) + \lambda_{t+1}G_{t+1}^{\bm{\lambda}}]]\\
& \text{(assuming $R_{t+1}$ \& $G_{t+1}^{\bm{{\bm{\lambda}}}}$ uncorrelated)}\\
& = Var[R_{t+1}] + \gamma_{t+1}^2 \lambda_{t+1}^2Var[G_{t+1}^{\bm{\lambda}}]\\
& \text{($(1 - \lambda_{t+1})V(s_{t+1})$ not random)}
\end{aligned}
\nonumber
\end{equation}

\begin{equation}
\begin{aligned}
& \frac{\partial J_{s_t}(\bm{\lambda})}{\partial \lambda_{t+1}}\\
& \equiv \frac{1}{2}\cdot \frac{\partial}{\partial \lambda_{t+1}} \left(\doubleE[(G_t^{\bm{{\bm{\lambda}}}} - \doubleE[G_t])^2] \right)\\
& = \frac{1}{2}\cdot \frac{\partial}{\partial \lambda_{t+1}} \left((\gamma_{t+1} (1 - \lambda_{t+1})V(s_{t+1}) + \gamma_{t+1} \lambda_{t+1} \doubleE[G_{t+1}^{\bm{\lambda}}] - \gamma_{t+1} \doubleE[G_{t+1}])^2 + Var[R_{t+1}]+ \gamma_{t+1}^2 \lambda_{t+1}^2Var[G_{t+1}^{\bm{\lambda}}]\right)\\
& = \frac{\gamma_{t+1}^2}{2} \cdot \frac{\partial}{\partial \lambda_{t+1}} \left((V(s_{t+1}) - \doubleE[G_{t+1}] - \lambda_{t+1} (V(s_{t+1}) - \doubleE[G_{t+1}^{\bm{\lambda}}]))^2
+ \lambda_{t+1}^2Var[G_{t+1}^{\bm{\lambda}}]\right)\\
& = \frac{\gamma_{t+1}^2}{2} \cdot \frac{\partial}{\partial \lambda_{t+1}} (
(V(s_{t+1}) - \doubleE[G_{t+1}])^2 + \lambda_{t+1}^2 (V(s_{t+1}) - \doubleE[G_{t+1}^{\bm{\lambda}}])^2 \\
&- 2 \lambda_{t+1} (V(s_{t+1}) - \doubleE[G_{t+1}])(V(s_{t+1}) - \doubleE[G_{t+1}^{\bm{\lambda}}]) + \lambda_{t+1}^2Var[G_{t+1}^{\bm{\lambda}}])\\
& = \frac{\gamma_{t+1}^2}{2} \cdot \frac{\partial}{\partial \lambda_{t+1}} (
\lambda_{t+1}^2 (V(s_{t+1}) - \doubleE[G_{t+1}^{\bm{\lambda}}])^2 - 2 \lambda_{t+1} (V(s_{t+1}) - \doubleE[G_{t+1}])(V(s_{t+1}) - \doubleE[G_{t+1}^{\bm{\lambda}}]) + \lambda_{t+1}^2Var[G_{t+1}^{\bm{\lambda}}])\\
& = \gamma_{t+1}^2 [\lambda_{t+1} \left( (V(s_{t+1}) - \doubleE[G_{t+1}^{\bm{{\bm{\lambda}}}}])^2 + Var[G_{t+1}^{\bm{{\bm{\lambda}}}}] \right) + (\doubleE[G_{t+1}^{\bm{{\bm{\lambda}}}}] - V(s_{t+1}))(\doubleE[G_{t+1}] - V(s_{t+1}))]\\
& + \frac{\gamma_{t+1}^2}{2} \left( \lambda_{t+1}^2 \frac{\partial Var[G_{t+1}^{\bm{\lambda}}]}{\partial \lambda_{t+1}} + 2 \lambda_{t+1} (1 - \lambda_{t+1}(V(s_{t+1}) - \doubleE[G_{t+1}^{\bm{{\bm{\lambda}}}}]))\frac{\partial \doubleE[G_{t+1}^{\bm{{\bm{\lambda}}}}]}{\partial \lambda_{t+1}}  \right)
\end{aligned}
\nonumber
\end{equation}

Note that the term $\frac{\gamma_{t+1}^2}{2} \left( \lambda_{t+1}^2 \frac{\partial Var[G_{t+1}^{\bm{\lambda}}]}{\partial \lambda_{t+1}} + 2 \lambda_{t+1} (1 - \lambda_{t+1}(V(s_{t+1}) - \doubleE[G_{t+1}^{\bm{{\bm{\lambda}}}}]))\frac{\partial \doubleE[G_{t+1}^{\bm{{\bm{\lambda}}}}]}{\partial \lambda_{t+1}}  \right)$ is heavily discounted by high orders of $\gamma(\cdot)$ as well as $\lambda(\cdot)$. If we assume that they are negligible, \ie{} not taking the partial derivatives of the expectation or the variance and regard them as $0$, we can obtain the semi-partial derivative, which is a reasonable $1$-step approximation of the derivative. If the algorithm runs offline, it is possible to compute the semi-gradient fully.

The minimizer is achieved by setting the partial derivative $0$.
\end{proof}

This proposition constructs a way to estimate the partial derivative that corresponds to the dimension of $\lambda_{t+1}$ in $\nabla \bm{\lambda}$, if we know or can effectively estimate the statistics of $\doubleE[G_{t+1}]$, $\doubleE[G_{t+1}^{\bm{\lambda}}]$ and $Var[G_{t+1}^{\bm{\lambda}}]$. This proposition also provides the way for finding a whole series of partial derivatives and also naturally yields a multi-step method of approximating the full gradient $\nabla_{\bm{\lambda}} \doubleE[(G_t^{\bm{{\bm{\lambda}}}} - \doubleE[G_t])^2]$. The partial derivative in the proposition is achieved by looking $1$-step into the future. We can also look more steps ahead, and get the partial derivatives \wrt{} $\lambda^{(t+2)}, \cdots$. These partial derivatives can be computed with the help of the auxiliary tasks as well. The more we assemble the partial derivatives, the closer we get to the full gradient. However, in our opinion, $1$-step is still the most preferred not only because it can be obtained online every step without the need of buffers but also for its dominance over other dimensions of $\bm{\lambda}$: the more steps we look into the future, the more the corresponding $\lambda$s of the states are more heavily discounted by the earlier $\gamma$'s and $\lambda$s. This result is extended naturally if we were to do the adaptations offline, in which case the partial derivatives and the gradient can be exactly computed with additional computational costs. Figure \ref{fig:interdependency} (a) and (b) compare the interdependency of state-$\lambda$'s and state update targets, under the true case and the $1$-step \algoname{} approximation case respectively.

\begin{figure*}
\centering

\subfloat[True interdependency of state-$\lambda$'s and state update targets]{
\captionsetup{justification = centering}
\includegraphics[width=0.9\textwidth]{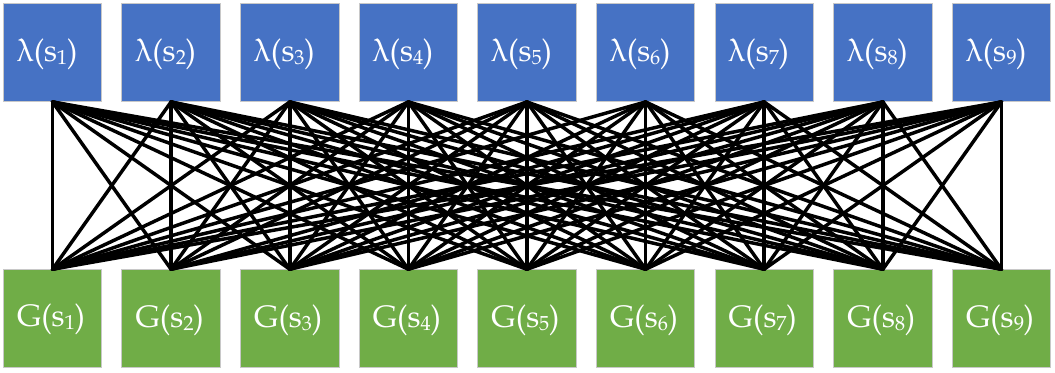}}

\subfloat[Approximated interdependency of $1$-step \algoname{}]{
\captionsetup{justification = centering}
\includegraphics[width=0.9\textwidth]{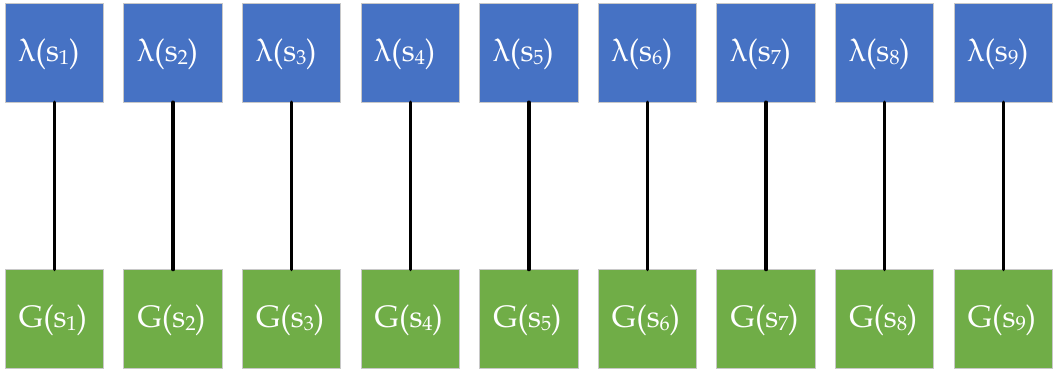}}

\caption{Interdependency among state-$\lambda$'s and state update targets. With the increment of steps into the future, according to Proposition \ref{prop:objective}, less connections would be neglected.}
\label{fig:interdependency}
\end{figure*}

It is interesting to observe that the yielded minimizer is a generalization of (\ref{eq:white_argmin}): the minimizer of the greedy target error can be achieved by setting $G_{t+1}^{\bm{\lambda}} = G_{t+1}$. In practice, given an unknown MDP, the distributional information of the targets, \eg{} $\doubleE[G_{t+1}]$, $\doubleE[G_{t+1}^{\bm{\lambda}}]$ and $Var[G_{t+1}^{\bm{\lambda}}]$, can only be estimated. However, such estimation has been proved viable in both offline and online settings of TD($\lambda$) and the variants, using supervised learning and auxiliary tasks using the direct VTD method \cite{sherstan2018directly}, respectively. This means the optimization for the ``true'' target error is as viable as the $\lambda$-greedy method proposed in \cite{white2016greedy}, while it requires more complicated estimations than that for the ``greedy'' target error: we need the estimates of $\doubleE[G_{t+1}]$, $\doubleE[G_{t+1}^{\bm{\lambda}}]$ and $Var[G_{t+1}^{\bm{\lambda}}]$, while for (\ref{eq:white_argmin}) we only need the estimation of $\doubleE[G_{t+1}]$ and $Var[G_{t+1}]$.

\section{Trust Region Optimization}
The auxiliary learners are also dependent on $\bm{\lambda}$, which brings new challenges. The optimization of the true state target error, \ie{} the MSE between $\bm{\lambda}$-return and the true value, together with the auxiliary estimation, brings new challenges: the auxiliary estimates are learnt online and requires the stationarity of the update targets. This means if a $\lambda$ for one state is changed dramatically, the auxiliary estimates of $\doubleE[G_{t+1}^{\bm{\lambda}}]$ and $Var[G_{t+1}^{\bm{\lambda}}]$ will be destroyed, since they depend on each element in ${\bm{\lambda}}$ (whereas in $\lambda$-greedy, the pseudo targets require no $\bm{\lambda}$-controlled distributional information). If we cannot handle such challenge, either we end up with a method that have to wait for some time after each change of $\bm{\lambda}$ or we end up with $\lambda$-greedy, bearing the high bias towards the MC return and disconnection from the overall target error.

Adjusting ${\bm{\lambda}}$ without destroying the auxiliary estimates is a core problem. We tackle such optimization by noticing that the expectation and variance of the update targets are continuous and differentiable \wrt{} $\bm{\lambda}$. Thus, a small change on $\lambda_{t+1}$ only yields a bounded shift of the estimates of the auxiliary tasks.

If we use small enough steps of the estimated gradients to change $\bm{\lambda}$, we can stabilize the auxiliary estimates since they will not deviate far and will be corrected by the TD updates quickly. This method inherits the ideas of trust region methods used in optimizing the dynamic systems.

Combining the approximation of gradient and the decomposed one-step optimization method, we now have an online method to optimize $\bm{\lambda}$ to achieve approximate optimization of the overall target error, which we name as \algoname{}. This method can be jointly used with value learning, serving as a plugin, to adapt $\bm{\lambda}$ real-time. The policy evaluation carried out by the value learner directs the value estimates toward the targets and \algoname{} jointly optimizes the targets. We present the pseudocode of \algoname{} in Alg. \ref{alg:MTA_PE} and present the mechanisms in Fig. \ref{fig:\algoname{}}.

\begin{figure*}
\centering
\includegraphics[width=0.9\textwidth]{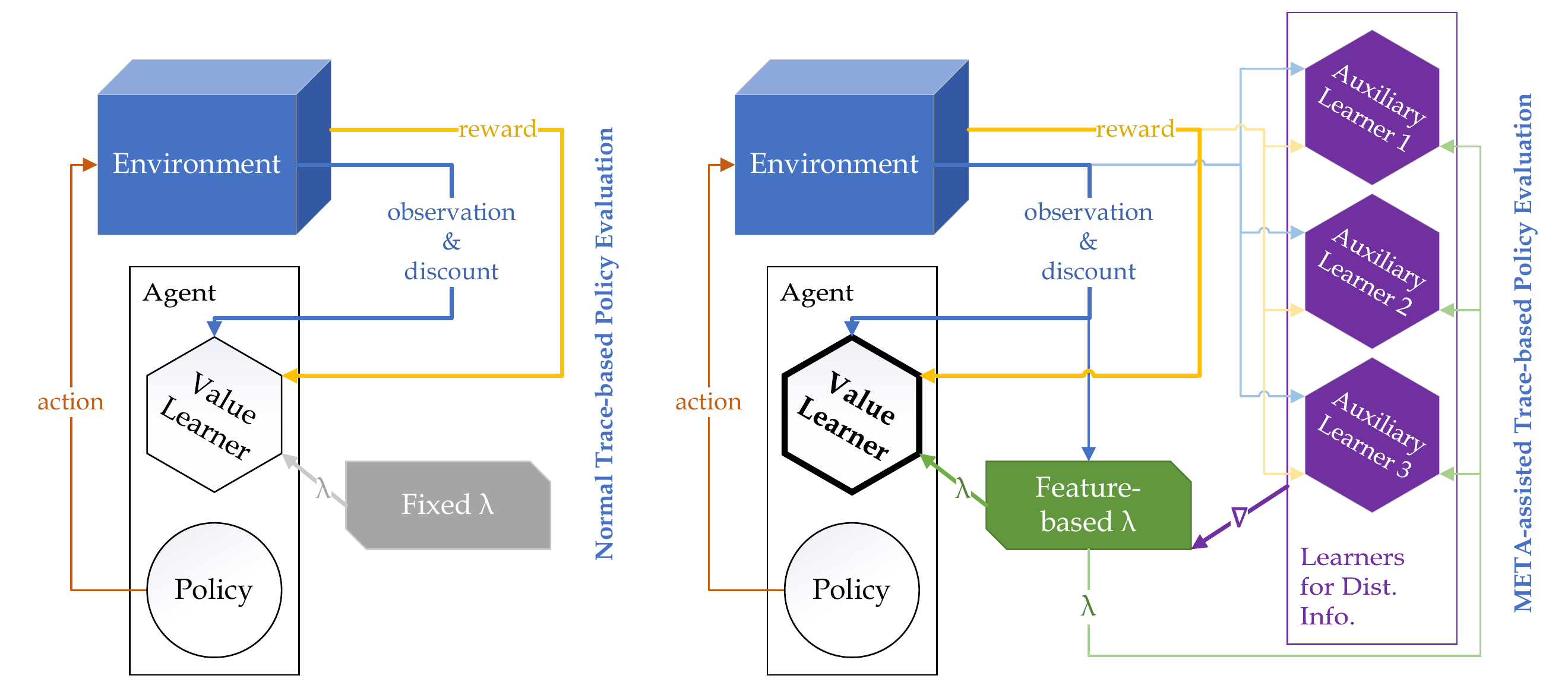}
\caption{Mechanisms for \algoname{}-assisted trace-based policy evaluation: the auxiliary learners learn the distributional information in parallel to the value learner and provide the approximated gradient for the adjustments of $\bm{\lambda}$.}
\label{fig:\algoname{}}
\end{figure*}




\begin{algorithm*}[htbp]
\caption{\algoname{}-assisted Online Policy Evaluation}
\label{alg:MTA_PE}
Initialize weights for the value learner and those for the auxiliary learners that learns $\hat{\doubleE}[G_t]$, $\hat{\doubleE}[G_t^{\bm{\lambda}}]$ and $\hat{V}ar[G_t^{\bm{\lambda}}]$\\
\For{episodes}{
    $\rho_{acc} = 1$; \textcolor{darkgreen}{//initialize cumulative product of importance sampling ratios}\\
    Set traces for value learner and auxiliary learners to be $\bm{0}$;\\
    $\bm{x}_0 = \text{initialize}(\scriptE{})$; \textcolor{darkgreen}{//Initialize the environment $\scriptE{}$ and get the initial feature (observation) $\bm{x}_0$}\\
    
    \While{$t \in \{0, 1, \dots\}$ until terminated}{
    \textcolor{darkgreen}{//INTERACT WITH ENVIRONMENT}\\
    $a_t \sim b(\bm{x}_{t})$; \textcolor{darkgreen}{//sample $a_t$ from behavior policy $b$}\\
    $\rho_t = {\pi(a_{t}, \bm{x}_{t})} / {b(a_{t}, \bm{x}_{t})}$; $\rho_{acc} = \rho_{acc} \cdot \rho_t$; \textcolor{darkgreen}{//get and accumulate importance sampling ratios}\\
    $\bm{x}_{t+1}, \gamma_{t+1} = \text{step}(a_t)$;
    \textcolor{darkgreen}{//take action $a_t$, get feature (observation) $\bm{x}_{t+1}$ and discount factor $\gamma_{t+1}$}\\
    \textcolor{purple}{//AUXILIARY TASKS}\\
    learn $\hat{\doubleE}[G_t]$, $\hat{\doubleE}[G_t^{\bm{\lambda}}]$ and $\hat{V}ar[G_t^{\bm{\lambda}}]$; \textcolor{purple}{//using direct VTD \cite{sherstan2018directly} with trace-based TD methods, \eg, true online GTD($\bm{\lambda}$) \cite{hasselt2014true}}\\
    
    \textcolor{purple}{//APPROXIMATE SGD ON OVERALL TARGET ERROR}\\
    
    $\lambda_{t+1} = \lambda_{t+1} - \kappa \gamma_{t+1}^2 \rho_{acc}$\\
    $\left[\lambda_{t+1} \left( (V(\bm{x}_{t+1}) - \hat{\doubleE}[G_{t+1}^{\bm{\lambda}}])^2 + \hat{V}ar[G_{t+1}^{\bm{\lambda}}] \right) + (\hat{\doubleE}[G_{t+1}^{\bm{\lambda}}] - V(\bm{x}_{t+1}))(\hat{\doubleE}[G_{t+1}] - V(\bm{x}_{t+1}))\right]$; \textcolor{purple}{// change $\lambda_{t+2}, \cdots$ when using multi-step approximation of the gradient}\\
    
    \textcolor{darkgreen}{//LEARN VALUE}\\
    learn $V(\bm{x}_{t})$ using a trace-based TD method;
}
}
\end{algorithm*}

\section{Discussions and Insights}

\subsection{Hyperparameter Search}

The proposed method \algoname{} trades the hyperparameter search for $\lambda$ with $\kappa$. However, $\kappa$ gives the algorithm the ability to have state-based $\lambda$s: state or feature (observation) based $\bm{\lambda}$ can lead to better convergence compared to fixing $\lambda$ for all states. Such potential may \textit{never} be achieved by searching a fixed $\lambda$. Let us consider the tabular case, where the hyperparameter search for constant $\bm{\lambda}$ is equivalent to searching along the $\bm{1}$ direction inside a $|\scriptS|$-dimensional box. By replacing the hyperparameter $\lambda$ with $\kappa$, we extend the search direction into the whole ${[0, 1]}^{|\scriptS|}$. The new degrees of freedom are crucial to the performance.

\subsection{Reliance on Auxiliary Tasks}\label{subsection:reliance}

The \algoname{} updates assume that $\hat{\doubleE}[G_t]$, $\hat{\doubleE}[G_t^{\bm{\lambda}}]$ and $\hat{V}ar[G_t^{\bm{\lambda}}]$ can be well estimated by the auxiliary tasks. This is very similar to the idea of actor changing the policy upon the estimation of the values of the critic in the actor-critic methods. To implement this, we can add a buffer period for the estimates to be stable at the beginning of the learning process; Additionally, we should set the learning rates of the auxiliary learners higher than the value learner \st{} the auxiliary tasks are learnt faster, resembling the guidelines for setting learning rates of actor-critic. 
With the buffer period, we can also view \algoname{} as approximately equivalent to offline hyperparameter search of $\bm{\lambda}$, where with \algoname{} we first reach a relatively stable accuracy and then adjust $\bm{\lambda}$ to slowly slide to fixed points with lower errors. Also, it is worth noting that \algoname{} is in theory compatible with fancier settings of learning rate, since the meta-adaptation is independent of the values of the learning rate.

\subsection{Generalization and Function Approximation}
In the case of function approximation, the meta-learning of $\bm{\lambda}$-greedy is still fully controlled by the value function and the two additional learned statistics but cannot directly make use of the features of the state itself. Whereas in \algoname{}, we can use a parametric function of $\bm{\lambda}$ and performs gradient descent on it to make use of the state features. This feature is helpful for generalization and can be very effective when the state features contain rich information (good potential to be used with deep neural networks). This is to be demonstrated in the experiments.





\subsection{From Prediction to Control}

Within the control tasks where the quality of prediction is crucial to the policy improvement, it is viable to apply \algoname{} to enhance the policy evaluation process. \algoname{} is a trust region method, which requires the policy to be also changing smoothly, \st{} the shift of values can be bounded. This constraint leads us naturally to the actor-critic architectures, where the value estimates can be used to improve a continuously changed parametric policy. We provide the pseudocode of \algoname{}-assisted actor-critic control in Algorithm \ref{alg:MTA_AC}.

\begin{algorithm*}[htbp]
\caption{\algoname{}-assisted Online Actor-Critic}
\label{alg:MTA_AC}
Initialize weights for the value learner and those for the auxiliary learners that learns $\hat{\doubleE}[G_t]$, $\hat{\doubleE}[G_t^{\bm{\lambda}}]$ and $\hat{V}ar[G_t^{\bm{\lambda}}]$\\
Initialize parameterized policies $\pi(\cdot | \theta_\pi)$ and $b(\cdot | \theta_b)$;\\
\For{episodes}{
    Set traces for value learner and auxiliary learners to be $\bm{0}$;\\
    $\bm{x}_0 = \text{initialize}(\scriptE{})$;\\
    
    \While{$t \in \{0, 1, \dots\}$ until terminated}{
    \textcolor{darkgreen}{//INTERACT WITH ENVIRONMENT}\\
    $a_t \sim b(\bm{x}_{t})$; $\rho_t = {\pi(a_{t}, \bm{x}_{t})} / {b(a_{t}, \bm{x}_{t})}$; $\rho_{acc} = \rho_{acc} \cdot \rho_t$;\\
    $\bm{x}_{t+1}, \gamma_{t+1} = \text{step}(a_t)$;\\
    \textcolor{purple}{//AUXILIARY TASKS and SGD ON OVERALL TARGET ERROR}\\
    learn $\hat{\doubleE}[G_t]$, $\hat{\doubleE}[G_t^{\bm{\lambda}}]$ and $\hat{V}ar[G_t^{\bm{\lambda}}]$;\\
    $\lambda_{t+1} = \lambda_{t+1} - \kappa \gamma_{t+1}^2 \rho_{acc}$\\
    $\left[\lambda_{t+1} \left( (V(\bm{x}_{t+1}) - \doubleE[G_{t+1}^{\bm{\lambda}}])^2 + Var[G_{t+1}^{\bm{\lambda}}] \right) + (\doubleE[G_{t+1}^{\bm{\lambda}}] - V(\bm{x}_{t+1}))(\doubleE[G_{t+1}] - V(\bm{x}_{t+1}))\right]$;\\
    
    \textcolor{darkgreen}{//LEARN VALUE}\\
    learn $V(\bm{x}_{t})$ using a trace-based TD method;\\
    
    \textcolor{red}{//LEARN POLICY}\\
    One (small) step of policy gradient (actor-critic) on $\theta_\pi$;\\
}
}
\end{algorithm*}

%% file: chapter_5_experiments.tex
\chapter{Experimental Studies}\label{chap:experiments}




In this chapter, we examine the empirical behavior of the proposed method, \algoname{}, by comparing it to the baselines true online TD($\lambda$) \cite{seijen2015true} and true online GTD($\lambda$) \cite{hasselt2014true} as well as the $\lambda$-greedy method \cite{white2016greedy}\footnote{Implementation is open-source at \url{https://github.com/PwnerHarry/META}}. For all the three sets of tests, we start adapting $\lambda$s from $1$, which is the same as $\lambda$-greedy \cite{white2016greedy}. This setting is enabled by using $\lambda(\bm{x}) = 1 - \bm{w}_\lambda^T\bm{x}$ as the function approximator of the parametric $\lambda$, with all the weights initialized as $\bm{0}$.

\section{RingWorld: Tabular Case, Low Variance}

\begin{figure}
\centering

\subfloat[$\gamma = 0.95$, $\langle 0.4, 0.35 \rangle$]{
\captionsetup{justification = centering}
\includegraphics[width=0.32\textwidth]{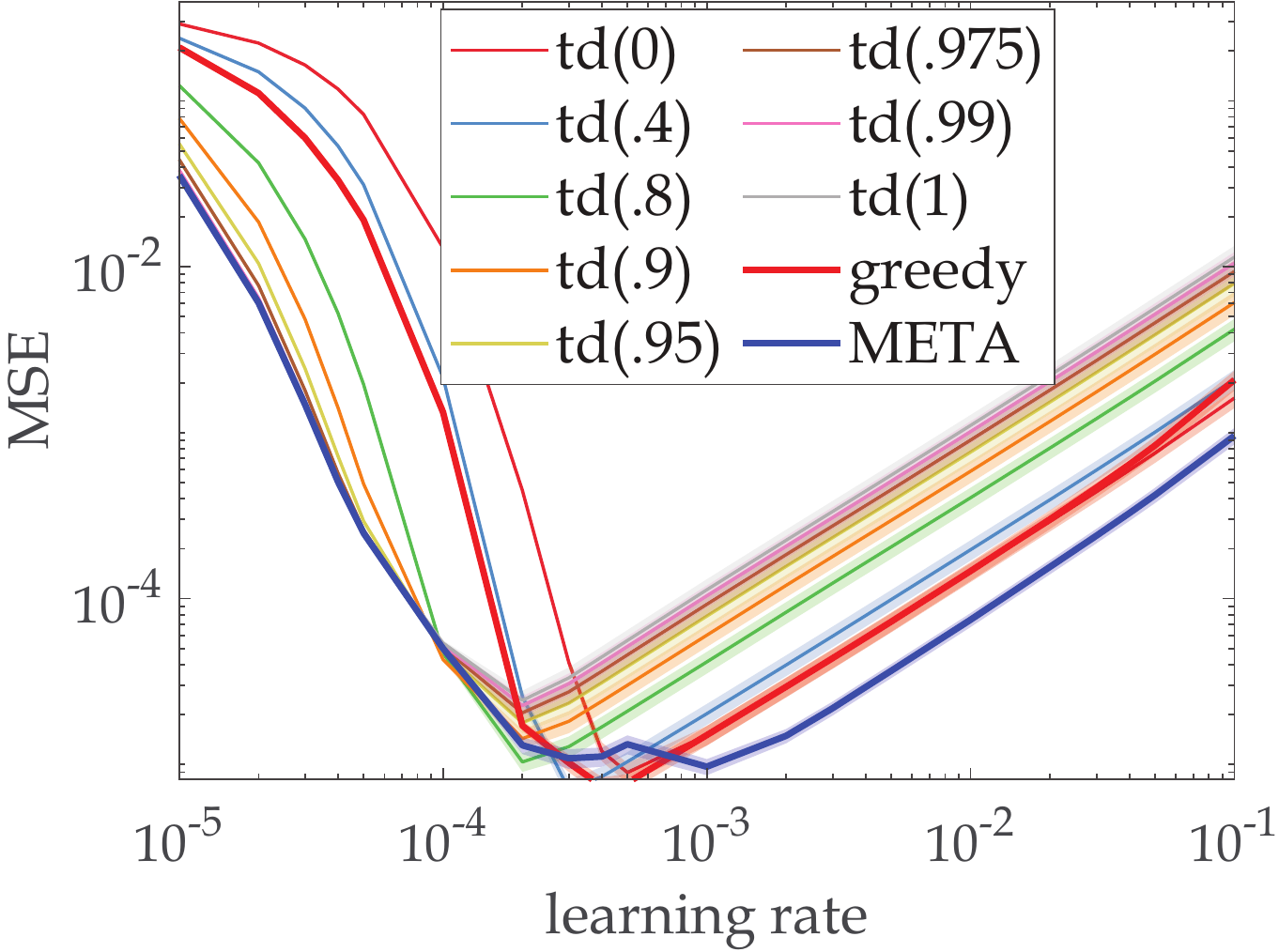}}
\hfill
\subfloat[$\gamma = 0.95$, $\langle 0.3, 0.25 \rangle$]{
\captionsetup{justification = centering}
\includegraphics[width=0.32\textwidth]{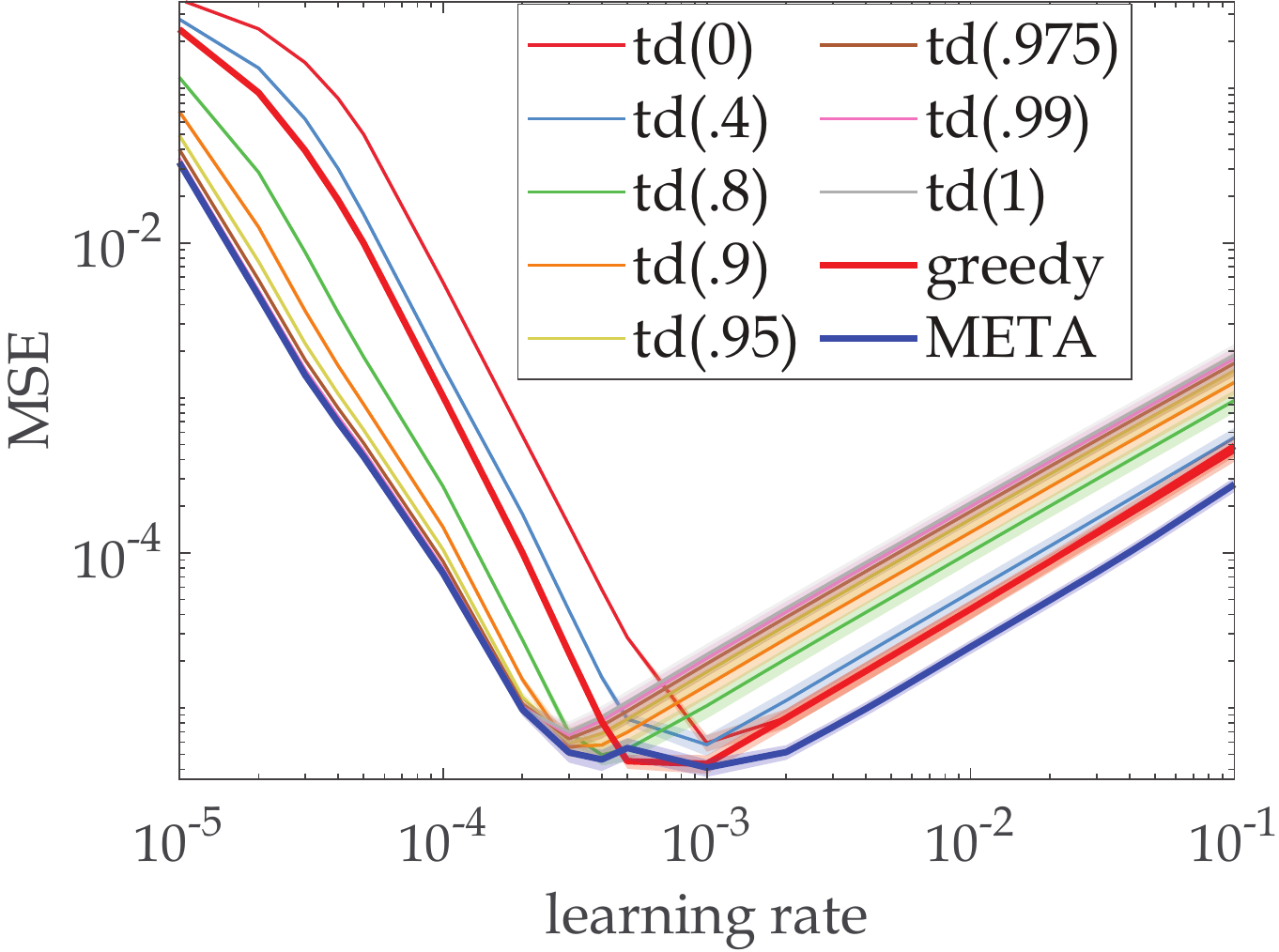}}
\hfill
\subfloat[$\gamma = 0.95$, $\langle 0.2, 0.15 \rangle$]{
\captionsetup{justification = centering}
\includegraphics[width=0.32\textwidth]{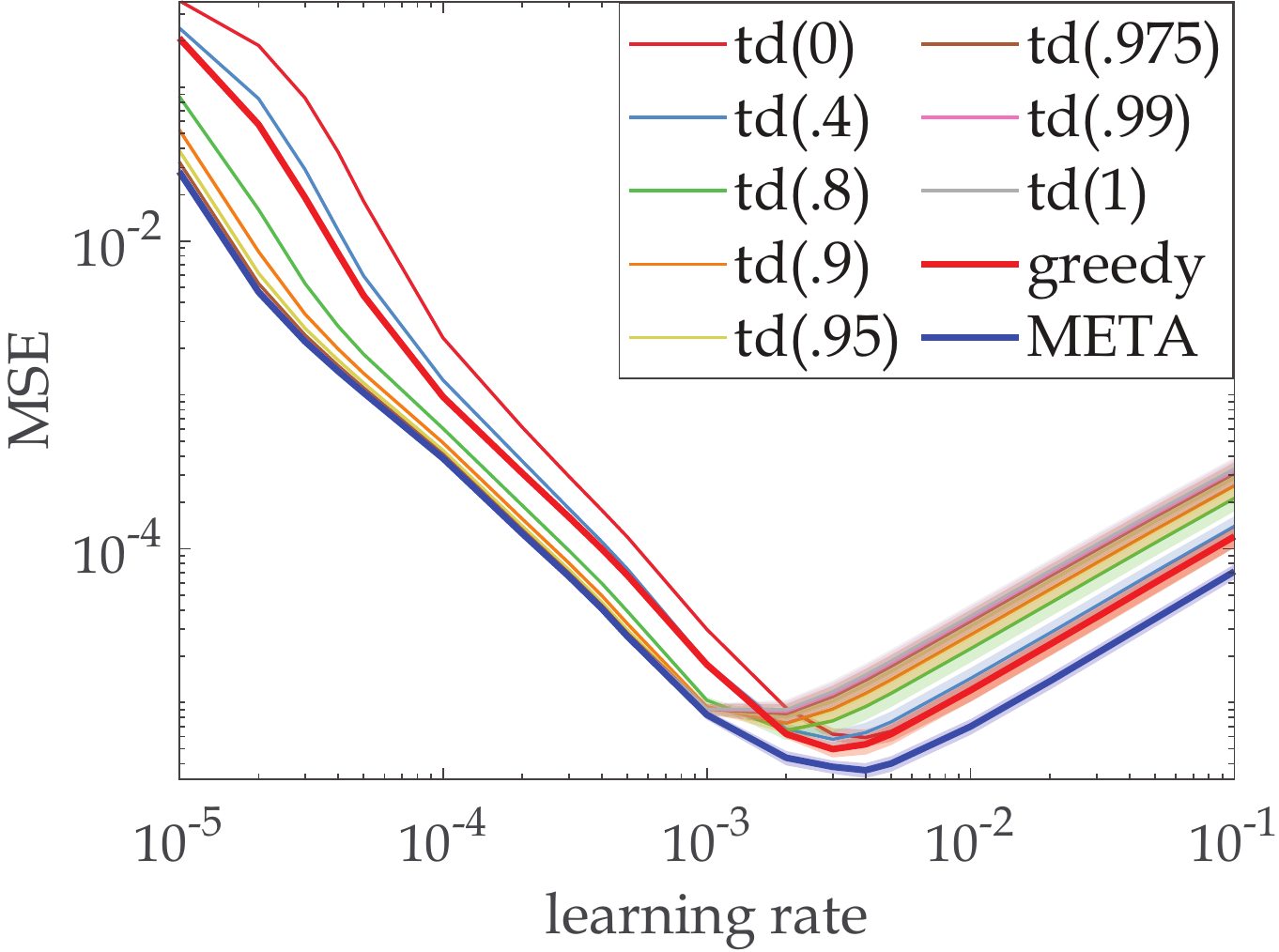}}

\subfloat[$\alpha = 0.01$, $\kappa = 0.01$, $\langle 0.4, 0.35 \rangle$]{
\captionsetup{justification = centering}
\includegraphics[width=0.32\textwidth]{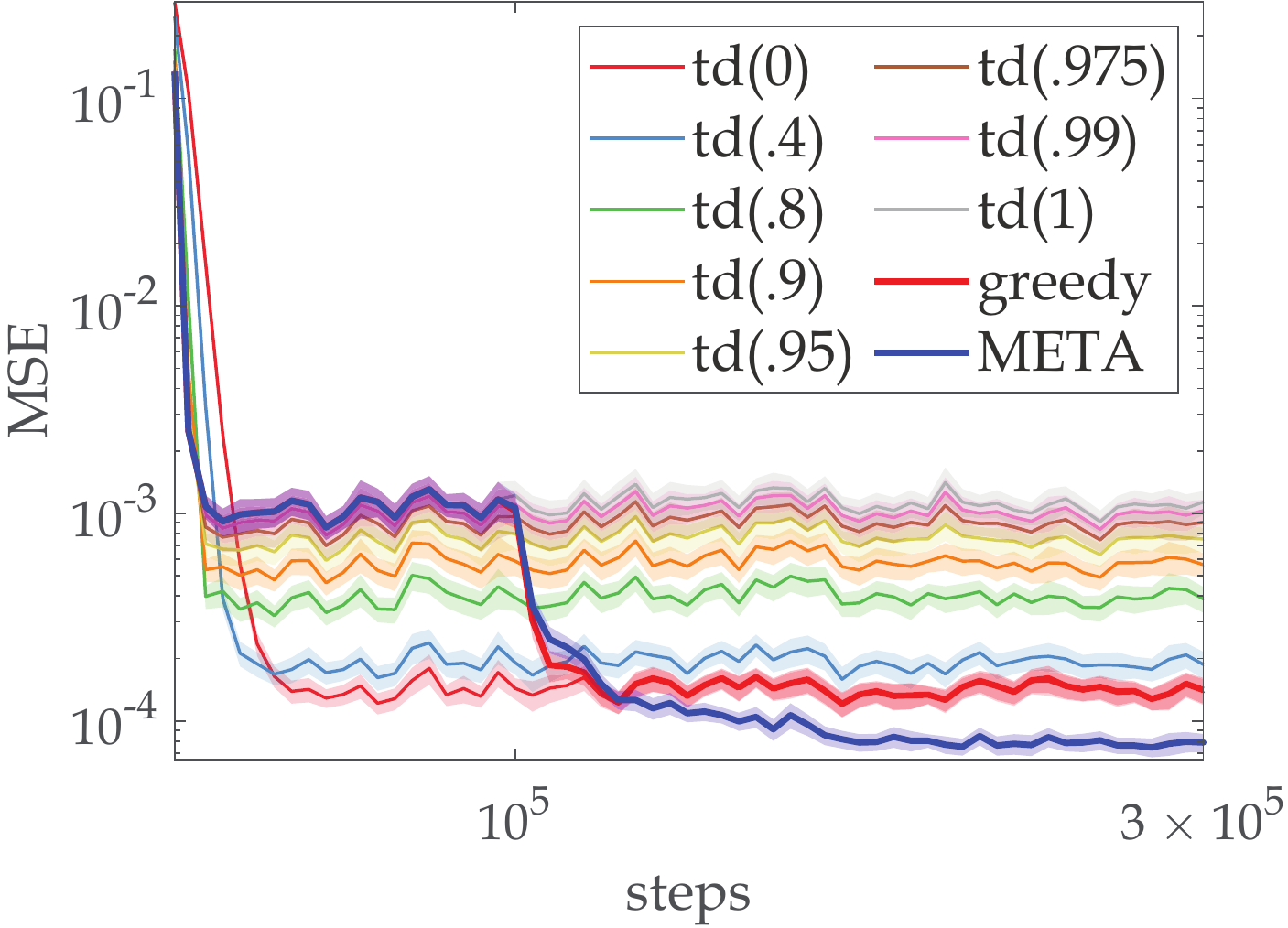}}
\hfill
\subfloat[$\alpha = 0.01$, $\kappa = 0.1$, $\langle 0.3, 0.25 \rangle$]{
\captionsetup{justification = centering}
\includegraphics[width=0.32\textwidth]{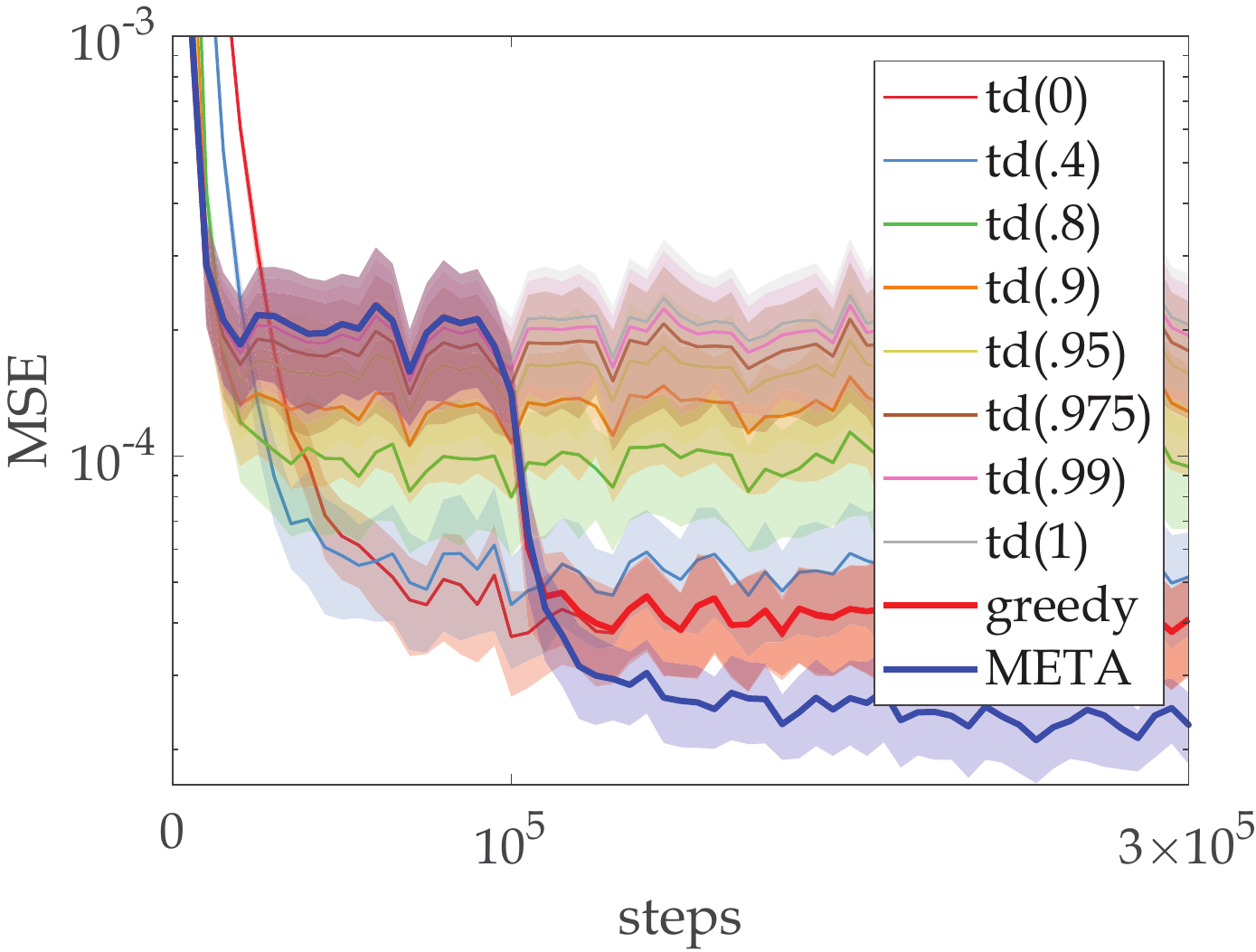}}
\hfill
\subfloat[$\alpha = 0.05$, $\kappa = 0.1$, $\langle 0.2, 0.15 \rangle$]{
\captionsetup{justification = centering}
\includegraphics[width=0.32\textwidth]{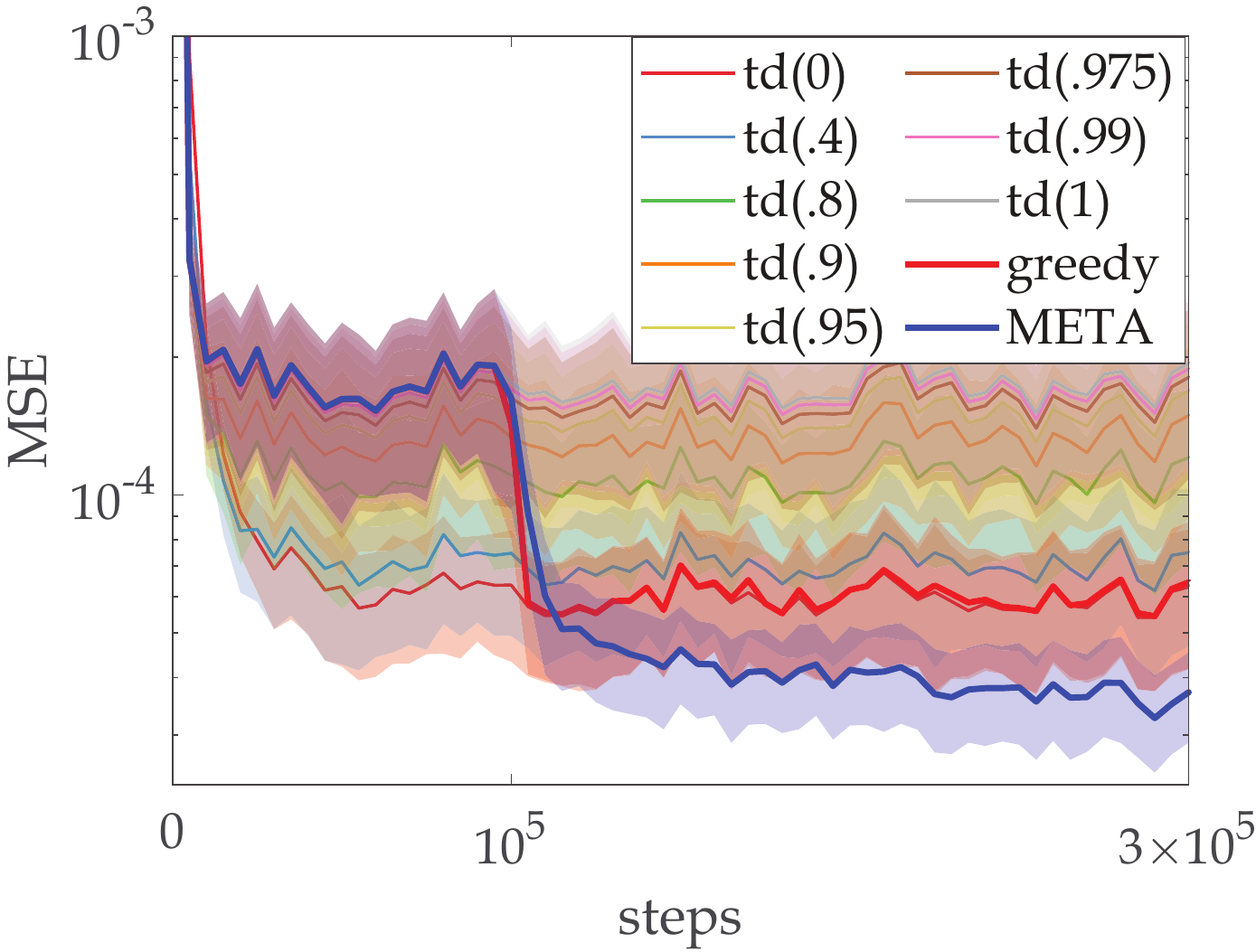}}

\caption{U-shaped curves and learning curves for \algoname{}, $\lambda$-greedy and the baselines on RingWorld, under three pairs of behavior-target policies. For (a), (b) and (c), the $x$-axes represent the values of the learning rate $\alpha$ for prediction (or the critic), while the $y$-axes represent the overall value errors. Each point in the graphs contains the mean (solid) and standard deviation (shaded) collected from $240$ independent runs, with $10^{6}$ steps for prediction; For (d), (e) and (f), the $x$-axis represents the steps of learning and $y$-axis is the same as (a), (b) and (c). We choose one representative case for each U-shaped curve corresponding to different policy pairs and plot the corresponding learning curves. In these learning curves, the best known values for the hyperparameters are used. The buffer period lengths are $10^{5}$ steps ($10\%$). The buffer period and the adaptation period have been ticked on the $x$-axes. The rest of the steps for (d), (e) and (f) have been cut off since there is no significant change afterwards.}
\label{fig:ringworld}
\end{figure}

This set of experiments focuses on a low-variance environment, the $11$-state ``ringworld'' \cite{white2016greedy}, in which the agent moves either left or right in a ring of states. The state transitions are deterministic and rewards only appear in the terminal states. In this set of experiments, we stick to the tabular setting and use true online TD($\bm{\lambda}$) \cite{hasselt2014true} as the learner\footnote{We prefer true online algorithms since they achieve the exact equivalence of the bi-directional view of $\lambda$-returns.}, for the value estimate as well as all the auxiliary estimates. As discussed in \ref{subsection:reliance}, for the accuracy of the auxiliary learners, we double their learning rate so that they can adapt to the changes of the estimates faster. We select $3$ pairs of behavior-target policies: 1) the behavior policy goes left with $0.4$ probability while the target policy does w.p. $0.35$; 2) behavior $0.3$ and target $0.25$; 3) behavior $0.2$ and target $0.15$. The three pairs of policies are increasingly greedy. The baseline true online TD has $2$ hyperparameters ($\alpha$ \& $\lambda$) and so does \algoname{} ($\alpha$ \& $\kappa$), excluding those for the auxiliary learners. We test the two methods on grids of hyperparameter pairs. More specifically, for the baseline true online TD, we test on $\langle \alpha, \lambda \rangle \in \{ 10^{-5}, \dots, 5\times10^{-5}, 10^{-4}, \dots, 5\times10^{-4}, \dots, 5\times10^{-2}, 10^{-1} \} \times \{0, 0.4, 0.8, 0.9, 0.95, 0.975, 0.99, 1\}$ while for \algoname{}, $\langle \alpha, \kappa \rangle \in \{ 10^{-5}, \dots, 5\times10^{-5}, 10^{-4}, \dots, 5\times10^{-4}, \dots, 5\times10^{-2}, 10^{-1} \} \times \{10^{-7}, \dots, 10^{-1}\}$. The results are presented as the U-shaped curves in Fig. \ref{fig:ringworld}. We plot the curves of the baseline under different $\lambda$s and the best performance that \algoname{} could get under each learning rate. The detailed results for the three pairs of policies are presented in Table \ref{tab:ringworld_behavior_04_target_035}, \ref{tab:ringworld_behavior_03_target_025} and \ref{tab:ringworld_behavior_02_target_015}, respectively.
\input{tab_ringworld_behavior_04_target_035.tex}
\input{tab_ringworld_behavior_03_target_025.tex}
\input{tab_ringworld_behavior_02_target_015.tex}
\par
The best performance of a fine-tuned baseline can be extracted from the figure by combining the lowest points of the set of curves obtained by the baseline under different $\lambda$s. But still, the fine-tuned \algoname{} provides better performance, especially when the learning rate is relatively high. We can say that once \algoname{} is fine-tuned, it provides significantly better performance that the baseline algorithm can possibly achieve, since it meta-learns state-based $\bm{\lambda}$s, which allow it to go beyond the scope of the optimization of the baseline. The results can also be interpreted as \algoname{} having less sensitivity to the learning rate hyperparameter than the baseline true online TD.
It is also interesting to notice that the greedier the policies, the larger the performance boost that \algoname{} can provide.

\section{FrozenLake: Linear Function Approximation with High Variance}

\begin{figure*}
\centering

\subfloat[FrozenLake, $\gamma = 0.95$]{
\captionsetup{justification = centering}
\includegraphics[width=0.48\textwidth]{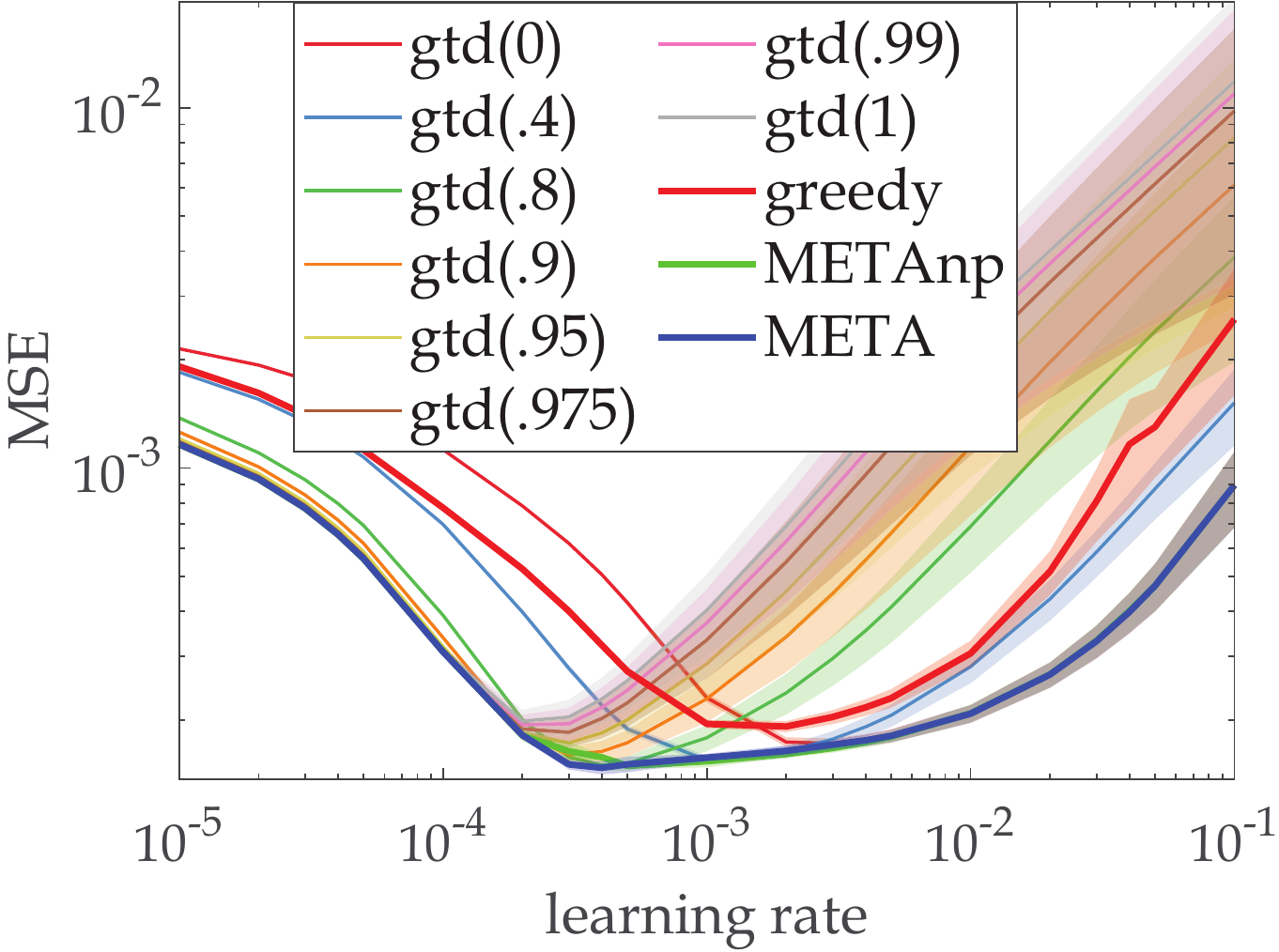}}
\hfill
\subfloat[FrozenLake, $\alpha = \beta = 0.0001$, $\kappa = 10^{-6}$, $\kappa_{\text{np}} = 10^{-4}$]{
\captionsetup{justification = centering}
\includegraphics[width=0.48\textwidth]{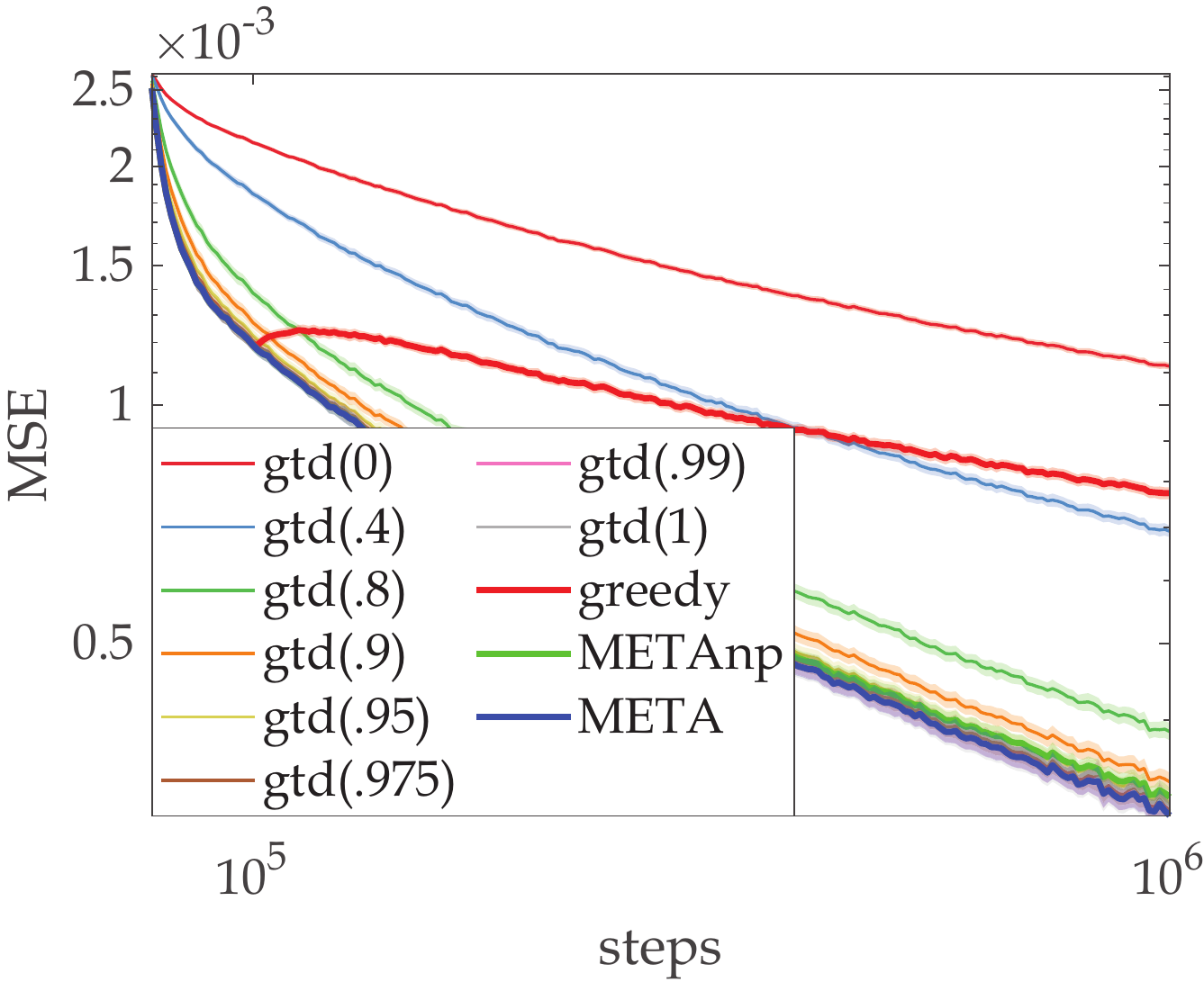}}

\caption{U-shaped curves and learning curves for \algoname{}, $\lambda$-greedy and the baselines on FrozenLake. For (a), the $x$-axis represents the values of the learning rate $\alpha$ for prediction, while the $y$-axis represents the overall value error. Each point in the graph contains the mean (solid) and standard deviation (shaded) collected from $240$ independent runs, with $10^{6}$ steps for prediction; For (d), the $x$-axis represents the steps of learning and $y$-axis is the same as (a). In the learning curves, the best known values for the hyperparameters are used. The buffer period lengths are $10^{5}$ steps ($10\%$). The buffer period and the adaptation period have been ticked on the $x$-axes.}
\label{fig:frozenlake}
\end{figure*}

\input{tab_frozenlake.tex}

This set of experiments is carried out on a higher variance environment, the ``$4\times4$'' FrozenLake, in which the transitions are stochastic and the agent seeks to fetch a frisbee back on a frozen lake surface with holes from the northwest to the southeast. We craft an exploratory policy that takes the actions with equal probability, and a heuristic policy that has $0.3$ probability for going south or east, $0.2$ for going north or west. This time we use the linear function approximation and  true online GTD($\lambda$) with discrete tile coding ($4$ tiles, $4$ randomly offset tilings). We set the second learning rate $\beta$ introduced in true online GTD($\lambda$) to be the same as $\alpha$ (for the value learners as well as the auxiliary learners in all the compared algorithms). Additionally, we remove the parametric setting of $\bm{\lambda}$ to get a method that we call ``\algoname{}(np)'', and which demonstrates the potential of a parametric feature (observation) based $\bm{\lambda}$. The U-shaped curves, obtained using the exact same settings as in the ringworld tests, are provided in Fig. \ref{fig:frozenlake}. The detailed results are provided in Table \ref{tab:frozenlake}.
\par
We observe similar patterns with the first set of experiments. Comparing \algoname{} with ``\algoname{}(np)'', the generalization provided by the parametric $\bm{\lambda}$ is beneficial: in panel (b) we observe generally better performance and in panel (e) we see that a parametric $\lambda$ has better sample efficiency. This would suggest that using parametric $\lambda$ in environments with relatively smooth dynamics would be generally beneficial for sample efficiency.

\section{MountainCar: On-Policy Actor-Critic Control with Linear Function Approximation}

\begin{figure*}
\centering

\subfloat[MountainCar, $\gamma = 1$, $\eta = 1$]{
\captionsetup{justification = centering}
\includegraphics[width=0.48\textwidth]{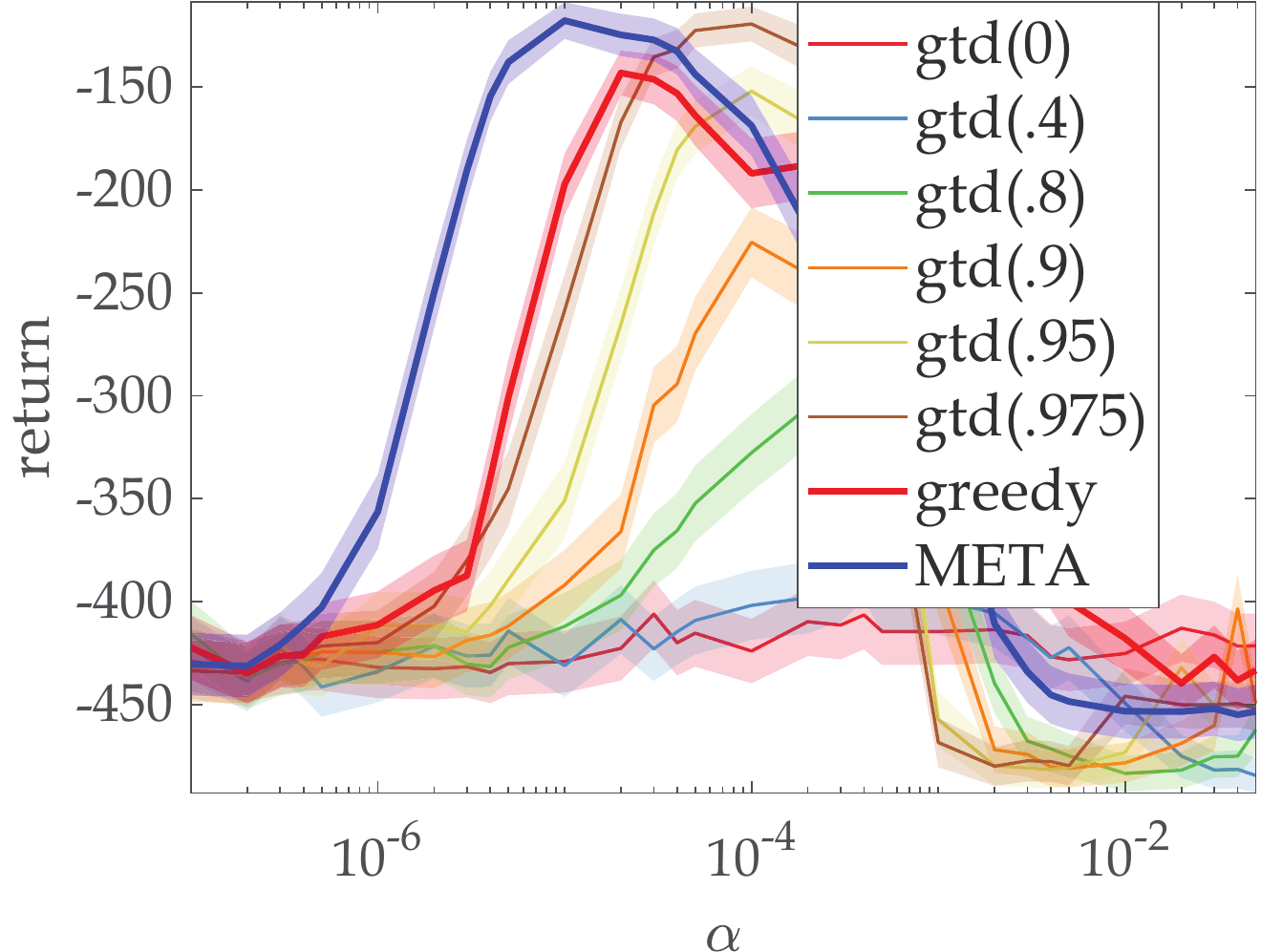}}
\hfill
\subfloat[MountainCar, $\alpha = \beta = 10^{-5}$, $\eta =1$, $\kappa = 10^{-5}$]{
\captionsetup{justification = centering}
\includegraphics[width=0.48\textwidth]{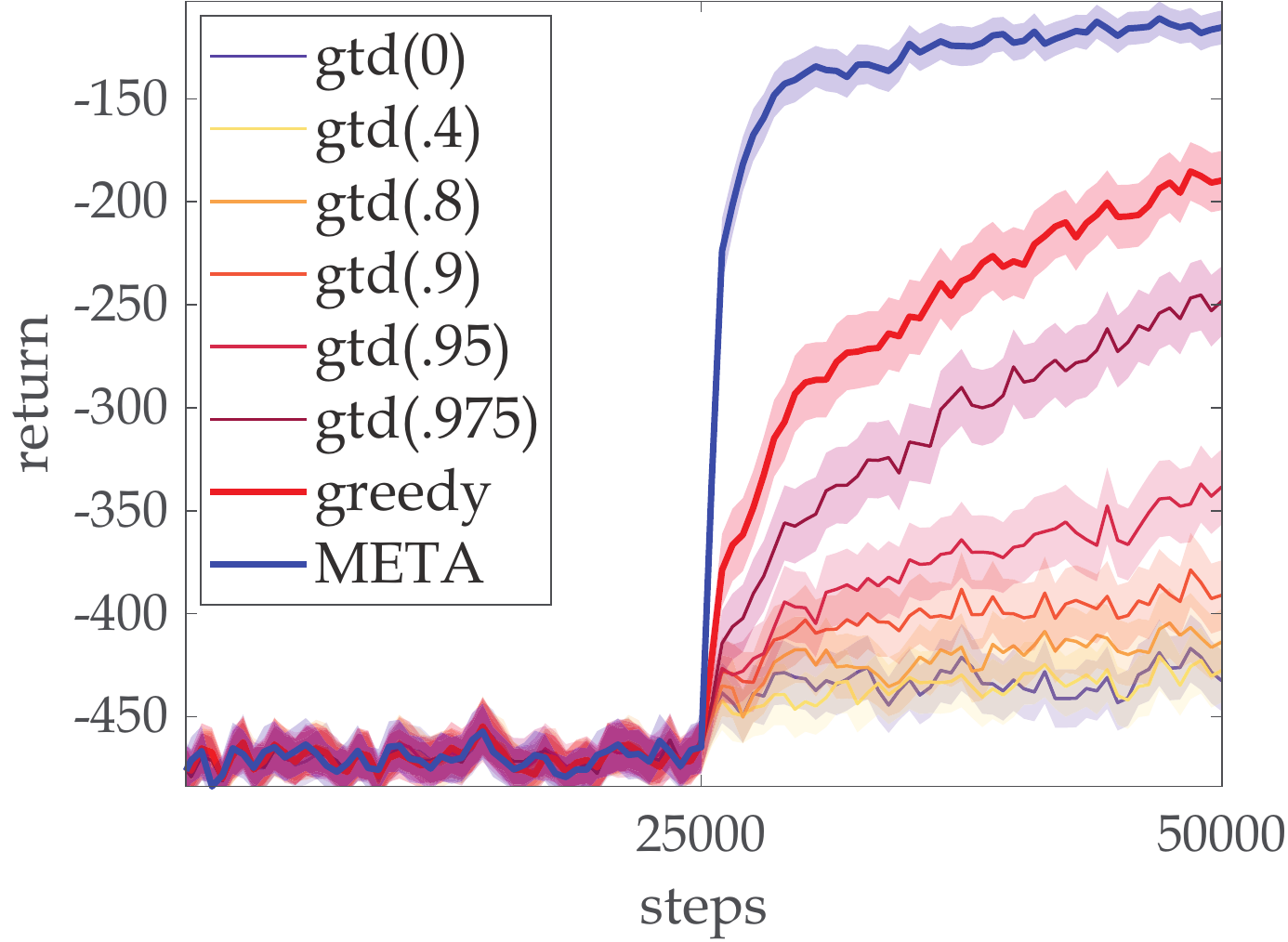}}

\caption{U-shaped curves and learning curves for \algoname{}, $\lambda$-greedy and baselines on MountainCar. For (a), the $x$-axis represents the values of the learning rate $\alpha$ for prediction (the critic), while the $y$-axis represents the discounted return for MountainCar. Each point in the graph contains the mean (solid) and standard deviation (shaded) collected from $240$ independent runs, $50000$ steps for control; For (d), the $x$-axis represents the steps of learning and $y$-axis is the same as (a). In the learning curves, the best known values for the hyperparameters are used. The buffer period length is $25000$ steps ($50\%$). The buffer period and the adaptation period have been ticked on the $x$-axes. Note that in the buffer period of control, we also freeze the policy.}
\label{fig:mountaincar}
\end{figure*}

In this set of experiments we investigate the case in which \algoname{} assists on-policy actor-critic control on a modified version of the environment MountainCar with noise added to the state transitions. The state features used in these experiments resemble the classical setting in \cite{sutton2018reinforcement}, with the tile encoding of $8$ tilings and $8$ offsets. We use a softmax policy parameterized by a $|A| \times D$ matrix, where $D$ is the dimension of the state features, with true online GTD($\lambda$) as the learners (critics). This time, the U-shaped curves presented in Fig. \ref{fig:mountaincar} show performance better than the baselines, and significantly better than $\lambda$-greedy assisted actor-critic.

In this set of experiments we intentionally set the step size of the gradient ascent of the policy to be high ($\eta = 1$), in order to emphasise the quality of policy evaluation. However, typically in actor-critic we keep $\eta$ small. If $\eta$ is small, \algoname{} is expected to help a lot less, as the accuracy of the value function would be less important. Enhancing the policy evaluation quality may not be sufficient for increasing the sample efficiency of control problems.

From the curves presented in Figure \ref{fig:mountaincar}, the most significant improvements are obtained when the learning rate of the critic is small. Typically in actor-critic, we set the learning rate of the critic to be higher than the actor, in order to improve the quality of the update of the actor. \algoname{} alleviates the requirement for such a setting by boosting the sample efficiency of the critic.

%% file: tab_ringworld_behavior_04_target_035.tex
\begin{table}[htbp]
\tiny
\setlength{\tabcolsep}{2pt}
  \centering
  \caption{Detailed Results on RingWorld (Target: 0.35, Behavior: 0.4)}
    \begin{tabular}{ccccccccccccccccccccc}
    \toprule
    \toprule
    baseline & \multicolumn{19}{c}{True Online TD}                                                                                                                   &  \\
    \multirow{2}[1]{*}{$\alpha$\textbackslash{}$\lambda$} & \multicolumn{2}{c}{0} & \multicolumn{2}{c}{0.4} & \multicolumn{2}{c}{0.8} & \multicolumn{2}{c}{0.9} & \multicolumn{2}{c}{0.95} & \multicolumn{2}{c}{0.975} & \multicolumn{2}{c}{0.99} & \multicolumn{2}{c}{1} & \multicolumn{2}{c}{greedy} & \multicolumn{2}{c}{\textit{META}} \\
          & mean  & std   & mean  & std   & mean  & std   & mean  & std   & mean  & std   & mean  & std   & mean  & std   & mean  & std   & mean  & std   & \textit{mean} & \textit{std} \\
    \midrule
    1.0e-5 & \cellcolor[rgb]{ .973,  .412,  .42}2.91e-1 & 2.38e-4 & \cellcolor[rgb]{ .976,  .518,  .525}2.39e-1 & 3.72e-4 & \cellcolor[rgb]{ .984,  .745,  .757}1.24e-1 & 5.1e-4 & \cellcolor[rgb]{ .984,  .835,  .847}7.86e-2 & 4.68e-4 & \cellcolor[rgb]{ .988,  .882,  .894}5.53e-2 & 4.06e-4 & \cellcolor[rgb]{ .988,  .902,  .914}4.43e-2 & 3.61e-4 & \cellcolor[rgb]{ .988,  .914,  .925}3.83e-2 & 3.28e-4 & \cellcolor[rgb]{ .988,  .922,  .933}\textbf{3.46e-2} & 3.04e-4 & \cellcolor[rgb]{ .98,  .573,  .584}2.1e-1 & 6.06e-4 & \cellcolor[rgb]{ .988,  .922,  .933}\textit{3.56e-2} & \textit{3.07e-4} \\
    2.0e-5 & \cellcolor[rgb]{ .976,  .549,  .557}2.23e-1 & 4.26e-4 & \cellcolor[rgb]{ .98,  .694,  .706}1.5e-1 & 5.23e-4 & \cellcolor[rgb]{ .988,  .906,  .918}4.23e-2 & 4.01e-4 & \cellcolor[rgb]{ .988,  .953,  .965}1.86e-2 & 2.7e-4 & \cellcolor[rgb]{ .988,  .969,  .98}1.05e-2 & 1.88e-4 & \cellcolor[rgb]{ .988,  .976,  .988}7.71e-3 & 1.46e-4 & \cellcolor[rgb]{ .988,  .976,  .988}6.49e-3 & 1.21e-4 & \cellcolor[rgb]{ .988,  .98,  .992}\textbf{5.85e-3} & 1.04e-4 & \cellcolor[rgb]{ .984,  .769,  .78}1.11e-1 & 6.53e-4 & \cellcolor[rgb]{ .988,  .98,  .992}\textit{6.06e-3} & \textit{1.08e-4} \\
    3.0e-5 & \cellcolor[rgb]{ .98,  .667,  .675}1.64e-1 & 5.34e-4 & \cellcolor[rgb]{ .984,  .812,  .824}9.04e-2 & 5.27e-4 & \cellcolor[rgb]{ .988,  .961,  .973}1.47e-2 & 2.52e-4 & \cellcolor[rgb]{ .988,  .98,  .992}4.83e-3 & 1.38e-4 & \cellcolor[rgb]{ .988,  .984,  .996}2.44e-3 & 8.4e-5 & \cellcolor[rgb]{ .988,  .988,  1}1.8e-3 & 5.94e-5 & \cellcolor[rgb]{ .988,  .988,  1}1.55e-3 & 4.6e-5 & \cellcolor[rgb]{ .988,  .988,  1}\textbf{1.44e-3} & 3.79e-5 & \cellcolor[rgb]{ .988,  .871,  .882}5.98e-2 & 5.36e-4 & \cellcolor[rgb]{ .988,  .988,  1}\textit{1.48e-3} & \textit{3.98e-5} \\
    4.0e-5 & \cellcolor[rgb]{ .984,  .757,  .769}1.18e-1 & 5.64e-4 & \cellcolor[rgb]{ .988,  .886,  .894}5.35e-2 & 4.58e-4 & \cellcolor[rgb]{ .988,  .98,  .992}5.25e-3 & 1.52e-4 & \cellcolor[rgb]{ .988,  .988,  1}1.41e-3 & 7.25e-5 & \cellcolor[rgb]{ .988,  .988,  1}7.22e-4 & 3.98e-5 & \cellcolor[rgb]{ .988,  .988,  1}5.71e-4 & 2.65e-5 & \cellcolor[rgb]{ .988,  .988,  1}5.21e-4 & 2.0e-5 & \cellcolor[rgb]{ .988,  .988,  1}\textbf{4.99e-4} & 1.66e-5 & \cellcolor[rgb]{ .988,  .925,  .937}3.33e-2 & 4.12e-4 & \cellcolor[rgb]{ .988,  .988,  1}\textit{5.06e-4} & \textit{1.73e-5} \\
    5.0e-5 & \cellcolor[rgb]{ .984,  .827,  .835}8.29e-2 & 5.39e-4 & \cellcolor[rgb]{ .988,  .929,  .941}3.12e-2 & 3.69e-4 & \cellcolor[rgb]{ .988,  .988,  1}1.95e-3 & 9.28e-5 & \cellcolor[rgb]{ .988,  .988,  1}4.91e-4 & 3.93e-5 & \cellcolor[rgb]{ .827,  .875,  .941}2.93e-4 & 2.0e-5 & \cellcolor[rgb]{ .769,  .835,  .922}2.59e-4 & 1.36e-5 & \cellcolor[rgb]{ .753,  .824,  .918}2.5e-4 & 1.12e-5 & \cellcolor[rgb]{ .749,  .82,  .914}\textbf{2.47e-4} & 1.03e-5 & \cellcolor[rgb]{ .988,  .953,  .965}1.92e-2 & 3.12e-4 & \cellcolor[rgb]{ .753,  .82,  .914}\textit{2.48e-4} & \textit{1.05e-5} \\
    1.0e-4 & \cellcolor[rgb]{ .988,  .965,  .976}1.3e-2 & 2.42e-4 & \cellcolor[rgb]{ .988,  .988,  1}2.16e-3 & 9.66e-5 & \cellcolor[rgb]{ .416,  .584,  .796}4.68e-5 & 9.39e-6 & \cellcolor[rgb]{ .412,  .58,  .796}\textbf{4.3e-5} & 6.39e-6 & \cellcolor[rgb]{ .42,  .588,  .8}4.78e-5 & 7.85e-6 & \cellcolor[rgb]{ .424,  .592,  .8}5.07e-5 & 8.83e-6 & \cellcolor[rgb]{ .427,  .592,  .8}5.26e-5 & 9.58e-6 & \cellcolor[rgb]{ .427,  .592,  .8}5.4e-5 & 1.02e-5 & \cellcolor[rgb]{ .988,  .988,  1}1.33e-3 & 7.25e-5 & \cellcolor[rgb]{ .424,  .588,  .8}\textit{5.01e-5} & \textit{7.71e-6} \\
    2.0e-4 & \cellcolor[rgb]{ .988,  .988,  1}4.48e-4 & 4.44e-5 & \cellcolor[rgb]{ .38,  .561,  .784}2.56e-5 & 1.18e-5 & \cellcolor[rgb]{ .357,  .541,  .776}\textbf{1.04e-5} & 6.65e-6 & \cellcolor[rgb]{ .365,  .549,  .78}1.43e-5 & 9.43e-6 & \cellcolor[rgb]{ .369,  .553,  .78}1.78e-5 & 1.21e-5 & \cellcolor[rgb]{ .373,  .557,  .784}2.05e-5 & 1.42e-5 & \cellcolor[rgb]{ .376,  .557,  .784}2.26e-5 & 1.59e-5 & \cellcolor[rgb]{ .38,  .561,  .784}2.44e-5 & 1.73e-5 & \cellcolor[rgb]{ .369,  .553,  .78}1.71e-5 & 5.96e-6 & \cellcolor[rgb]{ .361,  .545,  .776}\textit{1.31e-5} & \textit{7.5e-6} \\
    3.0e-4 & \cellcolor[rgb]{ .408,  .58,  .796}4.11e-5 & 1.48e-5 & \cellcolor[rgb]{ .353,  .541,  .776}\textbf{7.07e-6} & 5.34e-6 & \cellcolor[rgb]{ .361,  .545,  .776}1.28e-5 & 9.63e-6 & \cellcolor[rgb]{ .369,  .553,  .78}1.82e-5 & 1.39e-5 & \cellcolor[rgb]{ .376,  .557,  .784}2.34e-5 & 1.8e-5 & \cellcolor[rgb]{ .384,  .565,  .788}2.74e-5 & 2.12e-5 & \cellcolor[rgb]{ .388,  .569,  .788}3.07e-5 & 2.38e-5 & \cellcolor[rgb]{ .396,  .569,  .788}3.33e-5 & 2.6e-5 & \cellcolor[rgb]{ .357,  .541,  .776}1.02e-5 & 6.2e-6 & \cellcolor[rgb]{ .357,  .545,  .776}\textit{1.09e-5} & \textit{7.48e-6} \\
    4.0e-4 & \cellcolor[rgb]{ .361,  .545,  .776}1.2e-5 & 7.86e-6 & \cellcolor[rgb]{ .353,  .541,  .776}8.44e-6 & 6.18e-6 & \cellcolor[rgb]{ .369,  .549,  .78}1.68e-5 & 1.28e-5 & \cellcolor[rgb]{ .38,  .561,  .784}2.41e-5 & 1.84e-5 & \cellcolor[rgb]{ .392,  .569,  .788}3.11e-5 & 2.38e-5 & \cellcolor[rgb]{ .4,  .573,  .792}3.66e-5 & 2.82e-5 & \cellcolor[rgb]{ .408,  .58,  .796}4.09e-5 & 3.18e-5 & \cellcolor[rgb]{ .412,  .584,  .796}4.45e-5 & 3.47e-5 & \cellcolor[rgb]{ .353,  .541,  .776}\textbf{7.4e-6} & \textbf{4.9e-6} & \cellcolor[rgb]{ .357,  .545,  .776}\textit{1.12e-5} & \textit{7.43e-6} \\
    5.0e-4 & \cellcolor[rgb]{ .353,  .541,  .776}8.95e-6 & 6.34e-6 & \cellcolor[rgb]{ .357,  .541,  .776}1.04e-5 & 7.55e-6 & \cellcolor[rgb]{ .373,  .557,  .784}2.1e-5 & 1.59e-5 & \cellcolor[rgb]{ .388,  .565,  .788}3.01e-5 & 2.28e-5 & \cellcolor[rgb]{ .404,  .576,  .792}3.89e-5 & 2.98e-5 & \cellcolor[rgb]{ .416,  .584,  .796}4.59e-5 & 3.53e-5 & \cellcolor[rgb]{ .424,  .592,  .8}5.14e-5 & 3.99e-5 & \cellcolor[rgb]{ .431,  .596,  .804}5.59e-5 & 4.37e-5 & \cellcolor[rgb]{ .353,  .541,  .776}\textbf{7.77e-6} & \textbf{4.99e-6} & \cellcolor[rgb]{ .361,  .545,  .776}\textit{1.33e-5} & \textit{8.69e-6} \\
    1.0e-3 & \cellcolor[rgb]{ .365,  .549,  .78}1.5e-5 & 1.01e-5 & \cellcolor[rgb]{ .373,  .553,  .78}2.02e-5 & 1.46e-5 & \cellcolor[rgb]{ .408,  .58,  .796}4.15e-5 & 3.11e-5 & \cellcolor[rgb]{ .439,  .6,  .804}6.02e-5 & 4.55e-5 & \cellcolor[rgb]{ .471,  .624,  .816}7.84e-5 & 6.03e-5 & \cellcolor[rgb]{ .494,  .639,  .824}9.25e-5 & 7.2e-5 & \cellcolor[rgb]{ .51,  .651,  .831}1.04e-4 & 8.14e-5 & \cellcolor[rgb]{ .525,  .663,  .835}1.13e-4 & 8.91e-5 & \cellcolor[rgb]{ .365,  .549,  .78}1.5e-5 & 1.01e-5 & \cellcolor[rgb]{ .357,  .541,  .776}\textit{\textbf{9.7e-6}} & \textit{5.39e-6} \\
    2.0e-3 & \cellcolor[rgb]{ .388,  .565,  .788}2.94e-5 & 1.97e-5 & \cellcolor[rgb]{ .404,  .576,  .792}3.98e-5 & 2.86e-5 & \cellcolor[rgb]{ .478,  .627,  .82}8.28e-5 & 6.22e-5 & \cellcolor[rgb]{ .537,  .671,  .839}1.2e-4 & 9.22e-5 & \cellcolor[rgb]{ .6,  .714,  .863}1.56e-4 & 1.22e-4 & \cellcolor[rgb]{ .647,  .745,  .878}1.84e-4 & 1.44e-4 & \cellcolor[rgb]{ .682,  .773,  .89}2.07e-4 & 1.63e-4 & \cellcolor[rgb]{ .714,  .792,  .902}2.25e-4 & 1.78e-4 & \cellcolor[rgb]{ .388,  .565,  .788}2.94e-5 & 1.97e-5 & \cellcolor[rgb]{ .365,  .549,  .78}\textit{\textbf{1.49e-5}} & \textit{7.36e-6} \\
    3.0e-3 & \cellcolor[rgb]{ .412,  .58,  .796}4.36e-5 & 2.91e-5 & \cellcolor[rgb]{ .439,  .6,  .804}5.94e-5 & 4.24e-5 & \cellcolor[rgb]{ .545,  .675,  .843}1.24e-4 & 9.35e-5 & \cellcolor[rgb]{ .639,  .741,  .875}1.79e-4 & 1.38e-4 & \cellcolor[rgb]{ .725,  .804,  .906}2.33e-4 & 1.81e-4 & \cellcolor[rgb]{ .796,  .851,  .929}2.75e-4 & 2.15e-4 & \cellcolor[rgb]{ .851,  .89,  .949}3.08e-4 & 2.43e-4 & \cellcolor[rgb]{ .894,  .922,  .965}3.35e-4 & 2.65e-4 & \cellcolor[rgb]{ .412,  .58,  .796}4.38e-5 & 2.93e-5 & \cellcolor[rgb]{ .376,  .557,  .784}\textit{\textbf{2.19e-5}} & \textit{1.06e-5} \\
    4.0e-3 & \cellcolor[rgb]{ .435,  .6,  .804}5.78e-5 & 3.84e-5 & \cellcolor[rgb]{ .471,  .624,  .816}7.9e-5 & 5.61e-5 & \cellcolor[rgb]{ .612,  .722,  .867}1.65e-4 & 1.24e-4 & \cellcolor[rgb]{ .733,  .808,  .91}2.38e-4 & 1.83e-4 & \cellcolor[rgb]{ .851,  .89,  .949}3.09e-4 & 2.4e-4 & \cellcolor[rgb]{ .945,  .957,  .984}3.64e-4 & 2.85e-4 & \cellcolor[rgb]{ .988,  .988,  1}4.09e-4 & 3.22e-4 & \cellcolor[rgb]{ .988,  .988,  1}4.45e-4 & 3.52e-4 & \cellcolor[rgb]{ .435,  .6,  .804}5.81e-5 & 3.87e-5 & \cellcolor[rgb]{ .388,  .565,  .788}\textit{\textbf{2.93e-5}} & \textit{1.42e-5} \\
    5.0e-3 & \cellcolor[rgb]{ .459,  .616,  .812}7.22e-5 & 4.77e-5 & \cellcolor[rgb]{ .502,  .647,  .827}9.87e-5 & 6.99e-5 & \cellcolor[rgb]{ .678,  .769,  .89}2.05e-4 & 1.55e-4 & \cellcolor[rgb]{ .831,  .878,  .945}2.96e-4 & 2.27e-4 & \cellcolor[rgb]{ .976,  .98,  .996}3.85e-4 & 2.98e-4 & \cellcolor[rgb]{ .988,  .988,  1}4.54e-4 & 3.54e-4 & \cellcolor[rgb]{ .988,  .988,  1}5.09e-4 & 4.0e-4 & \cellcolor[rgb]{ .988,  .988,  1}5.55e-4 & 4.37e-4 & \cellcolor[rgb]{ .459,  .616,  .812}7.26e-5 & 4.82e-5 & \cellcolor[rgb]{ .4,  .573,  .792}\textit{\textbf{3.67e-5}} & \textit{1.78e-5} \\
    1.0e-2 & \cellcolor[rgb]{ .58,  .702,  .855}1.45e-4 & 9.42e-5 & \cellcolor[rgb]{ .667,  .761,  .886}1.97e-4 & 1.38e-4 & \cellcolor[rgb]{ .988,  .988,  1}4.05e-4 & 3.03e-4 & \cellcolor[rgb]{ .988,  .988,  1}5.86e-4 & 4.43e-4 & \cellcolor[rgb]{ .988,  .988,  1}7.64e-4 & 5.84e-4 & \cellcolor[rgb]{ .988,  .988,  1}9.03e-4 & 6.97e-4 & \cellcolor[rgb]{ .988,  .988,  1}1.02e-3 & 7.88e-4 & \cellcolor[rgb]{ .988,  .988,  1}1.11e-3 & 8.63e-4 & \cellcolor[rgb]{ .584,  .702,  .855}1.47e-4 & 9.57e-5 & \cellcolor[rgb]{ .463,  .62,  .816}\textit{\textbf{7.46e-5}} & \textit{3.62e-5} \\
    2.0e-2 & \cellcolor[rgb]{ .824,  .871,  .941}2.91e-4 & 1.86e-4 & \cellcolor[rgb]{ .988,  .988,  1}3.94e-4 & 2.73e-4 & \cellcolor[rgb]{ .988,  .988,  1}8.09e-4 & 5.93e-4 & \cellcolor[rgb]{ .988,  .988,  1}1.17e-3 & 8.79e-4 & \cellcolor[rgb]{ .988,  .988,  1}1.53e-3 & 1.17e-3 & \cellcolor[rgb]{ .988,  .988,  1}1.81e-3 & 1.41e-3 & \cellcolor[rgb]{ .988,  .988,  1}2.04e-3 & 1.6e-3 & \cellcolor[rgb]{ .988,  .988,  1}2.22e-3 & 1.75e-3 & \cellcolor[rgb]{ .839,  .882,  .945}3.01e-4 & 1.95e-4 & \cellcolor[rgb]{ .596,  .714,  .863}\textit{\textbf{1.55e-4}} & \textit{7.63e-5} \\
    3.0e-2 & \cellcolor[rgb]{ .988,  .988,  1}4.41e-4 & 2.81e-4 & \cellcolor[rgb]{ .988,  .988,  1}5.95e-4 & 4.08e-4 & \cellcolor[rgb]{ .988,  .988,  1}1.22e-3 & 8.87e-4 & \cellcolor[rgb]{ .988,  .988,  1}1.76e-3 & 1.33e-3 & \cellcolor[rgb]{ .988,  .988,  1}2.3e-3 & 1.79e-3 & \cellcolor[rgb]{ .988,  .984,  .996}2.73e-3 & 2.16e-3 & \cellcolor[rgb]{ .988,  .984,  .996}3.07e-3 & 2.45e-3 & \cellcolor[rgb]{ .988,  .984,  .996}3.34e-3 & 2.7e-3 & \cellcolor[rgb]{ .988,  .988,  1}4.68e-4 & 3.07e-4 & \cellcolor[rgb]{ .737,  .812,  .91}\textit{\textbf{2.4e-4}} & \textit{1.2e-4} \\
    4.0e-2 & \cellcolor[rgb]{ .988,  .988,  1}5.95e-4 & 3.8e-4 & \cellcolor[rgb]{ .988,  .988,  1}8.0e-4 & 5.46e-4 & \cellcolor[rgb]{ .988,  .988,  1}1.63e-3 & 1.19e-3 & \cellcolor[rgb]{ .988,  .988,  .996}2.36e-3 & 1.81e-3 & \cellcolor[rgb]{ .988,  .984,  .996}3.09e-3 & 2.44e-3 & \cellcolor[rgb]{ .988,  .984,  .996}3.66e-3 & 2.94e-3 & \cellcolor[rgb]{ .988,  .984,  .996}4.11e-3 & 3.35e-3 & \cellcolor[rgb]{ .988,  .98,  .992}4.48e-3 & 3.68e-3 & \cellcolor[rgb]{ .988,  .988,  1}6.44e-4 & 4.23e-4 & \cellcolor[rgb]{ .882,  .914,  .961}\textit{\textbf{3.28e-4}} & \textit{1.69e-4} \\
    5.0e-2 & \cellcolor[rgb]{ .988,  .988,  1}7.54e-4 & 4.82e-4 & \cellcolor[rgb]{ .988,  .988,  1}1.01e-3 & 6.85e-4 & \cellcolor[rgb]{ .988,  .988,  1}2.04e-3 & 1.51e-3 & \cellcolor[rgb]{ .988,  .984,  .996}2.97e-3 & 2.31e-3 & \cellcolor[rgb]{ .988,  .984,  .996}3.89e-3 & 3.11e-3 & \cellcolor[rgb]{ .988,  .98,  .992}4.6e-3 & 3.75e-3 & \cellcolor[rgb]{ .988,  .98,  .992}5.17e-3 & 4.26e-3 & \cellcolor[rgb]{ .988,  .98,  .992}5.63e-3 & 4.68e-3 & \cellcolor[rgb]{ .988,  .988,  1}8.37e-4 & 5.55e-4 & \cellcolor[rgb]{ .988,  .988,  1}\textit{\textbf{4.21e-4}} & \textit{2.22e-4} \\
    1.0e-1 & \cellcolor[rgb]{ .988,  .988,  1}1.62e-3 & 1.04e-3 & \cellcolor[rgb]{ .988,  .988,  1}2.1e-3 & 1.45e-3 & \cellcolor[rgb]{ .988,  .984,  .996}4.21e-3 & 3.24e-3 & \cellcolor[rgb]{ .988,  .98,  .992}6.09e-3 & 4.89e-3 & \cellcolor[rgb]{ .988,  .976,  .988}7.94e-3 & 6.52e-3 & \cellcolor[rgb]{ .988,  .973,  .984}9.38e-3 & 7.79e-3 & \cellcolor[rgb]{ .988,  .969,  .98}1.05e-2 & 8.81e-3 & \cellcolor[rgb]{ .988,  .969,  .98}1.14e-2 & 9.65e-3 & \cellcolor[rgb]{ .988,  .988,  1}2.08e-3 & 1.45e-3 & \cellcolor[rgb]{ .988,  .988,  1}\textit{\textbf{9.58e-4}} & \textit{5.44e-4} \\
    \bottomrule
    \bottomrule
    \multicolumn{21}{l}{Color indicators are added to locate the extreme values: the bluer the better accuracy, the redder the worse. Also, the best result for each $\alpha$ (each row) is marked in \textbf{bold} font.}          
    \end{tabular}%
  \label{tab:ringworld_behavior_04_target_035}%
\end{table}%

%% file: tab_ringworld_behavior_03_target_025.tex
\begin{table}[htbp]
\tiny
\setlength{\tabcolsep}{2pt}
  \centering
  \caption{Detailed Results on RingWorld (Target: 0.25, Behavior: 0.3)}
    \begin{tabular}{ccccccccccccccccccccc}
    \toprule
    \toprule
    baseline & \multicolumn{19}{c}{True Online TD}                                                                                                                   &  \\
    \multirow{2}[1]{*}{$\alpha$\textbackslash{}$\lambda$} & \multicolumn{2}{c}{0} & \multicolumn{2}{c}{0.4} & \multicolumn{2}{c}{0.8} & \multicolumn{2}{c}{0.9} & \multicolumn{2}{c}{0.95} & \multicolumn{2}{c}{0.975} & \multicolumn{2}{c}{0.99} & \multicolumn{2}{c}{1} & \multicolumn{2}{c}{greedy} & \multicolumn{2}{c}{\textit{META}} \\
          & mean  & std   & mean  & std   & mean  & std   & mean  & std   & mean  & std   & mean  & std   & mean  & std   & mean  & std   & mean  & std   & \textit{mean} & \textit{std} \\
    \midrule
    1.0e-5 & \cellcolor[rgb]{ .973,  .412,  .42}3.72e-1 & 3.14e-4 & \cellcolor[rgb]{ .98,  .561,  .569}2.78e-1 & 4.24e-4 & \cellcolor[rgb]{ .984,  .808,  .82}1.17e-1 & 3.99e-4 & \cellcolor[rgb]{ .988,  .882,  .89}7.06e-2 & 3.08e-4 & \cellcolor[rgb]{ .988,  .914,  .925}4.95e-2 & 2.44e-4 & \cellcolor[rgb]{ .988,  .929,  .941}4.02e-2 & 2.09e-4 & \cellcolor[rgb]{ .988,  .937,  .949}3.51e-2 & 1.87e-4 & \cellcolor[rgb]{ .988,  .941,  .953}\textbf{3.2e-2} & 1.72e-4 & \cellcolor[rgb]{ .98,  .62,  .627}2.39e-1 & 5.17e-4 & \cellcolor[rgb]{ .988,  .937,  .949}\textit{3.31e-2} & \textit{1.75e-4} \\
    2.0e-5 & \cellcolor[rgb]{ .98,  .62,  .627}2.4e-1 & 4.8e-4 & \cellcolor[rgb]{ .984,  .78,  .792}1.34e-1 & 4.54e-4 & \cellcolor[rgb]{ .988,  .945,  .957}2.85e-2 & 1.99e-4 & \cellcolor[rgb]{ .988,  .973,  .984}1.26e-2 & 1.15e-4 & \cellcolor[rgb]{ .988,  .98,  .992}7.57e-3 & 7.88e-5 & \cellcolor[rgb]{ .988,  .98,  .992}5.76e-3 & 6.35e-5 & \cellcolor[rgb]{ .988,  .984,  .996}4.9e-3 & 5.51e-5 & \cellcolor[rgb]{ .988,  .984,  .996}\textbf{4.41e-3} & 5.0e-5 & \cellcolor[rgb]{ .988,  .847,  .859}9.32e-2 & 4.14e-4 & \cellcolor[rgb]{ .988,  .984,  .996}\textit{4.55e-3} & \textit{5.14e-5} \\
    3.0e-5 & \cellcolor[rgb]{ .984,  .765,  .776}1.46e-1 & 4.95e-4 & \cellcolor[rgb]{ .988,  .894,  .902}6.31e-2 & 3.39e-4 & \cellcolor[rgb]{ .988,  .976,  .988}8.73e-3 & 9.03e-5 & \cellcolor[rgb]{ .988,  .984,  .996}3.65e-3 & 4.99e-5 & \cellcolor[rgb]{ .988,  .988,  1}2.25e-3 & 3.62e-5 & \cellcolor[rgb]{ .988,  .988,  1}1.76e-3 & 3.09e-5 & \cellcolor[rgb]{ .988,  .988,  1}1.52e-3 & 2.82e-5 & \cellcolor[rgb]{ .988,  .988,  1}\textbf{1.39e-3} & 2.66e-5 & \cellcolor[rgb]{ .988,  .929,  .941}3.95e-2 & 2.63e-4 & \cellcolor[rgb]{ .988,  .988,  1}\textit{1.4e-3} & \textit{2.69e-5} \\
    4.0e-5 & \cellcolor[rgb]{ .988,  .859,  .867}8.57e-2 & 4.19e-4 & \cellcolor[rgb]{ .988,  .945,  .957}3.02e-2 & 2.19e-4 & \cellcolor[rgb]{ .988,  .984,  .996}3.56e-3 & 4.74e-5 & \cellcolor[rgb]{ .988,  .988,  1}1.62e-3 & 2.92e-5 & \cellcolor[rgb]{ .988,  .988,  1}1.06e-3 & 2.34e-5 & \cellcolor[rgb]{ .988,  .988,  1}8.6e-4 & 2.12e-5 & \cellcolor[rgb]{ .988,  .988,  1}7.56e-4 & 2.0e-5 & \cellcolor[rgb]{ .988,  .988,  1}\textbf{6.93e-4} & 1.93e-5 & \cellcolor[rgb]{ .988,  .961,  .973}1.89e-2 & 1.66e-4 & \cellcolor[rgb]{ .988,  .988,  1}\textit{6.96e-4} & \textit{1.94e-5} \\
    5.0e-5 & \cellcolor[rgb]{ .988,  .914,  .925}5.03e-2 & 3.18e-4 & \cellcolor[rgb]{ .988,  .969,  .976}1.53e-2 & 1.36e-4 & \cellcolor[rgb]{ .988,  .988,  1}1.84e-3 & 3.0e-5 & \cellcolor[rgb]{ .988,  .988,  1}9.09e-4 & 2.1e-5 & \cellcolor[rgb]{ .988,  .988,  1}6.22e-4 & 1.79e-5 & \cellcolor[rgb]{ .988,  .988,  1}5.11e-4 & 1.66e-5 & \cellcolor[rgb]{ .988,  .988,  1}4.53e-4 & 1.59e-5 & \cellcolor[rgb]{ .988,  .988,  1}\textbf{4.17e-4} & 1.54e-5 & \cellcolor[rgb]{ .988,  .976,  .988}9.98e-3 & 1.09e-4 & \cellcolor[rgb]{ .988,  .988,  1}\textit{4.18e-4} & \textit{1.55e-5} \\
    1.0e-4 & \cellcolor[rgb]{ .988,  .98,  .992}5.57e-3 & 6.25e-5 & \cellcolor[rgb]{ .988,  .988,  1}1.59e-3 & 2.46e-5 & \cellcolor[rgb]{ .988,  .988,  1}2.69e-4 & 1.16e-5 & \cellcolor[rgb]{ .988,  .988,  1}1.47e-4 & 9.44e-6 & \cellcolor[rgb]{ .851,  .89,  .949}1.05e-4 & 8.38e-6 & \cellcolor[rgb]{ .765,  .831,  .922}8.81e-5 & 7.88e-6 & \cellcolor[rgb]{ .722,  .8,  .906}7.92e-5 & 7.6e-6 & \cellcolor[rgb]{ .694,  .78,  .894}\textbf{7.37e-5} & 7.42e-6 & \cellcolor[rgb]{ .988,  .988,  1}1.02e-3 & 2.38e-5 & \cellcolor[rgb]{ .694,  .78,  .894}\textit{7.38e-5} & \textit{7.43e-6} \\
    2.0e-4 & \cellcolor[rgb]{ .988,  .988,  1}5.71e-4 & 1.47e-5 & \cellcolor[rgb]{ .988,  .988,  1}1.78e-4 & 9.52e-6 & \cellcolor[rgb]{ .467,  .62,  .816}2.75e-5 & 5.14e-6 & \cellcolor[rgb]{ .408,  .576,  .792}1.52e-5 & 4.32e-6 & \cellcolor[rgb]{ .388,  .565,  .788}1.18e-5 & 4.2e-6 & \cellcolor[rgb]{ .384,  .561,  .784}1.06e-5 & 4.29e-6 & \cellcolor[rgb]{ .38,  .561,  .784}1.01e-5 & 4.42e-6 & \cellcolor[rgb]{ .38,  .561,  .784}9.88e-6 & 4.53e-6 & \cellcolor[rgb]{ .827,  .875,  .941}1.0e-4 & 7.62e-6 & \cellcolor[rgb]{ .38,  .557,  .784}\textit{\textbf{9.74e-6}} & \textit{4.24e-6} \\
    3.0e-4 & \cellcolor[rgb]{ .988,  .988,  1}1.51e-4 & 9.27e-6 & \cellcolor[rgb]{ .541,  .675,  .843}4.26e-5 & 6.08e-6 & \cellcolor[rgb]{ .365,  .549,  .78}6.92e-6 & 3.59e-6 & \cellcolor[rgb]{ .357,  .545,  .776}5.6e-6 & 3.94e-6 & \cellcolor[rgb]{ .361,  .545,  .776}5.89e-6 & 4.59e-6 & \cellcolor[rgb]{ .361,  .545,  .776}6.31e-6 & 5.1e-6 & \cellcolor[rgb]{ .365,  .549,  .78}6.69e-6 & 5.48e-6 & \cellcolor[rgb]{ .365,  .549,  .78}6.99e-6 & 5.78e-6 & \cellcolor[rgb]{ .443,  .604,  .808}2.27e-5 & 4.33e-6 & \cellcolor[rgb]{ .357,  .541,  .776}\textit{\textbf{5.15e-6}} & \textit{3.58e-6} \\
    4.0e-4 & \cellcolor[rgb]{ .616,  .725,  .867}5.8e-5 & 6.95e-6 & \cellcolor[rgb]{ .408,  .58,  .796}1.57e-5 & 4.47e-6 & \cellcolor[rgb]{ .357,  .541,  .776}4.97e-6 & 3.77e-6 & \cellcolor[rgb]{ .361,  .545,  .776}5.76e-6 & 4.87e-6 & \cellcolor[rgb]{ .365,  .549,  .78}6.84e-6 & 5.94e-6 & \cellcolor[rgb]{ .369,  .553,  .78}7.66e-6 & 6.68e-6 & \cellcolor[rgb]{ .373,  .553,  .78}8.27e-6 & 7.23e-6 & \cellcolor[rgb]{ .373,  .557,  .784}8.75e-6 & 7.64e-6 & \cellcolor[rgb]{ .373,  .553,  .78}8.13e-6 & 2.98e-6 & \cellcolor[rgb]{ .353,  .541,  .776}\textit{\textbf{4.64e-6}} & \textit{3.65e-6} \\
    5.0e-4 & \cellcolor[rgb]{ .471,  .624,  .816}2.84e-5 & 5.62e-6 & \cellcolor[rgb]{ .373,  .553,  .78}8.45e-6 & 3.74e-6 & \cellcolor[rgb]{ .357,  .545,  .776}5.41e-6 & 4.48e-6 & \cellcolor[rgb]{ .365,  .549,  .78}6.97e-6 & 6.05e-6 & \cellcolor[rgb]{ .373,  .553,  .78}8.47e-6 & 7.42e-6 & \cellcolor[rgb]{ .376,  .557,  .784}9.54e-6 & 8.37e-6 & \cellcolor[rgb]{ .38,  .561,  .784}1.03e-5 & 9.06e-6 & \cellcolor[rgb]{ .384,  .565,  .788}1.09e-5 & 9.58e-6 & \cellcolor[rgb]{ .353,  .541,  .776}\textbf{4.52e-6} & 2.42e-6 & \cellcolor[rgb]{ .357,  .545,  .776}\textit{5.5e-6} & \textit{4.38e-6} \\
    1.0e-3 & \cellcolor[rgb]{ .361,  .545,  .776}5.93e-6 & 3.68e-6 & \cellcolor[rgb]{ .361,  .545,  .776}5.78e-6 & 4.43e-6 & \cellcolor[rgb]{ .38,  .561,  .784}1.03e-5 & 8.91e-6 & \cellcolor[rgb]{ .4,  .573,  .792}1.39e-5 & 1.21e-5 & \cellcolor[rgb]{ .416,  .584,  .796}1.71e-5 & 1.48e-5 & \cellcolor[rgb]{ .427,  .592,  .8}1.92e-5 & 1.67e-5 & \cellcolor[rgb]{ .435,  .596,  .804}2.08e-5 & 1.8e-5 & \cellcolor[rgb]{ .439,  .6,  .804}2.21e-5 & 1.9e-5 & \cellcolor[rgb]{ .353,  .541,  .776}4.35e-6 & 3.0e-6 & \cellcolor[rgb]{ .353,  .541,  .776}\textit{\textbf{4.11e-6}} & \textit{2.65e-6} \\
    2.0e-3 & \cellcolor[rgb]{ .373,  .553,  .78}8.61e-6 & 6.08e-6 & \cellcolor[rgb]{ .384,  .565,  .788}1.11e-5 & 8.87e-6 & \cellcolor[rgb]{ .431,  .596,  .804}2.07e-5 & 1.76e-5 & \cellcolor[rgb]{ .471,  .624,  .816}2.8e-5 & 2.37e-5 & \cellcolor[rgb]{ .498,  .643,  .827}3.41e-5 & 2.88e-5 & \cellcolor[rgb]{ .522,  .659,  .835}3.84e-5 & 3.24e-5 & \cellcolor[rgb]{ .537,  .671,  .839}4.15e-5 & 3.51e-5 & \cellcolor[rgb]{ .549,  .678,  .843}4.39e-5 & 3.71e-5 & \cellcolor[rgb]{ .373,  .553,  .78}8.6e-6 & 6.08e-6 & \cellcolor[rgb]{ .357,  .541,  .776}\textit{\textbf{5.18e-6}} & \textit{2.81e-6} \\
    3.0e-3 & \cellcolor[rgb]{ .396,  .569,  .788}1.29e-5 & 9.25e-6 & \cellcolor[rgb]{ .412,  .584,  .796}1.68e-5 & 1.33e-5 & \cellcolor[rgb]{ .482,  .631,  .82}3.1e-5 & 2.57e-5 & \cellcolor[rgb]{ .537,  .671,  .839}4.17e-5 & 3.46e-5 & \cellcolor[rgb]{ .58,  .702,  .855}5.08e-5 & 4.22e-5 & \cellcolor[rgb]{ .612,  .725,  .867}5.72e-5 & 4.75e-5 & \cellcolor[rgb]{ .635,  .741,  .875}6.18e-5 & 5.14e-5 & \cellcolor[rgb]{ .655,  .753,  .882}6.54e-5 & 5.44e-5 & \cellcolor[rgb]{ .396,  .569,  .788}1.29e-5 & 9.29e-6 & \cellcolor[rgb]{ .369,  .549,  .78}\textit{\textbf{7.49e-6}} & \textit{3.93e-6} \\
    4.0e-3 & \cellcolor[rgb]{ .416,  .584,  .796}1.72e-5 & 1.24e-5 & \cellcolor[rgb]{ .439,  .604,  .808}2.24e-5 & 1.75e-5 & \cellcolor[rgb]{ .533,  .667,  .839}4.12e-5 & 3.36e-5 & \cellcolor[rgb]{ .604,  .718,  .863}5.53e-5 & 4.52e-5 & \cellcolor[rgb]{ .663,  .761,  .886}6.74e-5 & 5.53e-5 & \cellcolor[rgb]{ .706,  .788,  .898}7.57e-5 & 6.23e-5 & \cellcolor[rgb]{ .737,  .808,  .91}8.19e-5 & 6.75e-5 & \cellcolor[rgb]{ .757,  .827,  .918}8.65e-5 & 7.14e-5 & \cellcolor[rgb]{ .416,  .584,  .796}1.72e-5 & 1.24e-5 & \cellcolor[rgb]{ .38,  .561,  .784}\textit{\textbf{9.84e-6}} & \textit{5.07e-6} \\
    5.0e-3 & \cellcolor[rgb]{ .435,  .6,  .804}2.15e-5 & 1.54e-5 & \cellcolor[rgb]{ .471,  .624,  .816}2.79e-5 & 2.17e-5 & \cellcolor[rgb]{ .584,  .702,  .855}5.13e-5 & 4.12e-5 & \cellcolor[rgb]{ .671,  .765,  .886}6.88e-5 & 5.57e-5 & \cellcolor[rgb]{ .745,  .816,  .914}8.38e-5 & 6.82e-5 & \cellcolor[rgb]{ .796,  .851,  .929}9.42e-5 & 7.69e-5 & \cellcolor[rgb]{ .835,  .878,  .945}1.02e-4 & 8.33e-5 & \cellcolor[rgb]{ .863,  .898,  .953}1.08e-4 & 8.81e-5 & \cellcolor[rgb]{ .435,  .6,  .804}2.16e-5 & 1.55e-5 & \cellcolor[rgb]{ .392,  .569,  .788}\textit{\textbf{1.23e-5}} & \textit{6.28e-6} \\
    1.0e-2 & \cellcolor[rgb]{ .545,  .675,  .843}4.31e-5 & 2.98e-5 & \cellcolor[rgb]{ .604,  .718,  .863}5.56e-5 & 4.08e-5 & \cellcolor[rgb]{ .831,  .878,  .945}1.01e-4 & 7.9e-5 & \cellcolor[rgb]{ .988,  .988,  1}1.36e-4 & 1.08e-4 & \cellcolor[rgb]{ .988,  .988,  1}1.65e-4 & 1.33e-4 & \cellcolor[rgb]{ .988,  .988,  1}1.85e-4 & 1.5e-4 & \cellcolor[rgb]{ .988,  .988,  1}1.99e-4 & 1.62e-4 & \cellcolor[rgb]{ .988,  .988,  1}2.1e-4 & 1.72e-4 & \cellcolor[rgb]{ .545,  .675,  .843}4.34e-5 & 3.01e-5 & \cellcolor[rgb]{ .455,  .612,  .812}\textit{\textbf{2.47e-5}} & \textit{1.26e-5} \\
    2.0e-2 & \cellcolor[rgb]{ .757,  .824,  .918}8.63e-5 & 5.68e-5 & \cellcolor[rgb]{ .878,  .91,  .961}1.11e-4 & 7.88e-5 & \cellcolor[rgb]{ .988,  .988,  1}2.01e-4 & 1.59e-4 & \cellcolor[rgb]{ .988,  .988,  1}2.67e-4 & 2.18e-4 & \cellcolor[rgb]{ .988,  .988,  1}3.23e-4 & 2.66e-4 & \cellcolor[rgb]{ .988,  .988,  1}3.61e-4 & 2.99e-4 & \cellcolor[rgb]{ .988,  .988,  1}3.89e-4 & 3.23e-4 & \cellcolor[rgb]{ .988,  .988,  1}4.1e-4 & 3.41e-4 & \cellcolor[rgb]{ .769,  .831,  .922}8.82e-5 & 5.87e-5 & \cellcolor[rgb]{ .576,  .698,  .855}\textit{\textbf{4.96e-5}} & \textit{2.59e-5} \\
    3.0e-2 & \cellcolor[rgb]{ .973,  .976,  .992}1.3e-4 & 8.5e-5 & \cellcolor[rgb]{ .988,  .988,  1}1.66e-4 & 1.2e-4 & \cellcolor[rgb]{ .988,  .988,  1}2.99e-4 & 2.43e-4 & \cellcolor[rgb]{ .988,  .988,  1}3.96e-4 & 3.28e-4 & \cellcolor[rgb]{ .988,  .988,  1}4.77e-4 & 3.98e-4 & \cellcolor[rgb]{ .988,  .988,  1}5.33e-4 & 4.45e-4 & \cellcolor[rgb]{ .988,  .988,  1}5.73e-4 & 4.8e-4 & \cellcolor[rgb]{ .988,  .988,  1}6.03e-4 & 5.07e-4 & \cellcolor[rgb]{ .988,  .988,  1}1.35e-4 & 9.14e-5 & \cellcolor[rgb]{ .702,  .784,  .898}\textit{\textbf{7.47e-5}} & \textit{3.92e-5} \\
    4.0e-2 & \cellcolor[rgb]{ .988,  .988,  1}1.74e-4 & 1.15e-4 & \cellcolor[rgb]{ .988,  .988,  1}2.21e-4 & 1.64e-4 & \cellcolor[rgb]{ .988,  .988,  1}3.95e-4 & 3.25e-4 & \cellcolor[rgb]{ .988,  .988,  1}5.22e-4 & 4.35e-4 & \cellcolor[rgb]{ .988,  .988,  1}6.28e-4 & 5.25e-4 & \cellcolor[rgb]{ .988,  .988,  1}7.01e-4 & 5.88e-4 & \cellcolor[rgb]{ .988,  .988,  1}7.53e-4 & 6.34e-4 & \cellcolor[rgb]{ .988,  .988,  1}7.92e-4 & 6.69e-4 & \cellcolor[rgb]{ .988,  .988,  1}1.84e-4 & 1.27e-4 & \cellcolor[rgb]{ .831,  .875,  .941}\textit{\textbf{1.01e-4}} & \textit{5.27e-5} \\
    5.0e-2 & \cellcolor[rgb]{ .988,  .988,  1}2.19e-4 & 1.47e-4 & \cellcolor[rgb]{ .988,  .988,  1}2.76e-4 & 2.09e-4 & \cellcolor[rgb]{ .988,  .988,  1}4.91e-4 & 4.06e-4 & \cellcolor[rgb]{ .988,  .988,  1}6.47e-4 & 5.4e-4 & \cellcolor[rgb]{ .988,  .988,  1}7.77e-4 & 6.52e-4 & \cellcolor[rgb]{ .988,  .988,  1}8.66e-4 & 7.3e-4 & \cellcolor[rgb]{ .988,  .988,  1}9.3e-4 & 7.87e-4 & \cellcolor[rgb]{ .988,  .988,  1}9.78e-4 & 8.32e-4 & \cellcolor[rgb]{ .988,  .988,  1}2.32e-4 & 1.64e-4 & \cellcolor[rgb]{ .965,  .973,  .992}\textit{\textbf{1.28e-4}} & \textit{6.74e-5} \\
    1.0e-1 & \cellcolor[rgb]{ .988,  .988,  1}4.51e-4 & 3.18e-4 & \cellcolor[rgb]{ .988,  .988,  1}5.55e-4 & 4.31e-4 & \cellcolor[rgb]{ .988,  .988,  1}9.62e-4 & 7.98e-4 & \cellcolor[rgb]{ .988,  .988,  1}1.26e-3 & 1.06e-3 & \cellcolor[rgb]{ .988,  .988,  1}1.51e-3 & 1.29e-3 & \cellcolor[rgb]{ .988,  .988,  1}1.67e-3 & 1.47e-3 & \cellcolor[rgb]{ .988,  .988,  1}1.8e-3 & 1.6e-3 & \cellcolor[rgb]{ .988,  .988,  1}1.89e-3 & 1.7e-3 & \cellcolor[rgb]{ .988,  .988,  1}4.89e-4 & 3.65e-4 & \cellcolor[rgb]{ .988,  .988,  1}\textit{\textbf{2.77e-4}} & \textit{1.49e-4} \\
    \bottomrule
    \bottomrule
    \multicolumn{21}{l}{Color indicators are added to locate the extreme values: the bluer the better accuracy, the redder the worse. Also, the best result for each $\alpha$ (each row) is marked in \textbf{bold} font.}          
    \end{tabular}%
  \label{tab:ringworld_behavior_03_target_025}%
\end{table}%

%% file: tab_ringworld_behavior_02_target_015.tex
\begin{table}[htbp]
\tiny
\setlength{\tabcolsep}{2pt}
  \centering
  \caption{Detailed Results on RingWorld (Target: 0.15, Behavior: 0.2)}
    \begin{tabular}{ccccccccccccccccccccc}
    \toprule
    \toprule
    baseline & \multicolumn{19}{c}{True Online TD}                                                                                                                   &  \\
    \multirow{2}[1]{*}{$\alpha$\textbackslash{}$\lambda$} & \multicolumn{2}{c}{0} & \multicolumn{2}{c}{0.4} & \multicolumn{2}{c}{0.8} & \multicolumn{2}{c}{0.9} & \multicolumn{2}{c}{0.95} & \multicolumn{2}{c}{0.975} & \multicolumn{2}{c}{0.99} & \multicolumn{2}{c}{1} & \multicolumn{2}{c}{greedy} & \multicolumn{2}{c}{\textit{META}} \\
          & mean  & std   & mean  & std   & mean  & std   & mean  & std   & mean  & std   & mean  & std   & mean  & std   & mean  & std   & mean  & std   & \textit{mean} & \textit{std} \\
    \midrule
    1.0e-5 & \cellcolor[rgb]{ .973,  .412,  .42}3.65e-1 & 2.65e-4 & \cellcolor[rgb]{ .98,  .608,  .616}2.43e-1 & 3.15e-4 & \cellcolor[rgb]{ .988,  .851,  .863}8.72e-2 & 2.12e-4 & \cellcolor[rgb]{ .988,  .906,  .918}5.3e-2 & 1.48e-4 & \cellcolor[rgb]{ .988,  .929,  .941}3.89e-2 & 1.14e-4 & \cellcolor[rgb]{ .988,  .937,  .949}3.28e-2 & 9.74e-5 & \cellcolor[rgb]{ .988,  .945,  .957}2.94e-2 & 8.77e-5 & \cellcolor[rgb]{ .988,  .949,  .957}2.74e-2 & 8.15e-5 & \cellcolor[rgb]{ .98,  .663,  .671}2.09e-1 & 3.22e-4 & \cellcolor[rgb]{ .988,  .945,  .957}\textit{2.83e-2} & \textit{8.32e-5} \\
    2.0e-5 & \cellcolor[rgb]{ .98,  .694,  .706}1.87e-1 & 3.37e-4 & \cellcolor[rgb]{ .988,  .859,  .867}8.42e-2 & 2.41e-4 & \cellcolor[rgb]{ .988,  .965,  .976}1.6e-2 & 6.53e-5 & \cellcolor[rgb]{ .988,  .976,  .988}8.52e-3 & 3.93e-5 & \cellcolor[rgb]{ .988,  .98,  .992}6.17e-3 & 3.27e-5 & \cellcolor[rgb]{ .988,  .98,  .992}5.29e-3 & 3.07e-5 & \cellcolor[rgb]{ .988,  .984,  .996}4.85e-3 & 2.99e-5 & \cellcolor[rgb]{ .988,  .984,  .996}4.58e-3 & 2.94e-5 & \cellcolor[rgb]{ .988,  .902,  .91}5.72e-2 & 1.81e-4 & \cellcolor[rgb]{ .988,  .984,  .996}\textit{4.65e-3} & \textit{2.95e-5} \\
    3.0e-5 & \cellcolor[rgb]{ .988,  .855,  .867}8.51e-2 & 2.67e-4 & \cellcolor[rgb]{ .988,  .945,  .957}2.92e-2 & 1.21e-4 & \cellcolor[rgb]{ .988,  .98,  .992}5.3e-3 & 2.76e-5 & \cellcolor[rgb]{ .988,  .984,  .996}3.35e-3 & 2.33e-5 & \cellcolor[rgb]{ .988,  .984,  .996}2.71e-3 & 2.22e-5 & \cellcolor[rgb]{ .988,  .988,  1}2.45e-3 & 2.17e-5 & \cellcolor[rgb]{ .988,  .988,  1}2.3e-3 & 2.14e-5 & \cellcolor[rgb]{ .988,  .988,  1}2.22e-3 & 2.12e-5 & \cellcolor[rgb]{ .988,  .961,  .973}1.92e-2 & 8.72e-5 & \cellcolor[rgb]{ .988,  .988,  1}\textit{2.22e-3} & \textit{2.12e-5} \\
    4.0e-5 & \cellcolor[rgb]{ .988,  .929,  .941}3.8e-2 & 1.63e-4 & \cellcolor[rgb]{ .988,  .973,  .984}1.18e-2 & 5.48e-5 & \cellcolor[rgb]{ .988,  .984,  .996}2.8e-3 & 1.92e-5 & \cellcolor[rgb]{ .988,  .988,  1}1.99e-3 & 1.75e-5 & \cellcolor[rgb]{ .988,  .988,  1}1.68e-3 & 1.67e-5 & \cellcolor[rgb]{ .988,  .988,  1}1.55e-3 & 1.64e-5 & \cellcolor[rgb]{ .988,  .988,  1}1.47e-3 & 1.62e-5 & \cellcolor[rgb]{ .988,  .988,  1}1.43e-3 & 1.61e-5 & \cellcolor[rgb]{ .988,  .976,  .988}8.33e-3 & 4.64e-5 & \cellcolor[rgb]{ .988,  .988,  1}\textit{1.43e-3} & \textit{1.61e-5} \\
    5.0e-5 & \cellcolor[rgb]{ .988,  .961,  .973}1.82e-2 & 9.0e-5 & \cellcolor[rgb]{ .988,  .98,  .992}5.94e-3 & 2.87e-5 & \cellcolor[rgb]{ .988,  .988,  1}1.84e-3 & 1.51e-5 & \cellcolor[rgb]{ .988,  .988,  1}1.38e-3 & 1.41e-5 & \cellcolor[rgb]{ .988,  .988,  1}1.19e-3 & 1.36e-5 & \cellcolor[rgb]{ .988,  .988,  1}1.11e-3 & 1.35e-5 & \cellcolor[rgb]{ .988,  .988,  1}1.06e-3 & 1.34e-5 & \cellcolor[rgb]{ .988,  .988,  1}1.03e-3 & 1.33e-5 & \cellcolor[rgb]{ .988,  .984,  .996}4.43e-3 & 2.82e-5 & \cellcolor[rgb]{ .988,  .988,  1}\textit{1.03e-3} & \textit{1.33e-5} \\
    1.0e-4 & \cellcolor[rgb]{ .988,  .988,  1}2.35e-3 & 1.33e-5 & \cellcolor[rgb]{ .988,  .988,  1}1.25e-3 & 1.03e-5 & \cellcolor[rgb]{ .988,  .988,  1}6.05e-4 & 9.11e-6 & \cellcolor[rgb]{ .988,  .988,  1}4.88e-4 & 8.72e-6 & \cellcolor[rgb]{ .988,  .988,  1}4.36e-4 & 8.54e-6 & \cellcolor[rgb]{ .988,  .988,  1}4.11e-4 & 8.46e-6 & \cellcolor[rgb]{ .988,  .988,  1}3.96e-4 & 8.42e-6 & \cellcolor[rgb]{ .988,  .988,  1}3.86e-4 & 8.39e-6 & \cellcolor[rgb]{ .988,  .988,  1}9.71e-4 & 1.05e-5 & \cellcolor[rgb]{ .988,  .988,  1}\textit{3.87e-4} & \textit{8.39e-6} \\
    2.0e-4 & \cellcolor[rgb]{ .988,  .988,  1}6.14e-4 & 7.11e-6 & \cellcolor[rgb]{ .988,  .988,  1}3.7e-4 & 5.75e-6 & \cellcolor[rgb]{ .988,  .988,  1}1.91e-4 & 4.66e-6 & \cellcolor[rgb]{ .988,  .988,  1}1.56e-4 & 4.48e-6 & \cellcolor[rgb]{ .988,  .988,  1}1.4e-4 & 4.42e-6 & \cellcolor[rgb]{ .988,  .988,  1}1.32e-4 & 4.4e-6 & \cellcolor[rgb]{ .988,  .988,  1}1.28e-4 & 4.39e-6 & \cellcolor[rgb]{ .988,  .988,  1}1.25e-4 & 4.38e-6 & \cellcolor[rgb]{ .988,  .988,  1}3.09e-4 & 5.63e-6 & \cellcolor[rgb]{ .988,  .988,  1}\textit{1.25e-4} & \textit{4.38e-6} \\
    3.0e-4 & \cellcolor[rgb]{ .988,  .988,  1}2.95e-4 & 4.43e-6 & \cellcolor[rgb]{ .988,  .988,  1}1.82e-4 & 3.8e-6 & \cellcolor[rgb]{ .988,  .988,  1}9.74e-5 & 3.45e-6 & \cellcolor[rgb]{ .988,  .988,  1}8.07e-5 & 3.4e-6 & \cellcolor[rgb]{ .949,  .961,  .984}7.31e-5 & 3.39e-6 & \cellcolor[rgb]{ .922,  .941,  .976}6.95e-5 & 3.39e-6 & \cellcolor[rgb]{ .902,  .925,  .969}6.74e-5 & 3.4e-6 & \cellcolor[rgb]{ .89,  .918,  .965}6.6e-5 & 3.4e-6 & \cellcolor[rgb]{ .988,  .988,  1}1.61e-4 & 3.97e-6 & \cellcolor[rgb]{ .89,  .918,  .965}\textit{6.6e-5} & \textit{3.4e-6} \\
    4.0e-4 & \cellcolor[rgb]{ .988,  .988,  1}1.78e-4 & 3.43e-6 & \cellcolor[rgb]{ .988,  .988,  1}1.11e-4 & 3.19e-6 & \cellcolor[rgb]{ .835,  .878,  .945}5.95e-5 & 2.95e-6 & \cellcolor[rgb]{ .745,  .82,  .914}4.95e-5 & 2.93e-6 & \cellcolor[rgb]{ .706,  .788,  .898}4.49e-5 & 2.94e-6 & \cellcolor[rgb]{ .69,  .776,  .894}4.28e-5 & 2.96e-6 & \cellcolor[rgb]{ .678,  .769,  .89}4.16e-5 & 2.98e-6 & \cellcolor[rgb]{ .671,  .765,  .886}4.07e-5 & 2.99e-6 & \cellcolor[rgb]{ .988,  .988,  1}9.96e-5 & 3.24e-6 & \cellcolor[rgb]{ .671,  .765,  .886}\textit{4.07e-5} & \textit{2.99e-6} \\
    5.0e-4 & \cellcolor[rgb]{ .988,  .988,  1}1.19e-4 & 3.02e-6 & \cellcolor[rgb]{ .957,  .965,  .988}7.37e-5 & 2.82e-6 & \cellcolor[rgb]{ .659,  .757,  .882}3.92e-5 & 2.6e-6 & \cellcolor[rgb]{ .6,  .714,  .863}3.25e-5 & 2.62e-6 & \cellcolor[rgb]{ .576,  .698,  .855}2.96e-5 & 2.68e-6 & \cellcolor[rgb]{ .565,  .69,  .851}2.82e-5 & 2.73e-6 & \cellcolor[rgb]{ .557,  .682,  .847}2.74e-5 & 2.77e-6 & \cellcolor[rgb]{ .553,  .682,  .847}2.69e-5 & 2.8e-6 & \cellcolor[rgb]{ .898,  .925,  .969}6.71e-5 & 2.77e-6 & \cellcolor[rgb]{ .553,  .682,  .847}\textit{2.69e-5} & \textit{2.77e-6} \\
    1.0e-3 & \cellcolor[rgb]{ .576,  .698,  .855}2.99e-5 & 1.93e-6 & \cellcolor[rgb]{ .475,  .627,  .82}1.8e-5 & 1.87e-6 & \cellcolor[rgb]{ .408,  .58,  .796}1.03e-5 & 2.44e-6 & \cellcolor[rgb]{ .4,  .573,  .792}9.37e-6 & 2.91e-6 & \cellcolor[rgb]{ .4,  .573,  .792}9.11e-6 & 3.26e-6 & \cellcolor[rgb]{ .396,  .573,  .792}9.06e-6 & 3.48e-6 & \cellcolor[rgb]{ .396,  .573,  .792}9.06e-6 & 3.64e-6 & \cellcolor[rgb]{ .396,  .573,  .792}9.07e-6 & 3.75e-6 & \cellcolor[rgb]{ .471,  .624,  .816}1.77e-5 & 1.69e-6 & \cellcolor[rgb]{ .392,  .569,  .788}\textit{8.32e-6} & \textit{2.95e-6} \\
    2.0e-3 & \cellcolor[rgb]{ .4,  .573,  .792}9.21e-6 & 2.13e-6 & \cellcolor[rgb]{ .376,  .557,  .784}6.72e-6 & 2.66e-6 & \cellcolor[rgb]{ .376,  .557,  .784}6.59e-6 & 4.43e-6 & \cellcolor[rgb]{ .384,  .561,  .784}7.35e-6 & 5.46e-6 & \cellcolor[rgb]{ .388,  .565,  .788}8.01e-6 & 6.19e-6 & \cellcolor[rgb]{ .392,  .569,  .788}8.45e-6 & 6.64e-6 & \cellcolor[rgb]{ .396,  .569,  .788}8.76e-6 & 6.95e-6 & \cellcolor[rgb]{ .396,  .573,  .792}8.99e-6 & 7.18e-6 & \cellcolor[rgb]{ .373,  .557,  .784}6.24e-6 & 2.04e-6 & \cellcolor[rgb]{ .357,  .545,  .776}\textit{4.37e-6} & \textit{2.35e-6} \\
    3.0e-3 & \cellcolor[rgb]{ .373,  .557,  .784}6.24e-6 & 2.94e-6 & \cellcolor[rgb]{ .369,  .553,  .78}5.77e-6 & 3.87e-6 & \cellcolor[rgb]{ .384,  .565,  .788}7.6e-6 & 6.54e-6 & \cellcolor[rgb]{ .396,  .573,  .792}9.06e-6 & 8.05e-6 & \cellcolor[rgb]{ .408,  .58,  .796}1.02e-5 & 9.14e-6 & \cellcolor[rgb]{ .416,  .584,  .796}1.09e-5 & 9.82e-6 & \cellcolor[rgb]{ .42,  .588,  .8}1.14e-5 & 1.03e-5 & \cellcolor[rgb]{ .42,  .588,  .8}1.18e-5 & 1.06e-5 & \cellcolor[rgb]{ .361,  .549,  .78}4.98e-6 & 2.89e-6 & \cellcolor[rgb]{ .353,  .541,  .776}\textit{3.82e-6} & \textit{1.86e-6} \\
    4.0e-3 & \cellcolor[rgb]{ .373,  .553,  .78}5.92e-6 & 3.84e-6 & \cellcolor[rgb]{ .376,  .557,  .784}6.37e-6 & 5.11e-6 & \cellcolor[rgb]{ .4,  .573,  .792}9.41e-6 & 8.61e-6 & \cellcolor[rgb]{ .42,  .588,  .8}1.14e-5 & 1.06e-5 & \cellcolor[rgb]{ .431,  .596,  .804}1.3e-5 & 1.21e-5 & \cellcolor[rgb]{ .439,  .604,  .808}1.4e-5 & 1.3e-5 & \cellcolor[rgb]{ .447,  .608,  .808}1.46e-5 & 1.36e-5 & \cellcolor[rgb]{ .451,  .608,  .808}1.51e-5 & 1.4e-5 & \cellcolor[rgb]{ .365,  .549,  .78}5.35e-6 & 3.82e-6 & \cellcolor[rgb]{ .353,  .541,  .776}\textit{3.63e-6} & \textit{2.02e-6} \\
    5.0e-3 & \cellcolor[rgb]{ .376,  .557,  .784}6.51e-6 & 4.77e-6 & \cellcolor[rgb]{ .384,  .561,  .784}7.48e-6 & 6.34e-6 & \cellcolor[rgb]{ .42,  .588,  .8}1.15e-5 & 1.06e-5 & \cellcolor[rgb]{ .439,  .604,  .808}1.41e-5 & 1.31e-5 & \cellcolor[rgb]{ .459,  .616,  .812}1.6e-5 & 1.49e-5 & \cellcolor[rgb]{ .467,  .624,  .816}1.72e-5 & 1.61e-5 & \cellcolor[rgb]{ .475,  .627,  .82}1.8e-5 & 1.69e-5 & \cellcolor[rgb]{ .482,  .631,  .82}1.87e-5 & 1.74e-5 & \cellcolor[rgb]{ .373,  .557,  .784}6.24e-6 & 4.77e-6 & \cellcolor[rgb]{ .353,  .541,  .776}\textit{4.01e-6} & \textit{2.32e-6} \\
    1.0e-2 & \cellcolor[rgb]{ .424,  .588,  .8}1.2e-5 & 9.39e-6 & \cellcolor[rgb]{ .443,  .604,  .808}1.43e-5 & 1.23e-5 & \cellcolor[rgb]{ .514,  .655,  .831}2.24e-5 & 2.06e-5 & \cellcolor[rgb]{ .557,  .682,  .847}2.75e-5 & 2.57e-5 & \cellcolor[rgb]{ .588,  .706,  .859}3.13e-5 & 2.94e-5 & \cellcolor[rgb]{ .612,  .722,  .867}3.37e-5 & 3.17e-5 & \cellcolor[rgb]{ .624,  .733,  .871}3.54e-5 & 3.33e-5 & \cellcolor[rgb]{ .635,  .741,  .875}3.66e-5 & 3.45e-5 & \cellcolor[rgb]{ .424,  .592,  .8}1.2e-5 & 9.45e-6 & \cellcolor[rgb]{ .38,  .561,  .784}\textit{6.99e-6} & \textit{4.09e-6} \\
    2.0e-2 & \cellcolor[rgb]{ .525,  .663,  .835}2.37e-5 & 1.83e-5 & \cellcolor[rgb]{ .565,  .69,  .851}2.83e-5 & 2.4e-5 & \cellcolor[rgb]{ .702,  .784,  .898}4.42e-5 & 4.08e-5 & \cellcolor[rgb]{ .788,  .847,  .929}5.42e-5 & 5.07e-5 & \cellcolor[rgb]{ .851,  .89,  .949}6.17e-5 & 5.79e-5 & \cellcolor[rgb]{ .894,  .922,  .965}6.63e-5 & 6.24e-5 & \cellcolor[rgb]{ .922,  .941,  .976}6.95e-5 & 6.54e-5 & \cellcolor[rgb]{ .941,  .953,  .98}7.18e-5 & 6.76e-5 & \cellcolor[rgb]{ .525,  .663,  .835}2.39e-5 & 1.85e-5 & \cellcolor[rgb]{ .439,  .6,  .804}\textit{1.38e-5} & \textit{8.21e-6} \\
    3.0e-2 & \cellcolor[rgb]{ .624,  .733,  .871}3.54e-5 & 2.71e-5 & \cellcolor[rgb]{ .686,  .776,  .894}4.25e-5 & 3.61e-5 & \cellcolor[rgb]{ .89,  .918,  .965}6.6e-5 & 6.07e-5 & \cellcolor[rgb]{ .988,  .988,  1}8.07e-5 & 7.48e-5 & \cellcolor[rgb]{ .988,  .988,  1}9.15e-5 & 8.51e-5 & \cellcolor[rgb]{ .988,  .988,  1}9.83e-5 & 9.14e-5 & \cellcolor[rgb]{ .988,  .988,  1}1.03e-4 & 9.58e-5 & \cellcolor[rgb]{ .988,  .988,  1}1.06e-4 & 9.9e-5 & \cellcolor[rgb]{ .627,  .733,  .871}3.58e-5 & 2.77e-5 & \cellcolor[rgb]{ .498,  .643,  .827}\textit{2.08e-5} & \textit{1.23e-5} \\
    4.0e-2 & \cellcolor[rgb]{ .729,  .804,  .906}4.73e-5 & 3.62e-5 & \cellcolor[rgb]{ .808,  .863,  .937}5.67e-5 & 4.81e-5 & \cellcolor[rgb]{ .988,  .988,  1}8.77e-5 & 7.98e-5 & \cellcolor[rgb]{ .988,  .988,  1}1.07e-4 & 9.81e-5 & \cellcolor[rgb]{ .988,  .988,  1}1.21e-4 & 1.12e-4 & \cellcolor[rgb]{ .988,  .988,  1}1.3e-4 & 1.2e-4 & \cellcolor[rgb]{ .988,  .988,  1}1.36e-4 & 1.26e-4 & \cellcolor[rgb]{ .988,  .988,  1}1.4e-4 & 1.3e-4 & \cellcolor[rgb]{ .733,  .808,  .91}4.8e-5 & 3.73e-5 & \cellcolor[rgb]{ .561,  .686,  .847}\textit{2.79e-5} & \textit{1.64e-5} \\
    5.0e-2 & \cellcolor[rgb]{ .831,  .878,  .945}5.93e-5 & 4.54e-5 & \cellcolor[rgb]{ .929,  .949,  .98}7.08e-5 & 5.97e-5 & \cellcolor[rgb]{ .988,  .988,  1}1.09e-4 & 9.83e-5 & \cellcolor[rgb]{ .988,  .988,  1}1.33e-4 & 1.21e-4 & \cellcolor[rgb]{ .988,  .988,  1}1.5e-4 & 1.38e-4 & \cellcolor[rgb]{ .988,  .988,  1}1.61e-4 & 1.48e-4 & \cellcolor[rgb]{ .988,  .988,  1}1.68e-4 & 1.56e-4 & \cellcolor[rgb]{ .988,  .988,  1}1.73e-4 & 1.61e-4 & \cellcolor[rgb]{ .839,  .882,  .945}6.02e-5 & 4.69e-5 & \cellcolor[rgb]{ .624,  .729,  .871}\textit{3.5e-5} & \textit{2.05e-5} \\
    1.0e-1 & \cellcolor[rgb]{ .988,  .988,  1}1.19e-4 & 8.82e-5 & \cellcolor[rgb]{ .988,  .988,  1}1.4e-4 & 1.13e-4 & \cellcolor[rgb]{ .988,  .988,  1}2.13e-4 & 1.89e-4 & \cellcolor[rgb]{ .988,  .988,  1}2.57e-4 & 2.34e-4 & \cellcolor[rgb]{ .988,  .988,  1}2.89e-4 & 2.66e-4 & \cellcolor[rgb]{ .988,  .988,  1}3.08e-4 & 2.85e-4 & \cellcolor[rgb]{ .988,  .988,  1}3.21e-4 & 2.98e-4 & \cellcolor[rgb]{ .988,  .988,  1}3.31e-4 & 3.08e-4 & \cellcolor[rgb]{ .988,  .988,  1}1.2e-4 & 9.19e-5 & \cellcolor[rgb]{ .933,  .949,  .98}\textit{7.12e-5} & \textit{4.02e-5} \\
    \bottomrule
    \bottomrule
    \multicolumn{21}{l}{Color indicators are added to locate the extreme values: the bluer the better accuracy, the redder the worse. Also, the best result for each $\alpha$ (each row) is marked in \textbf{bold} font.}          
    \end{tabular}%
  \label{tab:ringworld_behavior_02_target_015}%
\end{table}%

%% file: tab_frozenlake.tex
\begin{table}[htbp]
\tiny
\setlength{\tabcolsep}{1pt}
  \centering
  \caption{Detailed Results on FrozenLake}
    \begin{tabular}{ccccccccccccccccccccccc}
    \toprule
    \toprule
    baseline & \multicolumn{22}{c}{True Online GTD} \\
    \multirow{2}[1]{*}{$\alpha$\textbackslash{}$\lambda$} & \multicolumn{2}{c}{0} & \multicolumn{2}{c}{0.4} & \multicolumn{2}{c}{0.8} & \multicolumn{2}{c}{0.9} & \multicolumn{2}{c}{0.95} & \multicolumn{2}{c}{0.975} & \multicolumn{2}{c}{0.99} & \multicolumn{2}{c}{1} & \multicolumn{2}{c}{greedy} & \multicolumn{2}{c}{\textit{META(np)}} & \multicolumn{2}{c}{\textit{META}} \\
          & mean  & std   & mean  & std   & mean  & std   & mean  & std   & mean  & std   & mean  & std   & mean  & std   & mean  & std   & mean  & std   & \textit{mean} & \textit{std} & \textit{mean} & \textit{std} \\
    \midrule
    1.0e-5 & \cellcolor[rgb]{ .988,  .914,  .925}2.14e-3 & 1.11e-5 & \cellcolor[rgb]{ .988,  .929,  .941}1.85e-3 & 1.63e-5 & \cellcolor[rgb]{ .988,  .953,  .965}1.38e-3 & 1.91e-5 & \cellcolor[rgb]{ .988,  .957,  .969}1.26e-3 & 1.72e-5 & \cellcolor[rgb]{ .988,  .961,  .973}1.21e-3 & 1.55e-5 & \cellcolor[rgb]{ .988,  .961,  .973}1.18e-3 & 1.45e-5 & \cellcolor[rgb]{ .988,  .961,  .973}1.17e-3 & 1.39e-5 & \cellcolor[rgb]{ .988,  .961,  .973}1.16e-3 & 1.35e-5 & \cellcolor[rgb]{ .988,  .925,  .933}1.92e-3 & 2.09e-5 & \cellcolor[rgb]{ .988,  .961,  .973}\textit{1.16e-3} & \textit{1.35e-5} & \cellcolor[rgb]{ .988,  .961,  .973}\textit{\textbf{1.16e-3}} & \textit{1.35e-5} \\
    2.0e-5 & \cellcolor[rgb]{ .988,  .925,  .937}1.93e-3 & 1.57e-5 & \cellcolor[rgb]{ .988,  .945,  .957}1.55e-3 & 2.15e-5 & \cellcolor[rgb]{ .988,  .965,  .976}1.1e-3 & 2.18e-5 & \cellcolor[rgb]{ .988,  .969,  .98}1.01e-3 & 1.93e-5 & \cellcolor[rgb]{ .988,  .973,  .984}9.67e-4 & 1.78e-5 & \cellcolor[rgb]{ .988,  .973,  .984}9.48e-4 & 1.72e-5 & \cellcolor[rgb]{ .988,  .973,  .984}9.38e-4 & 1.7e-5 & \cellcolor[rgb]{ .988,  .973,  .984}9.32e-4 & 1.7e-5 & \cellcolor[rgb]{ .988,  .941,  .949}1.62e-3 & 2.84e-5 & \cellcolor[rgb]{ .988,  .973,  .984}\textit{\textbf{9.32e-4}} & \textit{1.7e-5} & \cellcolor[rgb]{ .988,  .973,  .984}\textit{9.33e-4} & \textit{1.7e-5} \\
    3.0e-5 & \cellcolor[rgb]{ .988,  .933,  .945}1.76e-3 & 1.88e-5 & \cellcolor[rgb]{ .988,  .953,  .965}1.35e-3 & 2.45e-5 & \cellcolor[rgb]{ .988,  .976,  .988}9.28e-4 & 2.34e-5 & \cellcolor[rgb]{ .988,  .98,  .992}8.42e-4 & 2.1e-5 & \cellcolor[rgb]{ .988,  .98,  .992}8.05e-4 & 1.98e-5 & \cellcolor[rgb]{ .988,  .98,  .992}7.88e-4 & 1.95e-5 & \cellcolor[rgb]{ .988,  .98,  .992}7.79e-4 & 1.97e-5 & \cellcolor[rgb]{ .988,  .984,  .996}7.73e-4 & 2.01e-5 & \cellcolor[rgb]{ .988,  .949,  .961}1.41e-3 & 3.18e-5 & \cellcolor[rgb]{ .988,  .984,  .996}\textit{\textbf{7.73e-4}} & \textit{2.01e-5} & \cellcolor[rgb]{ .988,  .984,  .996}\textit{7.74e-4} & \textit{2.0e-5} \\
    4.0e-5 & \cellcolor[rgb]{ .988,  .941,  .953}1.63e-3 & 2.12e-5 & \cellcolor[rgb]{ .988,  .961,  .973}1.2e-3 & 2.64e-5 & \cellcolor[rgb]{ .988,  .98,  .992}7.97e-4 & 2.43e-5 & \cellcolor[rgb]{ .988,  .984,  .996}7.17e-4 & 2.19e-5 & \cellcolor[rgb]{ .988,  .988,  1}6.82e-4 & 2.1e-5 & \cellcolor[rgb]{ .988,  .988,  1}6.67e-4 & 2.12e-5 & \cellcolor[rgb]{ .988,  .988,  1}6.59e-4 & 2.19e-5 & \cellcolor[rgb]{ .988,  .988,  1}\textbf{6.53e-4} & 2.26e-5 & \cellcolor[rgb]{ .988,  .957,  .969}1.25e-3 & 3.35e-5 & \cellcolor[rgb]{ .988,  .988,  1}\textit{6.54e-4} & \textit{2.26e-5} & \cellcolor[rgb]{ .988,  .988,  1}\textit{6.54e-4} & \textit{2.24e-5} \\
    5.0e-5 & \cellcolor[rgb]{ .988,  .945,  .957}1.51e-3 & 2.3e-5 & \cellcolor[rgb]{ .988,  .969,  .98}1.07e-3 & 2.78e-5 & \cellcolor[rgb]{ .988,  .988,  1}6.93e-4 & 2.48e-5 & \cellcolor[rgb]{ .976,  .98,  .996}6.18e-4 & 2.24e-5 & \cellcolor[rgb]{ .933,  .949,  .98}5.86e-4 & 2.18e-5 & \cellcolor[rgb]{ .914,  .937,  .973}5.72e-4 & 2.26e-5 & \cellcolor[rgb]{ .906,  .929,  .969}5.65e-4 & 2.37e-5 & \cellcolor[rgb]{ .898,  .925,  .969}5.6e-4 & 2.49e-5 & \cellcolor[rgb]{ .988,  .965,  .976}1.12e-3 & 3.42e-5 & \cellcolor[rgb]{ .898,  .925,  .969}\textit{\textbf{5.6e-4}} & \textit{2.49e-5} & \cellcolor[rgb]{ .898,  .925,  .969}\textit{5.61e-4} & \textit{2.45e-5} \\
    1.0e-4 & \cellcolor[rgb]{ .988,  .965,  .976}1.13e-3 & 2.84e-5 & \cellcolor[rgb]{ .988,  .988,  1}6.98e-4 & 3.07e-5 & \cellcolor[rgb]{ .675,  .765,  .886}3.91e-4 & 2.4e-5 & \cellcolor[rgb]{ .604,  .718,  .863}3.39e-4 & 2.24e-5 & \cellcolor[rgb]{ .58,  .702,  .855}3.2e-4 & 2.55e-5 & \cellcolor[rgb]{ .573,  .694,  .851}3.13e-4 & 2.99e-5 & \cellcolor[rgb]{ .569,  .69,  .851}3.11e-4 & 3.41e-5 & \cellcolor[rgb]{ .569,  .69,  .851}3.1e-4 & 3.79e-5 & \cellcolor[rgb]{ .988,  .984,  .996}7.76e-4 & 3.45e-5 & \cellcolor[rgb]{ .565,  .69,  .851}\textit{3.1e-4} & \textit{3.43e-5} & \cellcolor[rgb]{ .565,  .69,  .851}\textit{\textbf{3.09e-4}} & \textit{2.94e-5} \\
    2.0e-4 & \cellcolor[rgb]{ .988,  .984,  .996}7.85e-4 & 3.4e-5 & \cellcolor[rgb]{ .667,  .761,  .886}3.99e-4 & 3.12e-5 & \cellcolor[rgb]{ .42,  .588,  .8}2.0e-4 & 2.11e-5 & \cellcolor[rgb]{ .4,  .573,  .792}1.83e-4 & 2.84e-5 & \cellcolor[rgb]{ .4,  .573,  .792}1.84e-4 & 4.2e-5 & \cellcolor[rgb]{ .408,  .576,  .792}1.89e-4 & 5.45e-5 & \cellcolor[rgb]{ .412,  .584,  .796}1.94e-4 & 6.51e-5 & \cellcolor[rgb]{ .42,  .588,  .8}1.99e-4 & 7.42e-5 & \cellcolor[rgb]{ .851,  .89,  .949}5.25e-4 & 3.39e-5 & \cellcolor[rgb]{ .396,  .573,  .792}\textit{1.81e-4} & \textit{3.32e-5} & \cellcolor[rgb]{ .396,  .573,  .792}\textit{\textbf{1.81e-4}} & \textit{2.04e-5} \\
    3.0e-4 & \cellcolor[rgb]{ .941,  .957,  .984}6.18e-4 & 3.71e-5 & \cellcolor[rgb]{ .514,  .655,  .831}2.78e-4 & 2.93e-5 & \cellcolor[rgb]{ .365,  .549,  .78}1.58e-4 & 2.35e-5 & \cellcolor[rgb]{ .369,  .553,  .78}1.6e-4 & 4.15e-5 & \cellcolor[rgb]{ .384,  .565,  .788}1.73e-4 & 6.51e-5 & \cellcolor[rgb]{ .4,  .573,  .792}1.85e-4 & 8.57e-5 & \cellcolor[rgb]{ .416,  .584,  .796}1.95e-4 & 1.03e-4 & \cellcolor[rgb]{ .427,  .592,  .8}2.04e-4 & 1.18e-4 & \cellcolor[rgb]{ .686,  .776,  .894}4.01e-4 & 3.29e-5 & \cellcolor[rgb]{ .373,  .553,  .78}\textit{1.64e-4} & \textit{5.08e-5} & \cellcolor[rgb]{ .357,  .541,  .776}\textit{\textbf{1.5e-4}} & \textit{2.42e-5} \\
    4.0e-4 & \cellcolor[rgb]{ .8,  .855,  .933}5.06e-4 & 3.85e-5 & \cellcolor[rgb]{ .439,  .6,  .804}2.2e-4 & 2.69e-5 & \cellcolor[rgb]{ .353,  .541,  .776}1.5e-4 & 2.91e-5 & \cellcolor[rgb]{ .373,  .553,  .78}1.64e-4 & 5.68e-5 & \cellcolor[rgb]{ .4,  .573,  .792}1.84e-4 & 9.08e-5 & \cellcolor[rgb]{ .424,  .592,  .8}2.02e-4 & 1.2e-4 & \cellcolor[rgb]{ .443,  .604,  .808}2.17e-4 & 1.46e-4 & \cellcolor[rgb]{ .459,  .616,  .812}2.29e-4 & 1.68e-4 & \cellcolor[rgb]{ .584,  .702,  .855}3.23e-4 & 3.12e-5 & \cellcolor[rgb]{ .365,  .549,  .78}\textit{1.58e-4} & \textit{1.41e-5} & \cellcolor[rgb]{ .353,  .541,  .776}\textit{\textbf{1.47e-4}} & \textit{3.06e-5} \\
    5.0e-4 & \cellcolor[rgb]{ .698,  .78,  .894}4.23e-4 & 3.87e-5 & \cellcolor[rgb]{ .4,  .573,  .792}1.89e-4 & 2.48e-5 & \cellcolor[rgb]{ .357,  .541,  .776}1.51e-4 & 3.61e-5 & \cellcolor[rgb]{ .384,  .565,  .788}1.73e-4 & 7.31e-5 & \cellcolor[rgb]{ .42,  .588,  .8}2.0e-4 & 1.18e-4 & \cellcolor[rgb]{ .451,  .612,  .812}2.23e-4 & 1.58e-4 & \cellcolor[rgb]{ .475,  .627,  .82}2.42e-4 & 1.92e-4 & \cellcolor[rgb]{ .498,  .643,  .827}2.57e-4 & 2.21e-4 & \cellcolor[rgb]{ .518,  .655,  .831}2.73e-4 & 2.94e-5 & \cellcolor[rgb]{ .353,  .541,  .776}\textit{\textbf{1.49e-4}} & \textit{1.59e-5} & \cellcolor[rgb]{ .357,  .541,  .776}\textit{1.51e-4} & \textit{3.78e-5} \\
    1.0e-3 & \cellcolor[rgb]{ .455,  .612,  .812}2.3e-4 & 3.29e-5 & \cellcolor[rgb]{ .357,  .545,  .776}1.54e-4 & 2.2e-5 & \cellcolor[rgb]{ .392,  .569,  .788}1.78e-4 & 7.43e-5 & \cellcolor[rgb]{ .459,  .616,  .812}2.29e-4 & 1.59e-4 & \cellcolor[rgb]{ .537,  .671,  .839}2.86e-4 & 2.66e-4 & \cellcolor[rgb]{ .596,  .714,  .863}3.33e-4 & 3.63e-4 & \cellcolor[rgb]{ .647,  .749,  .878}3.72e-4 & 4.48e-4 & \cellcolor[rgb]{ .69,  .776,  .894}4.03e-4 & 5.21e-4 & \cellcolor[rgb]{ .416,  .584,  .796}1.95e-4 & 2.56e-5 & \cellcolor[rgb]{ .361,  .545,  .776}\textit{\textbf{1.54e-4}} & \textit{2.9e-5} & \cellcolor[rgb]{ .365,  .549,  .78}\textit{1.57e-4} & \textit{1.97e-5} \\
    2.0e-3 & \cellcolor[rgb]{ .38,  .561,  .784}1.74e-4 & 2.69e-5 & \cellcolor[rgb]{ .369,  .553,  .78}1.63e-4 & 3.3e-5 & \cellcolor[rgb]{ .471,  .624,  .816}2.38e-4 & 1.53e-4 & \cellcolor[rgb]{ .608,  .718,  .863}3.4e-4 & 3.43e-4 & \cellcolor[rgb]{ .761,  .827,  .918}4.55e-4 & 5.89e-4 & \cellcolor[rgb]{ .882,  .914,  .961}5.49e-4 & 8.16e-4 & \cellcolor[rgb]{ .988,  .988,  1}6.26e-4 & 1.02e-3 & \cellcolor[rgb]{ .988,  .988,  1}6.89e-4 & 1.19e-3 & \cellcolor[rgb]{ .412,  .58,  .796}1.92e-4 & 3.39e-5 & \cellcolor[rgb]{ .369,  .553,  .78}\textit{\textbf{1.61e-4}} & \textit{1.82e-5} & \cellcolor[rgb]{ .373,  .557,  .784}\textit{1.64e-4} & \textit{2.81e-5} \\
    3.0e-3 & \cellcolor[rgb]{ .38,  .561,  .784}1.72e-4 & 2.9e-5 & \cellcolor[rgb]{ .384,  .565,  .788}1.77e-4 & 4.63e-5 & \cellcolor[rgb]{ .549,  .678,  .843}2.96e-4 & 2.37e-4 & \cellcolor[rgb]{ .753,  .82,  .914}4.49e-4 & 5.43e-4 & \cellcolor[rgb]{ .976,  .98,  .996}6.18e-4 & 9.38e-4 & \cellcolor[rgb]{ .988,  .984,  .996}7.57e-4 & 1.3e-3 & \cellcolor[rgb]{ .988,  .976,  .988}8.7e-4 & 1.61e-3 & \cellcolor[rgb]{ .988,  .973,  .984}9.63e-4 & 1.88e-3 & \cellcolor[rgb]{ .427,  .592,  .8}2.04e-4 & 4.56e-5 & \cellcolor[rgb]{ .376,  .557,  .784}\textit{\textbf{1.67e-4}} & \textit{2.41e-5} & \cellcolor[rgb]{ .38,  .561,  .784}\textit{1.71e-4} & \textit{2.81e-5} \\
    4.0e-3 & \cellcolor[rgb]{ .384,  .565,  .788}1.76e-4 & 3.31e-5 & \cellcolor[rgb]{ .404,  .576,  .792}1.91e-4 & 5.98e-5 & \cellcolor[rgb]{ .624,  .733,  .871}3.54e-4 & 3.26e-4 & \cellcolor[rgb]{ .894,  .922,  .965}5.56e-4 & 7.53e-4 & \cellcolor[rgb]{ .988,  .984,  .992}7.78e-4 & 1.3e-3 & \cellcolor[rgb]{ .988,  .973,  .984}9.58e-4 & 1.79e-3 & \cellcolor[rgb]{ .988,  .965,  .976}1.1e-3 & 2.21e-3 & \cellcolor[rgb]{ .988,  .961,  .973}1.22e-3 & 2.57e-3 & \cellcolor[rgb]{ .443,  .604,  .808}2.17e-4 & 5.7e-5 & \cellcolor[rgb]{ .384,  .565,  .788}\textit{\textbf{1.73e-4}} & \textit{2.97e-5} & \cellcolor[rgb]{ .388,  .565,  .788}\textit{1.76e-4} & \textit{3.28e-5} \\
    5.0e-3 & \cellcolor[rgb]{ .392,  .569,  .788}1.81e-4 & 3.75e-5 & \cellcolor[rgb]{ .424,  .588,  .8}2.06e-4 & 7.33e-5 & \cellcolor[rgb]{ .702,  .784,  .898}4.11e-4 & 4.17e-4 & \cellcolor[rgb]{ .988,  .988,  1}6.6e-4 & 9.68e-4 & \cellcolor[rgb]{ .988,  .973,  .984}9.32e-4 & 1.66e-3 & \cellcolor[rgb]{ .988,  .965,  .976}1.15e-3 & 2.27e-3 & \cellcolor[rgb]{ .988,  .953,  .965}1.33e-3 & 2.79e-3 & \cellcolor[rgb]{ .988,  .945,  .957}1.47e-3 & 3.24e-3 & \cellcolor[rgb]{ .459,  .616,  .812}2.3e-4 & 6.83e-5 & \cellcolor[rgb]{ .392,  .569,  .788}\textit{\textbf{1.8e-4}} & \textit{3.5e-5} & \cellcolor[rgb]{ .396,  .569,  .788}\textit{1.81e-4} & \textit{3.73e-5} \\
    1.0e-2 & \cellcolor[rgb]{ .424,  .592,  .8}2.08e-4 & 5.91e-5 & \cellcolor[rgb]{ .518,  .655,  .831}2.81e-4 & 1.41e-4 & \cellcolor[rgb]{ .988,  .988,  1}6.89e-4 & 8.84e-4 & \cellcolor[rgb]{ .988,  .965,  .976}1.15e-3 & 2.05e-3 & \cellcolor[rgb]{ .988,  .937,  .949}1.62e-3 & 3.41e-3 & \cellcolor[rgb]{ .988,  .922,  .929}1.99e-3 & 4.54e-3 & \cellcolor[rgb]{ .988,  .906,  .918}2.27e-3 & 5.44e-3 & \cellcolor[rgb]{ .988,  .894,  .906}2.5e-3 & 6.17e-3 & \cellcolor[rgb]{ .561,  .686,  .847}3.06e-4 & 1.3e-4 & \cellcolor[rgb]{ .431,  .596,  .804}\textit{2.09e-4} & \textit{5.97e-5} & \cellcolor[rgb]{ .431,  .596,  .804}\textit{\textbf{2.08e-4}} & \textit{5.91e-5} \\
    2.0e-2 & \cellcolor[rgb]{ .498,  .643,  .827}2.67e-4 & 1.07e-4 & \cellcolor[rgb]{ .71,  .792,  .902}4.33e-4 & 2.82e-4 & \cellcolor[rgb]{ .988,  .961,  .973}1.19e-3 & 1.84e-3 & \cellcolor[rgb]{ .988,  .922,  .933}1.97e-3 & 4.14e-3 & \cellcolor[rgb]{ .988,  .882,  .894}2.73e-3 & 6.67e-3 & \cellcolor[rgb]{ .988,  .855,  .863}3.28e-3 & 8.61e-3 & \cellcolor[rgb]{ .984,  .831,  .843}3.7e-3 & 1.01e-2 & \cellcolor[rgb]{ .984,  .816,  .827}4.02e-3 & 1.12e-2 & \cellcolor[rgb]{ .839,  .882,  .945}5.17e-4 & 3.6e-4 & \cellcolor[rgb]{ .51,  .651,  .831}\textit{2.68e-4} & \textit{1.07e-4} & \cellcolor[rgb]{ .51,  .651,  .831}\textit{\textbf{2.67e-4}} & \textit{1.07e-4} \\
    3.0e-2 & \cellcolor[rgb]{ .58,  .702,  .855}3.31e-4 & 1.71e-4 & \cellcolor[rgb]{ .898,  .925,  .969}5.85e-4 & 4.37e-4 & \cellcolor[rgb]{ .988,  .937,  .949}1.63e-3 & 2.81e-3 & \cellcolor[rgb]{ .988,  .886,  .898}2.66e-3 & 6.15e-3 & \cellcolor[rgb]{ .984,  .835,  .847}3.63e-3 & 9.67e-3 & \cellcolor[rgb]{ .984,  .8,  .808}4.34e-3 & 1.23e-2 & \cellcolor[rgb]{ .984,  .773,  .784}4.86e-3 & 1.42e-2 & \cellcolor[rgb]{ .984,  .753,  .761}5.26e-3 & 1.57e-2 & \cellcolor[rgb]{ .988,  .98,  .992}8.09e-4 & 9.22e-4 & \cellcolor[rgb]{ .596,  .71,  .859}\textit{3.32e-4} & \textit{1.73e-4} & \cellcolor[rgb]{ .596,  .71,  .859}\textit{\textbf{3.31e-4}} & \textit{1.71e-4} \\
    4.0e-2 & \cellcolor[rgb]{ .667,  .761,  .886}3.99e-4 & 2.54e-4 & \cellcolor[rgb]{ .988,  .984,  .996}7.32e-4 & 6.05e-4 & \cellcolor[rgb]{ .988,  .918,  .929}2.03e-3 & 3.8e-3 & \cellcolor[rgb]{ .988,  .855,  .867}3.27e-3 & 8.09e-3 & \cellcolor[rgb]{ .984,  .792,  .804}4.44e-3 & 1.25e-2 & \cellcolor[rgb]{ .984,  .749,  .761}5.28e-3 & 1.57e-2 & \cellcolor[rgb]{ .984,  .718,  .729}5.9e-3 & 1.81e-2 & \cellcolor[rgb]{ .98,  .694,  .706}6.38e-3 & 1.99e-2 & \cellcolor[rgb]{ .988,  .961,  .973}1.17e-3 & 1.94e-3 & \cellcolor[rgb]{ .686,  .776,  .894}\textit{4.01e-4} & \textit{2.59e-4} & \cellcolor[rgb]{ .686,  .773,  .89}\textit{\textbf{3.99e-4}} & \textit{2.54e-4} \\
    5.0e-2 & \cellcolor[rgb]{ .757,  .824,  .918}4.71e-4 & 3.54e-4 & \cellcolor[rgb]{ .988,  .98,  .992}8.76e-4 & 7.84e-4 & \cellcolor[rgb]{ .988,  .898,  .91}2.4e-3 & 4.78e-3 & \cellcolor[rgb]{ .984,  .824,  .835}3.83e-3 & 9.98e-3 & \cellcolor[rgb]{ .984,  .757,  .765}5.18e-3 & 1.52e-2 & \cellcolor[rgb]{ .984,  .706,  .718}6.15e-3 & 1.9e-2 & \cellcolor[rgb]{ .98,  .671,  .678}6.87e-3 & 2.18e-2 & \cellcolor[rgb]{ .98,  .639,  .651}7.42e-3 & 2.4e-2 & \cellcolor[rgb]{ .988,  .957,  .969}1.3e-3 & 1.8e-3 & \cellcolor[rgb]{ .78,  .843,  .925}\textit{4.71e-4} & \textit{3.54e-4} & \cellcolor[rgb]{ .78,  .843,  .925}\textit{\textbf{4.71e-4}} & \textit{3.54e-4} \\
    1.0e-1 & \cellcolor[rgb]{ .988,  .976,  .988}8.97e-4 & 1.06e-3 & \cellcolor[rgb]{ .988,  .945,  .957}1.52e-3 & 1.82e-3 & \cellcolor[rgb]{ .984,  .824,  .835}3.85e-3 & 9.39e-3 & \cellcolor[rgb]{ .984,  .71,  .718}6.11e-3 & 1.84e-2 & \cellcolor[rgb]{ .98,  .596,  .608}8.26e-3 & 2.74e-2 & \cellcolor[rgb]{ .976,  .518,  .525}9.8e-3 & 3.38e-2 & \cellcolor[rgb]{ .976,  .459,  .467}1.09e-2 & 3.87e-2 & \cellcolor[rgb]{ .973,  .412,  .42}1.18e-2 & 4.24e-2 & \cellcolor[rgb]{ .988,  .89,  .902}2.59e-3 & 4.99e-3 & \cellcolor[rgb]{ .988,  .976,  .988}\textit{8.97e-4} & \textit{1.06e-3} & \cellcolor[rgb]{ .988,  .976,  .988}\textit{\textbf{8.97e-4}} & \textit{1.06e-3} \\
    \bottomrule
    \bottomrule
    \multicolumn{23}{l}{Color indicators are added to locate the extreme values: the bluer the better accuracy, the redder the worse. Also, the best result for each $\alpha$ (each row) is marked in \textbf{bold} font.}    
    \end{tabular}%
  \label{tab:frozenlake}%
\end{table}%

%% file: chapter_6_conclusion.tex
\chapter{Conclusion and Future Work}

\section{Summary of Contributions}
In this thesis, we provided the following original contributions:
\begin{enumerate}
\item
The surrogate approach of optimizing sample efficiency by the optimization of update targets.
\item
A principled approach, \algoname{}, which achieves approximate meta-optimization of the sample efficiency of policy evaluation while remaining compatible with online updating, off-policy learning, function approximation and control, with minimal additional computational cost.
\item
Identification of the distribution mismatch between the ``on-policy'' distribution and the assumed ``stationary'' distribution.
\end{enumerate}
We demonstrated the merits of the proposed approach in several experiments.

Beyond this, several smaller contributions are also presented:
\begin{enumerate}
\item
A re-framing of the basics of RL using state-based discounting setting.
\item
Identification of the distribution mismatch between the ``on-policy'' distribution with the assumed ``stationary'' distribution in the episodic setting.
\item
A new proof of the policy gradient theorem under state-based discount setting.
\end{enumerate}

\section{Future Work}
The contribution of this thesis points out several promising directions worthy of research in the future:
\begin{enumerate}
\item
Investigate the possibilities of using traces to accumulate one-step meta-gradient updates, in order to achieve better approximation of the true gradients: the approximation of the per-state meta-gradients can be ameliorated by looking more steps into the future, which draws great similarity to the online accumulation of the eligibility traces. Perhaps better meta-gradients can be achieved by utilizing traces.
\item
Investigate the convergence properties of the update targets under meta-optimization: under certain assumptions, we proved that the meta-gradient updates can optimize the overall targets. But do these targets converge? Or, what dynamic patterns do these targets follow?
\item
Investigate formally the relationship between update targets and sample efficiency: optimization of update targets, as a surrogate of sample efficiency, greatly simplifies the analyses and exhibits significant empirical performance boost. Yet, we have not formally established how the quality of the update targets influences the sample efficiency.
\item
Investigate better ways of combining the proposed approach with actor-critic methods: the trust regions required by the auxiliary learners already limit the learning rate of the critic to be small. Together with the fact that we need to limit the learning rate of the actor to be smaller than the critic, this further limits the learning speed of actor-critic systems. We would like to find a more efficient algorithmic approach than the one we used in this thesis.
\end{enumerate}

%% file: chapter_appendices.tex
\chapter*{APPENDICES}
In this chapter, the assistive details of the thesis will be provided.

\section{Technical Auxiliaries for Experiments}

\subsection{Environments}

\subsubsection{RingWorld}
The RingWorld domain was introduced as a suitable domain for investigating $\lambda$ in \cite{kearns2000bias}. It describes an environment a ring of states. The starting state is always in the top-middle of the ring and the agent can take two actions, either moving to the state to the left or the state to the right. There are two adjoining terminal states at the bottom middle of the ring. The left end gives $-1$ reward and the right end gives $+1$. Reaching the bottom two terminal states result in the teleportation back to the starting state. The sparsity of the rewards made the selection of an appropriate $\lambda$ value worthy of investigating.

In the experiments, the RingWorld environment is reproduced with the help of the description in \cite{white2016greedy} and \cite{kearns2000bias}. Despite being loyal to the original setting as much as possible, due to limitations of our understanding, we cannot see the difference between the RingWorld and a random walk environment with the rewards on the two tails. The RingWorld environment is described as a symmetric ring of the states with the starting state at the top-middle, for which we think the number of states should be odd. However, the authors claimed that they conducted experiments on the $10$-state and $50$-state instances. We instead did the experiments on the $11$-state instance. Thus in the experiments, we adopted a random walk version of RingWorld.

\subsubsection{FrozenLake}
The FrozenLake environment is a very noisy and high-variance domain in the OpenAI gym environment bundle \cite{brockman2016gym}. The environment features a scenario in which an agent tries to fetch back a lost frisbee on the surface of a frozen lake. In this task, the agent controls the movement of a character in a grid world. Some tiles of the grid are walkable, and others lead to the agent falling into the water, which terminates the episode and teleports the agent back to the starting point. Furthermore, the movement direction of the agent is uncertain and only partially depends on the chosen direction. The agent is rewarded for finding a walkable path to a goal tile.

To make sure the Markovian properties of the environment, we removed the episode length limit of the FrozenLake environment and thus the environment can be solvable by dynamic programming (the true values of a policy as well as the state distribution).

It is modified based on the Gym environment with the same name. We have used the instance of ``4x4'', \ie{} with $16$ states.

The episode length limit of MountainCar is also removed. We also added noise to the state transitions: actions will be randomized at $20\%$ probability. The noise is to prevent the cases in which $\lambda = 1$ yields the best performance (to prevent \algoname{} from using extremely small $\kappa$'s to get good performance). Additionally, due to the poor exploration of the softmax policy, we extended the starting location to be uniformly anywhere from the left to right on the slopes.

\subsubsection{State Features}
For RingWorld, we used onehot encoding to get equivalence to tabular case; For FrozenLake, we used a discrete variant of tile coding, for which there are $4$ tilings, with each tile covering one grid as well as symmetric and even offset; For MountainCar, we adopted the roughly the same setting as Chapter 10.1 Page 245 in \cite{sutton2018reinforcement}, except that we used ordinary symmetric and even offset instead of the asymmetric offset.

\subsubsection{About ${\bm{\lambda}}$-greedy}
We have replaced VTD \cite{white2016greedy} with direct VTD \cite{sherstan2018directly}. This modification is expected only to improve the stability, without touching the core mechanisms of $\lambda$-greedy \cite{white2016greedy}.

The target used in \cite{white2016greedy} is biased toward $\lambda = 1$, as the $\lambda$'s into the future are assumed to be $1$. Thus we do not think it is helpful to conduct tests on environments with very low variance. For example, RingWorld is low-variance, as the state transitions are deterministic. Also, the policies adopted in \cite{white2016greedy} is very greedy. Such setting further reduces the variance. This is the reason why we have tested different policies (less greedy) in our experiments.

\subsubsection{Buffer Period and Learning Rate}
In the prediction tasks, we used the first $10\%$ of the episodes as the buffer period. The learning rate hyperparameters of the auxiliary learners are set to be the twice of that of the value learner. These settings were not considered in \cite{white2016greedy}, in which there were no buffer period and identical learning rates were used for all learners. 

For the control task of MountainCar, $\lambda$-greedy and \algoname{} will both crash without these two additional settings.

\subsubsection{Details for Non-Parametric $\lambda(\cdot)$}
To disable the generalization of the parametric $\bm{{\bm{\lambda}}}$ for ``\algoname{}(np)'', we replaced the feature vectors for each state with onehot-encoded features.

\subsubsection{More Policies for Prediction}
For RingWorld, we have done $6$ different behavior-target policy pairs (3 on-policy \& 3 off-policy). The off-policy pair that we have shown in the manuscript shares the same patterns as the rest of the pairs. The accuracy improvement brought by \algoname{} is significant across these pairs of policies; For FrozenLake, we have done two pairs of policies (on- and off-policy). We observe the same pattern as in the RingWorld tests.

\subsubsection{Implementation}
The source code could be found at \url{https://github.com/PwnerHarry/META}. The implementation is based on numpy and python, with massive parallelization for hyperprocessing workers. These are to ensure fast experimental results on large scale CPUs.

Due to the estimation instability, sometimes the \algoname{} updates could bring state $\lambda$ values outside $[0, 1]$. In the implementation, whenever we detect such kind of update, we simply cancel that operation.